\documentclass{article}

\usepackage{iclr2027_conference,times}

\usepackage{mathtools}
\usepackage{amssymb,amsthm,amsfonts}
\usepackage{bm}
\usepackage{enumitem}
\usepackage[table]{xcolor}
\usepackage{longtable}
\usepackage{booktabs}
\usepackage{array}
\usepackage{graphicx}
\usepackage{booktabs}
\usepackage{multirow}
\usepackage{array}
\usepackage{longtable}
\usepackage{threeparttable}
\usepackage{algorithm}
\usepackage{algpseudocode}
\algrenewcommand\algorithmicrequire{\textbf{Input:}}
\algrenewcommand\algorithmicensure{\textbf{Output:}}

\usepackage[table]{xcolor}

\usepackage{wrapfig}

\usepackage{amsmath,amsfonts,bm}

\def\1{\mathbf{1}}

\DeclareMathAlphabet{\mathsfit}{\encodingdefault}{\sfdefault}{m}{sl}
\SetMathAlphabet{\mathsfit}{bold}{\encodingdefault}{\sfdefault}{bx}{n}

\newcommand{\R}{\mathbb{R}}

\newcommand{\Gop}{\mathcal{G}}

\newcommand{\LN}{\operatorname{LN}}

\newcommand{\xq}{\bm{x}_q}
\newcommand{\bx}{\bm{x}}
\newcommand{\bk}{\bm{k}}

\definecolor{refblue}{HTML}{0000CD}
\definecolor{emporange}{HTML}{FF8C00}

\newcommand{\refvar}[1]{\textcolor{refblue}{#1}}
\newcommand{\discvar}[1]{\textcolor{emporange}{#1}}

\usepackage{xspace}

\usepackage{pgf}

\usepackage{xurl}
\usepackage{hyperref}
\usepackage{pifont}

\newtheorem{theorem}{Theorem}
\newtheorem{lemma}{Lemma}

\newtheorem{corollary}{Corollary}

\theoremstyle{definition}

\newtheorem{assumption}{Assumption}

\newcommand{\model}{\textbf{\textsc{Stfo}}\xspace}
\newcommand{\CFE}{\textsc{CFE}\xspace}
\newcommand{\SRE}{\textsc{SRE}\xspace}
\newcommand{\DFO}{\textsc{DFO}\xspace}
\newcommand{\CQD}{\textsc{CQD}\xspace}

\newcommand{\best}[1]{\textbf{#1}}
\newcommand{\second}[1]{\underline{#1}}

\definecolor{datasetgray}{RGB}{232,232,232}

\definecolor{heatblue}{RGB}{105,155,205}
\definecolor{heatpink}{RGB}{225,155,175}

\newcommand{\heatcell}[4]{%
  \pgfmathsetmacro{\heatpos}{%
    min(
      1,
      max(
        0,
        (#1-#2)/max(#3-#2,0.000001)
      )
    )
  }%

  \pgfmathparse{\heatpos <= 0.5 ? 1 : 0}%

  \ifnum\pgfmathresult=1
    \pgfmathtruncatemacro{\heatlevel}{%
      round(
        34*(1-2*\heatpos)
      )
    }%

    \edef\heatcmd{%
      \noexpand\cellcolor{heatblue!\heatlevel}%
    }%

  \else
    \pgfmathtruncatemacro{\heatlevel}{%
      round(
        34*(2*\heatpos-1)
      )
    }%

    \edef\heatcmd{%
      \noexpand\cellcolor{heatpink!\heatlevel}%
    }%
  \fi

  \heatcmd #4%
}

\newcommand{\hval}[3]{%
  \heatcell{#1}{#2}{#3}{#1}%
}

\newcommand{\hbest}[3]{%
  \heatcell{#1}{#2}{#3}{\best{#1}}%
}

\newcommand{\hsecond}[3]{%
  \heatcell{#1}{#2}{#3}{\second{#1}}%
}

\definecolor{cfecol}{HTML}{D6E8F5}
\definecolor{srecol}{HTML}{FBE5D6}
\definecolor{dfocol}{HTML}{E2F0D9}
\definecolor{cqdcol}{HTML}{EDE0F2}
\definecolor{rowgray}{HTML}{F2F2F2}

\title{
More Sensors Only One Field: Rethinking Continual Spatio-Temporal  Forecasting
}

\newcommand{\stfoauthorblock}{%
  \parbox[t]{\dimexpr\textwidth-2\tabcolsep\relax}{%
    \raggedright
    Lewei~Xie$^{1}$\thanks{Equal contribution.}\quad
    Haoyu~Zhang$^{2}$\footnotemark[1]\quad
    Jiajun~Zhou$^{3}$\quad
    Yulong~Chen$^{1}$\quad
    Guanxing~Chen$^{1}$\quad
    Yu-An~Huang$^{4}$\quad
    Hau-San~Wong$^{2}$\quad
    Yifan~Zhang$^{1}$\thanks{\raggedright
      Corresponding authors:
      Yifan Zhang (\texttt{yif.zhang@cityu-dg.edu.cn});
      Zhi-An Huang (\texttt{huang.za@cityu-dg.edu.cn}).}\quad
    Zhi-An~Huang$^{1}$\footnotemark[2]
    \par\vspace{0.5em}
    {\normalfont\small
      $^{1}$City University of Hong Kong (Dongguan)\\
      $^{2}$City University of Hong Kong\\
      $^{3}$Zhuhai College of Science and Technology\\
      $^{4}$Northwestern Polytechnical University
      \par
    }%
  }%
}

\author{\stfoauthorblock}

\iclrfinalcopy

\makeatletter
\let\stfo@addcontentsline\addcontentsline
\def\stfo@tocfile{toc}
\def\stfo@notoc#1#2#3{%
  \def\stfo@file{#1}%
  \ifx\stfo@file\stfo@tocfile\else\stfo@addcontentsline{#1}{#2}{#3}\fi}
\let\addcontentsline\stfo@notoc
\makeatother


\begin{document}

\maketitle
\fancyhead{}
\fancyhead[L]{\small\textsc{Preprint}}
\renewcommand{\headrulewidth}{0.4pt}

\begin{abstract}
Continual spatio-temporal forecasting supports traffic management and
environmental monitoring under evolving dynamics and expanding sensor
networks. However, conventional graph-based continual learning methods
tie forecasting representations to the current sensor layout, so sensor
expansion can alter the representation of learned spatial relationships.
\textbf{\textit{Our key insight is that sensor expansion changes the evidence
available about a process without necessarily changing the dynamics
to be learned.}}
We propose \model{} (Spatio-Temporal Field Operator), which parameterizes
forecasting knowledge as a shared field-evolution operator and handles
changing sensor layouts through observation and query interfaces.
Normalized coordinate-based aggregation lifts irregular sensor histories
onto a fixed latent grid, enabling reuse of learned spatial maps across
observation sets without sensor-specific parameters.
To accommodate process drift, a spectral descriptor summarizes variation
across spatial scales and conditions Fourier propagation and attention
to adapt operator responses to the current spatial regime.
Coordinate-based decoding queries the evolved field at sensor locations
and combines spatial corrections with local-history predictions.
Experiments on PEMS-Stream, CA-Stream, and AIR-Stream demonstrate
state-of-the-art average forecasting performance. \model-Large
reduces average MAE over DOL by 8.4\% on PEMS-Stream and
4.7\% on CA-Stream.
Our code is available at
\url{https://github.com/Xielewei/Spatio-Temporal-Field-Operator}.
\end{abstract}

\section{Introduction}
\label{sec:intro}

\textbf{\textit{Continual spatio-temporal forecasting}} is essential for
traffic management and environmental monitoring systems that operate over
extended periods. These systems must maintain accurate predictions as both
underlying dynamics and sensor layouts evolve: seasonal and long-term
changes alter the process being forecast, while newly deployed sensors expand observation coverage
\citep{chen2021trafficstream,francies2025framework}.
This calls for reusing forecasting knowledge from earlier periods while
incorporating new sensors and adapting to changing dynamics.

Existing spatio-temporal forecasting methods typically represent sensors as graph nodes and learn spatial dependencies through graph convolution or attention
\citep{yu2018spatiotemporal,wu2019graph,guo2019attention,zheng2020gman}.
Continual methods extend these representations through experience replay,
knowledge consolidation, and pattern memory
\citep{chen2021trafficstream,wang2023knowledge,wang2023pattern},
or use online adaptation to separate stable and dynamic components
\citep{wang2024towards,wang2025dol}.

\textbf{\textit{Despite this progress, transferring forecasting knowledge
across sensor layouts remains challenging}}, because
\textbf{\textit{changes in the underlying dynamics and observation layouts}} 
enter the same graph-based representation. 
In graph-based formulations, the spatial computation is defined directly over the current sensor set.
Adding sensors can therefore change graph neighborhoods and aggregation weights, even when the underlying spatial dynamics among previously observed locations remain similar.
Figure~\ref{fig:motivation}(a) quantifies this effect: on CA-Stream,
where the sensor count grows from 480 to 1698, up to 72\% of $k$NN edges
are reconfigured after a single expansion.
\textbf{\textit{Our key insight is that such large changes in observation layouts need not imply comparable changes in underlying spatial relationships.}}
Figure~\ref{fig:motivation}(b) tests this 
\begin{wrapfigure}{r}{0.6\linewidth}
\vspace{-0cm}
\centering
\includegraphics[width=\linewidth]{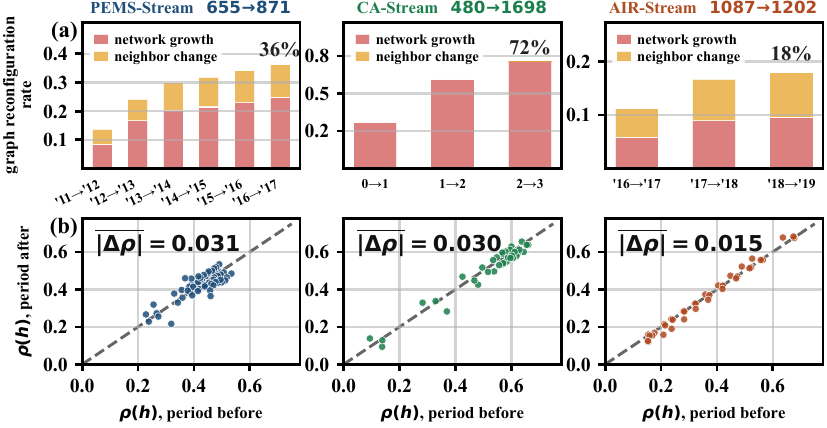}
\vspace{-0.68cm}
\caption{Motivation. (a)~Fraction of $k$NN edges reconfigured after each
expansion. (b)~Pairwise spatial correlation.}
\label{fig:motivation}
\vspace{-0.35cm}
\end{wrapfigure}
claim on the same three streams.
For each pair of consecutive periods, we compute the mean correlation $\rho(h)$ between sensor pairs at separation $h$ and plot the values before and after expansion on the $x$ and $y$ axes, respectively. Points near the diagonal indicate similar correlations across periods.
We first remove the shared time-of-day cycle to reduce its influence on
the measured correlations and use a fixed sensor set shared across all
periods to exclude changes in sensor composition.
Across the three datasets, the average shift never exceeds $0.031$, even
where the sensor count grows by a factor of $3.5$.
These results suggest that sensor expansion changes how densely the
underlying process is sampled, without necessarily changing the spatial dependencies that to be learned.

The contrast between panels (a) and (b) motivates a separation that graph
representations do not naturally provide:
\textbf{\textit{decouple knowledge of underlying spatial dynamics
from how it interfaces with a particular observation layout.}}
Existing methods typically adapt spatial representations defined over
the current sensor set, coupling layout changes with adaptation to
underlying dynamics.
Ideally, when the layout changes while spatial dynamics remain similar,
previously learned forecasting knowledge should remain reusable across
layouts.
The function-space perspective of neural operators
\citep{li2020fourier,lu2021learning} offers a natural basis for this
separation: sensor measurements provide observations of a
spatio-temporal field at given coordinates, while a shared operator
models its evolution. New sensors then contribute additional observations without requiring changes to the operator's parameterization. 
However, separating sensor layouts from underlying dynamics does not eliminate the need to adapt to \textbf{\textit{drift in spatial dynamics}}. The balance between broad spatial trends and localized variation can shift across periods, requiring the model to adjust its spatial responses. The resulting design principle is to \textbf{\textit{maintain a fixed computational domain for knowledge reuse while adapting field evolution to changing spatial dynamics}}.

Motivated by this principle, we introduce \model{}, the
\textbf{\textit{Spatio-Temporal Field Operator}}, which realizes the
separation above through two complementary mechanisms.
\ding{182} \textbf{\textit{Sensor-independent field representation.}}
\model{} maps observations from changing sensor layouts onto a fixed latent grid through coordinate-based aggregation, representing forecasting knowledge without sensor-specific parameters. A coordinate-based decoder queries the evolved field at sensor locations and combines the resulting features with local histories to produce
forecasts. The latent grid and parameter shapes remain fixed during fine-tuning
across periods.
\ding{183} \textbf{\textit{Drift-adaptive field evolution.}}
To accommodate changes in the underlying spatial dynamics, a compact spectral descriptor summarizes how variation is distributed across spatial scales and conditions the operator's channels on the current regime. 
A dual-path operator combines Fourier propagation for broad spatial structure with state-dependent attention for localized variation, enabling field evolution to adapt to changing dynamics.

Our contributions are threefold:
\ding{182} We reveal a previously underexplored mismatch in continual spatio-temporal forecasting between changes in observation layout and changes in the underlying spatial dynamics, and propose \model{} to separate them through sensor-independent field representation and drift-adaptive field evolution.
\ding{183} For fixed neural maps and coordinates, we establish sensor-sampling consistency under explicit sampling assumptions and latent-grid consistency for a fixed spectral design. We separately quantify the effect of coordinate reassignment when the sensor set changes;
\ding{184} We validate our method on three real-world datasets, demonstrating substantial improvements in forecasting accuracy over conventional graph-based continual learning methods.

\section{Problem Formulation}
\label{sec:prelim}
Continual forecasting proceeds through periods $p=1,\ldots,P$, with nested sensor sets $\mathcal V_p\subseteq\mathcal V_{p+1}$ and $N_p=|\mathcal V_p|$. Each period supplies observations from the current sensor set and updates the parameters learned in the previous period. Table~\ref{tab:notation} in Appendix~\ref{app:notation} lists the notation in this paper.

\textbf{\textit{Coordinates and observations.}}
Sensor $i$ is assigned a coordinate $\bm x_i^{(p)}=\chi_p(\bm s_i)\in\Omega=[0,1]^2$,
where $\bm s_i$ contains static descriptors such as geographic location and available attributes. The mapping $\chi_p$ depends on the current sensor set and is used consistently for encoding and querying within each period. Appendix~\ref{app:setup} details the coordinate assignment procedure.
Omitting $p$ below, let $u_t:\Omega\to\R$ denote the spatial field and
$y_{i,t}=u_t(\bm x_i)+\varepsilon_{i,t}$ its noisy observation.
The input set is $\mathcal D=\{(\bm x_i,\bm u_i)\}_{i=1}^{N}$,
where $\bm u_i=[y_{i,1},\ldots,y_{i,T_h}]^\top$ is sensor $i$'s
observation history.
For an observed query sensor $q$ with history $\bm u_q$, the forecasting
operator produces $\hat{\bm y}_q=\Gop_\theta(\mathcal D)(\bm x_q)\in\R^{T_f}$,
whose $k$-th component estimates $u_{T_h+k}(\bm x_q)$.

\textbf{\textit{Discretization consistency.}}
\label{def:invariance}
For a fixed coordinate representation, fixed neural maps, and a noiseless field history, consider sensor layouts $\mathcal S\subset\Omega$ with a fill distance $h_{\mathcal S}=\sup_{\bm x\in\Omega}\min_{\bm x_i\in\mathcal S}\|\bm x-\bm x_i\|_2$.
Consistency requires predictions to converge to a common reference under
an admissible sampling refinement:
$\big\|\Gop_\theta(\mathcal D_{\mathcal S})(\bm x_q)-\overline{\Gop}_\theta(\bm x_q)\big\|_2\longrightarrow0$. 
Theorem~\ref{thm:convergence} gives sampling bounds for fixed and shrinking
bandwidths; Corollary~\ref{cor:noforget} compares layouts sharing the same
reference, and Theorem~\ref{thm:grid_consistency} bounds the error under
latent-grid refinement.
Appendix~\ref{app:proofs} provides the proofs and separately analyzes
coordinate reassignment.

\section{Methodology}
\label{sec:method}

\model{} transfers forecasting knowledge through a shared latent field, as illustrated in Figure~\ref{fig:framework}.
The pipeline comprises four stages: \CFE aggregates observations into
spatial features; \SRE summarizes their spectral regime; \DFO evolves
the conditioned field; and \CQD decodes forecasts at sensor locations.
We use a fixed Cartesian grid $\Xi_{h_g}=\{\bm\xi_g\}_{g=1}^{G}$ with
$n_g$ points per axis, spacing $h_g=1/(n_g-1)$, and $G=n_g^2$.
Each period uses a $G\times d$ latent field, where $d$ is the hidden width. Sensor expansion changes the input observation set, while all temporal, spatial, and output maps remain shared across sensors.

\begin{figure}
    \centering
    \includegraphics[width=\linewidth]{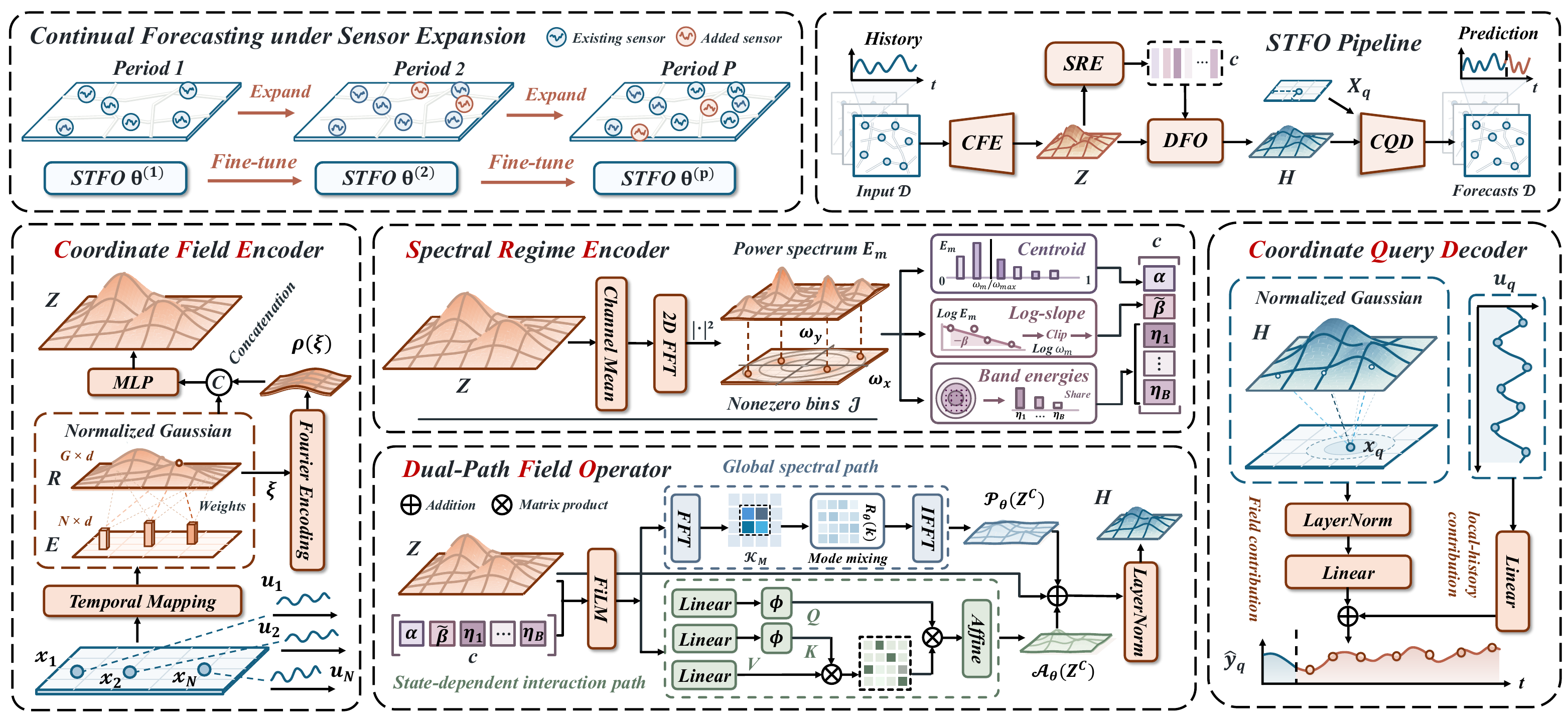}
    \vspace{-0.6cm}
    \caption{Overview of \model. The upper panels illustrate continual forecasting under sensor expansion and the overall pipeline; the lower panels detail \CFE, \SRE, \DFO, and \CQD.}
    \label{fig:framework}
    \vspace{-0.4cm}
\end{figure}

\subsection{Coordinate Field Encoder}
\label{sec:cfe}

The encoder maps observations from irregular sensor layouts onto a fixed latent grid. A shared temporal map $\tau_\theta:\R^{T_h}\to\R^d$
encodes each sensor's history as $\bm e_i=\tau_\theta(\bm u_i)$.
To lift these codes, define the Gaussian kernel $\kappa_\sigma(\bm x,\bm x')=\exp[-\|\bm x-\bm x'\|_2^2/(2\sigma^2)]$. At each grid location, the normalized weighted sum is
\begin{equation}
\bm r_g=\sum_{i=1}^{N}a_{gi}\bm e_i,
\qquad
 a_{gi}=\frac{\kappa_{\sigma_{\mathrm{enc}}}(\bm\xi_g,\bm x_i)}{\sum_{j=1}^{N}\kappa_{\sigma_{\mathrm{enc}}}(\bm\xi_g,\bm x_j)}.
\label{eq:cfe_lift}
\end{equation}
The learned bandwidth $\sigma_{\mathrm{enc}}>0$ controls the spatial
extent of aggregation. The normalized weights give greater emphasis
to nearby sensors and keep the lift a weighted average as the sensor
count changes. The aggregation is invariant to sensor ordering.
Lemma~\ref{lem:fixed_bandwidth} bounds the resulting lift error by the distance between sampling measures at a fixed bandwidth, and Lemma~\ref{lem:encoder_interp} shows that the error decays linearly in the fill distance when the bandwidth shrinks with it.

The field retains incorporates position information through $\bm z_g=\psi_\theta([\bm r_g\,\|\,\rho(\bm\xi_g)])$, where $\rho$ is a fixed Fourier coordinate encoding \citep{tancik2020fourierfeatures} and $\psi_\theta:\R^{d+d_\rho}\to\R^d$ is shared across the grid.\label{eq:cfe_post}
Position features allow grid locations with similar aggregated features to receive different latent representations.
In matrix form, the lift is $\mathbf{R}=\mathbf{A}\mathbf{E}$, with $[\mathbf{A}]_{gi}=a_{gi}$ and $[\mathbf{E}]_{i,:}=\bm e_i^\top$.  
Since $\mathbf{A}$ has nonnegative entries and each row sums to one,
each $\bm r_g$ is a convex combination of sensor codes, and the induced matrix norm satisfies $\|\mathbf{A}\|_\infty=1$ for any sensor count.
Stacking the resulting representations gives
$\mathbf{Z}=[\bm z_1,\ldots,\bm z_G]^\top\in\R^{G\times d}$.
Thus, the sensor count determines the input dimension of aggregation,
while the latent forecasting state retains a fixed shape $G\times d$.

\subsection{Spectral Regime Encoder}
\label{sec:sre}

The spectral regime describes how variation is distributed over spatial scales. Broad spatial patterns concentrate energy at low frequencies, while fine-scale variation contributes to higher frequencies. 
\SRE captures this distribution by averaging latent channels, $\bar z_g=d^{-1}\mathbf{1}_d^\top\bm z_g$, and transforming the resulting scalar grid $\bar z=\{\bar z_{\bm j}\}$, where $\bm j\in\{0,\ldots,n_g-1\}^2$ indexes grid entries. 
We use a forward-normalized discrete Fourier transform with the corresponding inverse:
\begin{equation}
\widehat{\bar z}(\bm k)=\frac1G\sum_{\bm j}\bar z_{\bm j}e^{-2\pi\mathrm{i}\bm k\cdot\bm j/n_g},
\qquad
\bar z_{\bm j}=\sum_{\bm k}\widehat{\bar z}(\bm k)e^{2\pi\mathrm{i}\bm k\cdot\bm j/n_g}.
\label{eq:dft_convention}
\end{equation}
Both transforms use phase nodes $\bm j/n_g$, whereas the latent-grid
coordinates are $\bm j/(n_g-1)$.
Appendix~\ref{app:grid_proof} accounts for this difference in the
grid-consistency proof.

For a set $\mathcal J$ of nonzero Fourier bins with wavevectors
$\bm\omega_m$, define
$E_m=|\widehat{\bar z}(\bm\omega_m)|^2+\epsilon_s$ and
$\omega_m=\|\bm\omega_m\|_2+\epsilon_f$,
where $\epsilon_s,\epsilon_f>0$ are numerical stabilizers.
Let $\lambda_m=\log\omega_m$, $\mu_m=\log E_m$, and use an overbar
to denote the average over $\mathcal J$, e.g.,
$\bar\lambda=|\mathcal J|^{-1}\sum_{m\in\mathcal J}\lambda_m$.
We partition $\mathcal J$ into $B$ logarithmically spaced radial
frequency bands $\mathcal B_b$.
Three complementary summaries form the regime descriptor:
\begin{equation}
\begin{aligned}
\alpha&=\frac{\sum_{m\in\mathcal J}E_m\omega_m}{\omega_{\max}\sum_{m\in\mathcal J}E_m+\epsilon_s},
&\eta_b&=\frac{\sum_{m\in\mathcal B_b}E_m}{\sum_{m\in\mathcal J}E_m+\epsilon_s},\\
\beta&=-\frac{\sum_{m\in\mathcal J}(\lambda_m-\bar\lambda)(\mu_m-\bar\mu)}{\sum_{m\in\mathcal J}(\lambda_m-\bar\lambda)^2+\epsilon_s},
&\bm c&=[\alpha,\tilde\beta,\eta_1,\ldots,\eta_B]^\top\in\R^{B+2}.
\end{aligned}
\label{eq:sre_statistics_main}
\end{equation}
Here $\omega_{\max}=\max_{m\in\mathcal J}\omega_m$ and
$\tilde\beta=\mathrm{clip}(\beta,-10,10)$.
The centroid $\alpha$ summarizes the relative frequency location of
energy; $\beta$ is a regularized negative log--log slope describing
radial spectral decay; and the band ratios $\eta_b$, $b=1,\ldots,B$,
capture energy contributions across spatial scales.
The positive stabilizers ensure continuity under latent-field
perturbations, while Theorem~\ref{thm:spectral} quantifies stability
under spectral-power perturbations.
For the grid-consistency analysis, we keep $\mathcal J$ and its band partition fixed across grid resolutions.

\subsection{Dual-Path Field Operator}
\label{sec:dfo}

Each operator block uses the spectral descriptor $\bm c$ to modulate
latent channels before spatial propagation.
A FiLM-style map \citep{perez2018film} produces
$(\bm\gamma,\bm\delta)
=0.1\tanh(\mathrm{MLP}_\theta(\bm c))
\in\R^d\times\R^d$.
The conditioned field is
$\mathbf{Z}^{\mathrm c}
=\mathbf{Z}\odot\big(\mathbf{1}_G(\mathbf{1}_d+\bm\gamma)^\top\big)
+\mathbf{1}_G\bm\delta^\top$,
where $\odot$ denotes elementwise multiplication.
Equivalently, each grid feature is updated as
$\bm z_g^{\mathrm c}
=\bm z_g\odot(\mathbf{1}_d+\bm\gamma)+\bm\delta$,
using the same channel gains and offsets at all grid locations.
The bounded gain and offset preserve a direct connection to the original field while adapting its channel emphasis. Stacked blocks have separate parameters and share the descriptor extracted from the initial field.

\paragraph{\textit{Global spectral propagation.}}
Following Fourier neural operators \citep{li2020fourier}, the Fourier path communicates through learned spatial modes. 
For a retained low-frequency set $\mathcal K_M$, its update is
\begin{equation}
\mathcal P_\theta(\mathbf{Z}^{\mathrm c})
=\mathcal F_{2D}^{-1}\!\left[
\mathbf 1\{\bm k\in\mathcal K_M\}\,\mathbf{R}_\theta(\bm k)\,
\mathcal F_{2D}[\mathbf{Z}^{\mathrm c}](\bm k)\right].
\label{eq:dfo_spectral_main}
\end{equation}
Equivalently, at grid location $g$,
\begin{equation}
\big(\mathcal P_\theta(\mathbf{Z}^{\mathrm c})\big)_g
=\sum_{\bm k\in\mathcal K_M}\mathbf{R}_\theta(\bm k)\,
\widehat{\mathbf{Z}}{}^{\mathrm c}(\bm k)\,
e^{2\pi\mathrm{i}\bm k\cdot\bm t_g},
\qquad
\widehat{\mathbf{Z}}{}^{\mathrm c}(\bm k)
=\frac1G\sum_{g=1}^{G}\bm z_g^{\mathrm c}
e^{-2\pi\mathrm{i}\bm k\cdot\bm t_g},
\label{eq:dfo_spectral_modes}
\end{equation}
where $\bm t_g$ is the phase node of grid entry $g$.
In general, $\mathbf{R}_\theta(\bm k)\in\mathbb C^{d\times d}$
mixes channels at each retained mode.
Our implementation uses the depthwise form
$\mathbf{R}_\theta(\bm k)=\operatorname{diag}(\bm r_\theta(\bm k))$,
which applies a learned complex multiplier to each channel.
Retaining modes in conjugate pairs with
$\mathbf{R}_\theta(-\bm k)=\overline{\mathbf{R}_\theta(\bm k)}$
and using real multipliers on self-conjugate modes ensure real outputs
for real inputs.
The cutoff $M$ determines the retained spatial frequencies.
Because each Fourier mode spans the domain, this branch propagates
broad spatial structure across the entire grid in one update.
Lemma~\ref{lem:grid_consistency} shows that this finite-mode branch stays within $O(h_g)$ of its continuum counterpart as the grid is refined.

\paragraph{\textit{State-dependent interactions.}}
The second path adapts spatial interactions to the current field features.
Define $\bm q_g=\phi(Q_\theta(\bm z_g^{\mathrm c}))$,
$\bm k_g=\phi(K_\theta(\bm z_g^{\mathrm c}))$, and
$\bm v_g=V_\theta(\bm z_g^{\mathrm c})$,
where $\phi(x)=\mathrm{ELU}(x)+1$,
$\bm q_g,\bm k_g\in\R^{d_a}$, and $\bm v_g\in\R^d$.
Normalized linear attention \citep{katharopoulos2020transformers} gives
\begin{equation}
\begin{aligned}
\mathcal A_\theta(\mathbf{Z}^{\mathrm c})_g
&=O_\theta\!\left(\frac{\sum_{j=1}^{G}(\bm q_g^\top\bm k_j)\bm v_j}{\sum_{j=1}^{G}\bm q_g^\top\bm k_j+\epsilon_a}\right)=O_\theta\!\left(\frac{(\mathbf{V}^\top\mathbf{K})\bm q_g}{\bm q_g^\top\mathbf{K}^\top\mathbf{1}_G+\epsilon_a}\right),
\end{aligned}
\label{eq:dfo_attention_main}
\end{equation}
where $\epsilon_a>0$ is a numerical stabilizer and $O_\theta$ is an
affine output map.
The rows of $\mathbf{Q},\mathbf{K}\in\R^{G\times d_a}$ and
$\mathbf{V}\in\R^{G\times d}$ contain the queries, keys, and values.
The second expression uses two aggregates shared by all queries:
$\mathbf{V}^\top\mathbf{K}
=\sum_{j=1}^{G}\bm v_j\bm k_j^\top\in\R^{d\times d_a}$
and
$\mathbf{K}^\top\mathbf{1}_G
=\sum_{j=1}^{G}\bm k_j\in\R^{d_a}$.
Multiplying these aggregates by $\bm q_g$ recovers the weighted value
sum and the normalization term in the first expression.
Forming the aggregates costs $O(Gdd_a)$ and supplies all queries. 
The positive feature map produces positive interaction scores, whose
scale is controlled by normalization.
For bounded features and fixed neural maps,
Lemma~\ref{lem:attention_lipschitz} establishes a Lipschitz bound
independent of $G$.

The block output is
\begin{equation}
\mathbf{H}=\LN\!\big(\mathbf{Z}+\mathcal P_\theta(\mathbf{Z}^{\mathrm c})+\mathcal A_\theta(\mathbf{Z}^{\mathrm c})\big)+\mathbf{1}_G\bm b_{\mathrm{res}}^\top,
\label{eq:dfo_residual}
\end{equation}
where $\LN$ normalizes channels independently at each grid location
and $\bm b_{\mathrm{res}}\in\R^d$ is a learned bias shared across
the grid.
Equivalently,
$\bm h_g=\LN(\bm z_g
+\mathcal P_\theta(\mathbf{Z}^{\mathrm c})_g
+\mathcal A_\theta(\mathbf{Z}^{\mathrm c})_g)
+\bm b_{\mathrm{res}}$.
This update combines a direct feature-preserving path, global spectral propagation, and content-dependent spatial mixing. Their roles are complementary: spectral weights encode shared propagation patterns, and attention changes its affinities with the observed field. The unmodulated residual provides a reference for both updates. 
Lemma~\ref{lem:lipschitz_modules} combines the branchwise bounds,
accounting for shared spectral conditioning, to establish stability
of the full operator stack.

\subsection{Coordinate Query Decoder}
\label{sec:cqd}

The decoder maps the evolved field to predictions at observed query
sensors.
Using the Gaussian kernel, define
$b_{qg}=\kappa_{\sigma_{\mathrm{dec}}}(\bm x_q,\bm\xi_g)/
\sum_{r=1}^{G}\kappa_{\sigma_{\mathrm{dec}}}(\bm x_q,\bm\xi_r)$
and gather
$\tilde{\bm h}_q=\sum_{g=1}^{G}b_{qg}\bm h_g
=\mathbf{H}^\top\bm b_q$,
where $\bm b_q=[b_{q1},\ldots,b_{qG}]^\top$ satisfies
$b_{qg}\ge0$ and $\mathbf{1}_G^\top\bm b_q=1$.
The learned decoder bandwidth $\sigma_{\mathrm{dec}}>0$ controls the readout scale.
Shared prediction maps combine field and local-history contributions:
\begin{equation}
\hat{\bm y}_q=
\underbrace{\mathbf{W}_{\mathrm{out}}\LN(\tilde{\bm h}_q)+\bm b_{\mathrm{out}}}_{\text{\textit{field contribution }}\bm f_q}
+\underbrace{\mathbf{W}_{\mathrm{skip}}\bm u_q+\bm b_{\mathrm{skip}}}_{\text{\textit{local-history contribution}}}.
\label{eq:cqd_readout}
\end{equation}
Here $\mathbf{W}_{\mathrm{out}}\in\R^{T_f\times d}$ and $\mathbf{W}_{\mathrm{skip}}\in\R^{T_f\times T_h}$. 
We initialize the local-history map with
$\mathbf{W}_{\mathrm{skip}}=\mathbf{1}_{T_f}\bm e_{T_h}^{\top}$
and zero bias, where $\bm e_{T_h}$ is the last standard basis vector
in $\R^{T_h}$. This initialization repeats the latest observation across all forecast horizons.
This decomposition lets the local-history path express sensor-specific temporal persistence and the field path supply spatially informed corrections. 
Lemma~\ref{lem:decoder_interp} bounds the gather error of the normalized weights by $O(h_g)$.

\subsection{Continual Update and Consistency}
\label{sec:continual}

We train \model{} from scratch in the first period.
For each subsequent period $p>1$, we initialize the model from
$\theta^{(p-1)}$ and fine-tune the same parameter set on current
training samples.
New sensors enter through their histories and period-specific coordinates.
The latent grid and all parameter shapes, including those of the Fourier
multipliers, remain fixed, allowing the learned operator to be reused and fine-tuned across periods.
Algorithms~\ref{alg:forward} and~\ref{alg:continual} in
Appendix~\ref{app:algorithm} describe the per-period forward pass and
continual training protocol, respectively.
Appendix~\ref{app:setup} provides the masked MAE objective and
optimization settings.

The analysis follows this computational structure. Kernel-lift errors propagate through a Lipschitz post-network (Lemmas~\ref{lem:fixed_bandwidth} and~\ref{lem:encoder_interp}), a stable spectral descriptor (Theorem~\ref{thm:spectral}), a finite stack of operator blocks (Lemmas~\ref{lem:attention_lipschitz} and~\ref{lem:lipschitz_modules}), and a normalized decoder (Lemma~\ref{lem:decoder_interp}). With fixed bandwidth and a common limiting sampling measure, prediction discrepancy from the reference is bounded by the distance between sampling measures; with local bandwidth proportional to fill distance, it is $O(h_{\mathcal S})$.  For a fixed spectral design and decoder bandwidth proportional to grid spacing, the grid error is $O(h_g)$, where the continuum comparison rests on endpoint-grid quadrature (Lemma~\ref{lem:riemann}) and on the grid stack (Lemma~\ref{lem:grid_consistency}). Theorems~\ref{thm:convergence} and~\ref{thm:grid_consistency} state the resulting bounds, Corollary~\ref{cor:noforget} compares two layouts under a shared reference, and Appendix~\ref{app:remapping} separates the effect of coordinate reassignment.

\section{Experiments}
\label{sec:exp}

We investigate the following research questions:
\begin{itemize}[leftmargin=1.5em]
\item \textbf{\textit{RQ1:}} How accurate is the proposed method across different datasets and horizons?
\item \textbf{\textit{RQ2:}} How much does each model component contribute to the overall forecasting performance?
\item \textbf{\textit{RQ3:}} How does the proposed method compare in terms of forecasting accuracy, training efficiency, and memory usage?
\item \textbf{\textit{RQ4:}} What spatial and channel-wise structures characterize the learned field updates, and how do spectral conditioning and channel modulation relate to forecasting improvements?
\item \textbf{\textit{RQ5:}} How well does the proposed method forecast at newly introduced sensors while retaining performance on existing sensors under different adaptation settings?
\item \textbf{\textit{RQ6:}} How sensitive is the proposed method to operator depth, grid resolution, and the number of retained modes? We report this sensitivity study in Appendix~\ref{app:sensitivity}.
\end{itemize}

\subsection{Experimental Setup}
\label{sec:exp:setup}
\textbf{\textit{Datasets.}}
We use \textsc{PEMS-Stream}~\citep{chen2021trafficstream},
\textsc{CA-Stream}~\citep{liu2023largest}, and
\textsc{AIR-Stream}~\citep{chen2025eac}.
\textsc{PEMS-Stream} grows from $655$ to $871$ traffic sensors over seven
yearly periods; \textsc{CA-Stream} grows from $480$ to $1{,}698$ sensors
over four monthly periods; \textsc{AIR-Stream} grows from $1{,}087$ to
$1{,}202$ air-quality stations over four yearly periods.

\textbf{\textit{Baselines.}}
We group the baselines by modeling and adaptation paradigm:
\ding{182} General spatio-temporal forecasting methods include
\textsc{GWNet}~\citep{wu2019graph},
\textsc{STID}~\citep{shao2022stid}, and
\textsc{iTransformer}~\citep{liu2023itransformer};
\ding{183} Neural-operator baselines include
\textsc{DeepONet}~\citep{lu2021learning} and
\textsc{FNO}~\citep{li2020fourier};
\ding{184} Continual-learning methods for evolving spatio-temporal
systems include
\textsc{TrafficStream}~\citep{chen2021trafficstream},
\textsc{STKEC}~\citep{wang2023knowledge},
\textsc{PECPM}~\citep{wang2023pattern},
\textsc{STRAP}~\citep{zhang2026strap},
\textsc{EAC}~\citep{chen2025eac}, and
\textsc{EVOLVE}~\citep{yang2027evolve};
\ding{185} Online adaptation methods include
\textsc{DOL}~\citep{wang2025dol} and
\textsc{ST-TTC}~\citep{chen2025stttc}, which update or calibrate
predictions using streaming observations at deployment time.
Appendix~\ref{app:baselines} describes each baseline and its role
in the comparison.

\textbf{\textit{Protocol and metrics.}}
For subsequent periods, we initialize the model from the previous
checkpoint and fine-tune on current training data.
Sensor histories and coordinates reflect the current sensor set, while
the latent grid and parameter shapes remain fixed.
Within each dataset, Small, Medium, and Large differ only in hidden width.
We report MAE, RMSE, and MAPE at horizons $3$, $6$, and $12$, and their
averages over $12$ horizons.
Main results are averaged over three runs with random seeds
$42$, $43$, and $44$.
Appendices~\ref{app:setup}, \ref{app:algorithm}, and \ref{app:notation}
provide experimental details, algorithms, and notation.

\subsection{All-Horizon Forecasting Results (RQ1)}
\label{sec:exp:main}
\begin{table*}[t]
\centering

\providecommand{\stbpstd}[1]{%
  {\normalfont\fontsize{4.6}{5.0}\selectfont\ensuremath{\pm#1}}%
}
\caption{Forecasting results across datasets and horizons, with methods grouped by forecasting paradigm. Results are mean $\pm$ standard deviation over three runs. Lower is better; \best{bold} and \second{underlined} denote the best and second-best methods within each dataset, metric, and horizon.}
\label{tab:main_horizon}

\setlength{\tabcolsep}{1.25pt}
\renewcommand{\arraystretch}{1.04}
\fontsize{6.4}{7.4}\selectfont

\resizebox{\linewidth}{!}{%
\begin{tabular}{
@{}ll
ccc
cc
cccccc
ccc
ccc@{}
}
\toprule

\multirow{2}{*}{\textbf{Metric}}
&
\multirow{2}{*}{\textbf{Hor.}}
&
\multicolumn{3}{c}{\textbf{Traditional}}
&
\multicolumn{2}{c}{\textbf{Neural Operators}}
&
\multicolumn{6}{c}{\textbf{Continual Learning}}
&
\multicolumn{2}{c}{\textbf{Online Adaptation}}
&
\multicolumn{3}{c}{\textbf{\model} \textbf{\textit{(Ours)}}}
\\[-0.15em]

\cmidrule(lr){3-5}
\cmidrule(lr){6-7}
\cmidrule(lr){8-13}
\cmidrule(lr){14-15}
\cmidrule(lr){16-18}

&
&
GWNet & STID & iTrans
&
DeepONet & FNO
&
TStream & STKEC & PECPM & STRAP & EAC & EVOLVE
&
DOL & ST-TTC
&
Small & Medium & Large
\\
\midrule

\rowcolor{datasetgray}
\multicolumn{18}{@{}l}{\hspace{0.25em}\textbf{\textit{PEMS-Stream}}}
\\[-0.1em]

\multirow{4}{*}{\textbf{MAE}}
& 3
& \hval{19.64}{10.76}{24.34}\stbpstd{0.12}
& \hval{24.34}{10.76}{24.34}\stbpstd{0.13}
& \hval{17.63}{10.76}{24.34}\stbpstd{0.76}
& \hval{12.25}{10.76}{24.34}\stbpstd{0.02} & \hval{12.63}{10.76}{24.34}\stbpstd{0.02}
& \hval{14.23}{10.76}{24.34}\stbpstd{0.09}
& \hval{14.29}{10.76}{24.34}\stbpstd{0.12}
& \hval{14.26}{10.76}{24.34}\stbpstd{0.13}
& \hval{13.09}{10.76}{24.34}\stbpstd{0.01}
& \hval{13.86}{10.76}{24.34}\stbpstd{0.16}
& \hval{13.30}{10.76}{24.34}\stbpstd{0.09}
& \hval{11.40}{10.76}{24.34}\stbpstd{0.03} & \hval{12.94}{10.76}{24.34}\stbpstd{0.11}
& \hval{11.26}{10.76}{24.34}\stbpstd{0.02} & \hsecond{10.94}{10.76}{24.34}\stbpstd{0.02} & \hbest{10.76}{10.76}{24.34}\stbpstd{0.02}
\\
\bottomrule
\end{tabular}%
}
\vspace{-0.2cm}
\end{table*}

Table~\ref{tab:main_horizon} shows that \model-Large achieves the lowest
errors across all reported horizons and metrics on both traffic streams,
reducing average MAE and RMSE relative to \textsc{DOL} by about $8\%$
on PEMS-Stream and $5\%$ on CA-Stream.
Relative gains over \textsc{DOL} increase with the prediction horizon
and are largest at horizon $12$.
Scaling from Small to Large consistently improves performance on the
two traffic streams, whereas improvements on AIR-Stream are not monotonic.
Wilcoxon signed-rank tests with Holm correction in
Appendix~\ref{app:statistical_significance} support the overall
forecasting advantage of \model-Large over the baselines.

\subsection{Ablation Study (RQ2)}
\label{sec:exp:ablation}

Figure~\ref{fig:stfo_ablation} compares \model-Large with four ablated
variants on PEMS-Stream and AIR-Stream:
w/o SRE removes spectral conditioning,
w/o DFO-S removes the spectral path,
w/o DFO-A removes the attention path, and
w/o CFE-C removes coordinate-feature injection.
\begin{wrapfigure}{r}{0.6\linewidth}
\vspace{-0.4cm}
\includegraphics[width=\linewidth]{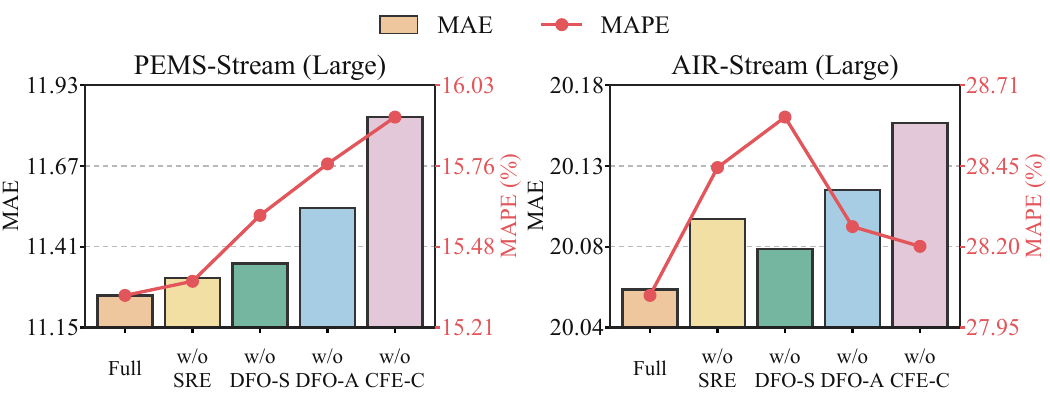}
\vspace{-0.8cm}
\caption{Component ablation of \model-Large on PEMS-Stream and AIR-Stream. Results are from single-seed runs.}
\vspace{-0.4cm}
\label{fig:stfo_ablation}
\end{wrapfigure}
The full model has the lowest MAE and MAPE in both panels. Removing coordinate features produces the largest MAE increase on both streams and the largest MAPE increase on PEMS-Stream, highlighting the role of explicit location information in the lifted field. 
On AIR-Stream, removing the spectral path causes the largest MAPE
increase but the smallest MAE increase among the ablations.
These differing rankings suggest that component contributions depend
on the dataset and evaluation metric.
The performance degradation after removing either DFO path further
suggests that both spectral propagation and state-dependent interactions contribute to forecasting accuracy.
\subsection{Efficiency Analysis (RQ3)}
\label{sec:exp:efficiency}
\begin{wrapfigure}{r}{0.5\linewidth}
\vspace{-0.9cm}
\includegraphics[width=\linewidth]{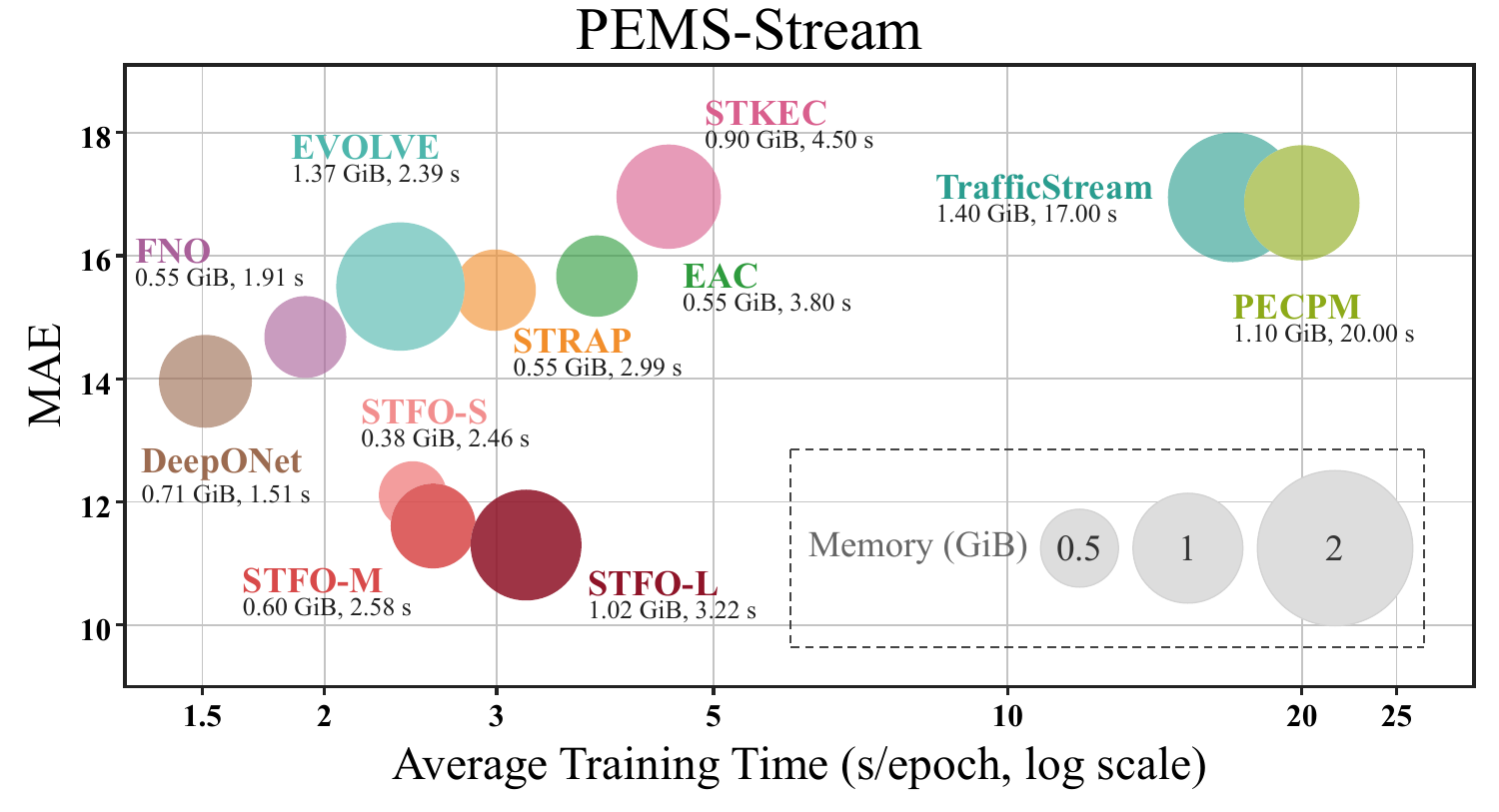}
\vspace{-0.8cm}
\caption{Forecasting accuracy, training time per epoch, and GPU memory
usage on PEMS-Stream.}
\label{fig:efficiency}
\vspace{-0.4cm}
\end{wrapfigure}
Figure~\ref{fig:efficiency} compares forecasting accuracy, training
time per epoch, and GPU memory usage on PEMS-Stream.
All three \model{} configurations achieve lower MAE than the plotted
baselines, offering favorable accuracy--cost trade-offs across model sizes.
\model-Small achieves strong forecasting performance with
$2.46$ seconds per epoch and $0.38$ GiB of GPU memory.
Increasing model capacity further improves accuracy, with \model-Large
achieving the lowest MAE at $3.22$ seconds per epoch and $1.02$ GiB
of memory.
The larger configurations therefore improve accuracy while retaining
competitive computational efficiency.
Across these configurations, capacity increases through hidden width
without enlarging the latent grid, while linear attention avoids
constructing and storing a full grid-to-grid attention matrix.
Combining a fixed latent grid with linear attention supports increased
model capacity through wider hidden representations while keeping
spatial computation efficient.
\subsection{Mechanism Diagnostics (RQ4)}
\label{sec:exp:mechanism}
\label{sec:exp:operator}
\label{sec:exp:spectral}
\begin{figure}
\centering
\includegraphics[width=\linewidth]{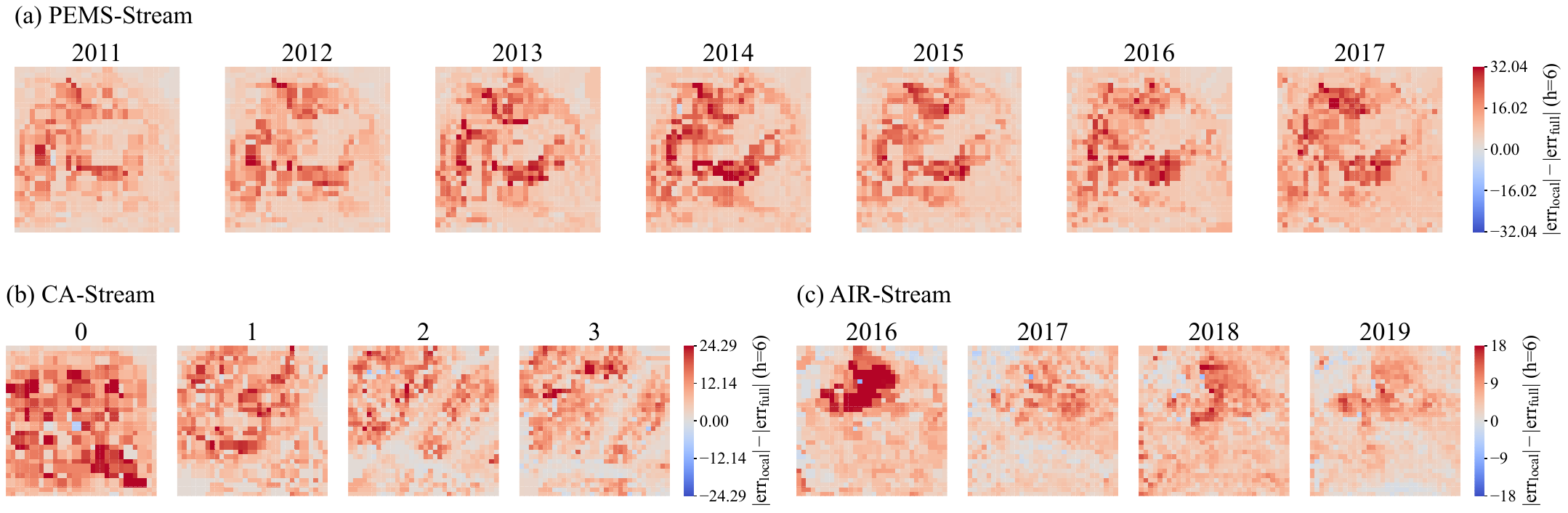}
\vspace{-0.7cm}
\caption{Sensor-wise error reduction on the assignment grid across three streams and all periods at horizon $6$: $|\mathrm{err}_{\mathrm{local}}|-|\mathrm{err}_{\mathrm{full}}|$, where $\mathrm{err}_{\mathrm{local}}$ denotes the error of the local-history-only prediction. Positive values indicate improvements from adding the field contribution.}
\label{fig:spatial_error_reduction_h6}
\end{figure}

\begin{figure}
\centering
\includegraphics[width=\linewidth]{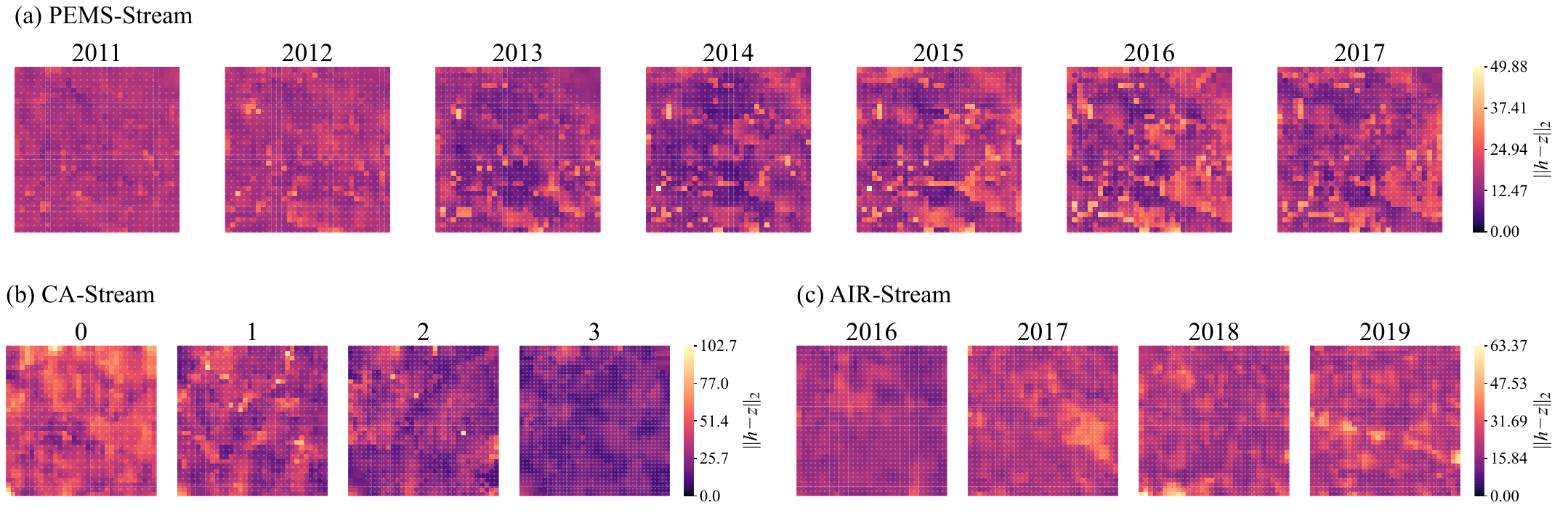}
\vspace{-0.5cm}
\caption{DFO update magnitudes on the latent grid across three streams and all periods. Each location $g$ shows $\|h_g-z_g\|_2$, where $z_g$ and $h_g$ denote the input and output latent features, respectively.}
\label{fig:operator_update_energy}
\vspace{-0.5cm}
\end{figure}

\ding{182} \textbf{\textit{At the spatial level,}}
forecasting gains and latent-field updates exhibit spatially structured patterns that evolve differently across datasets.
Figures~\ref{fig:spatial_error_reduction_h6}
and~\ref{fig:operator_update_energy} show error reductions relative
to the local-history skip at horizon $6$ and latent-field update
magnitudes, respectively.
On PEMS-Stream, forecasting gains remain broadly distributed across
the coordinate domain, while field updates exhibit more pronounced
local peaks in later periods.
CA-Stream shows more visible changes in the spatial distribution
of forecasting gains, while field-update magnitudes tend to decrease
across periods.
On AIR-Stream, both forecasting gains and field updates display
spatial structure, with their locations and intensities varying
across years.
These observations suggest that the shared field operator accommodates both persistent spatial structure and changes across periods, with the field contribution improving predictions at many sensor locations.

\ding{183} \textbf{\textit{At the channel level,}}
spectral conditioning produces modulation patterns that differ across
channels and periods.
Figure~\ref{fig:sre_film_correlations} shows positive and negative
correlations between SRE descriptor components and the FiLM coefficients
$\bm\gamma$ and $\bm\delta$, indicating distinct channel responses
to the spectral regime.
Figure~\ref{fig:film_gain_regions} further relates these modulation
patterns to forecasting gains.
PEMS-Stream and CA-Stream exhibit dispersed channel configurations
with mixed gain correlations, whereas AIR-Stream shows tighter
clusters with similar correlations within each period.
Together, these results suggest that spectral conditioning differentiates
channel responses, while their associations with forecasting gains
depend on the dataset and period.

\begin{figure}
\centering
\includegraphics[width=\linewidth]{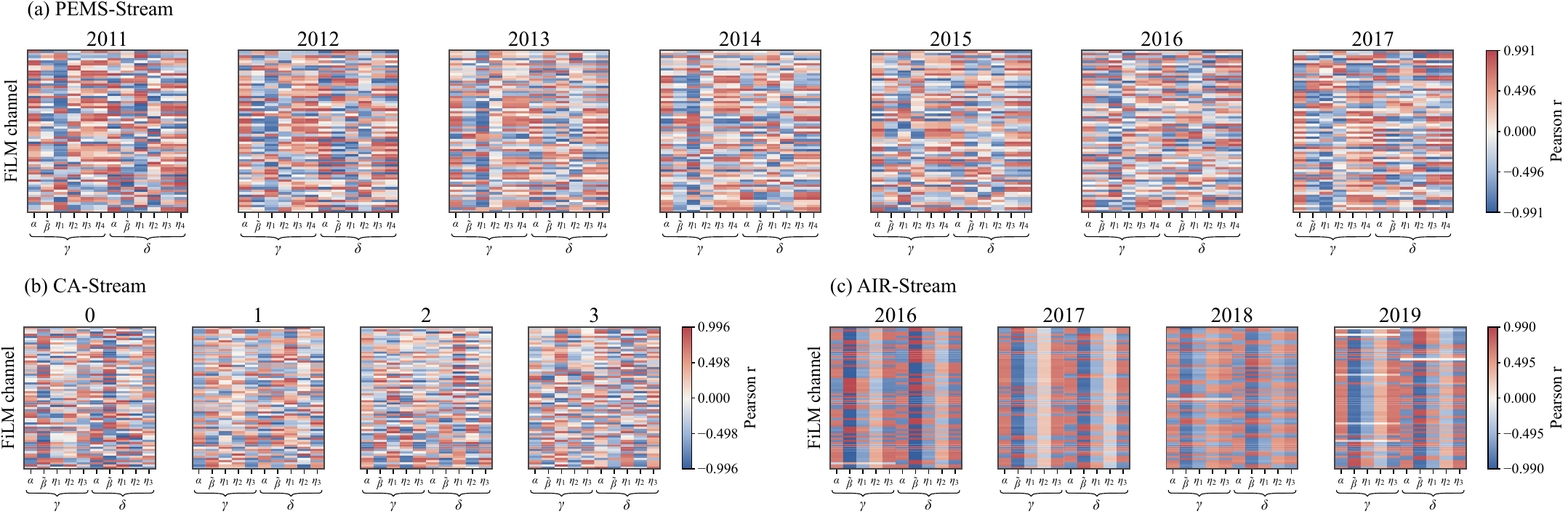}
\vspace{-0.5cm}
\caption{SRE-to-FiLM channel correlations across three streams. The heatmaps show correlations between SRE descriptor components and channel-wise FiLM parameters $\bm{\gamma}$ and $\bm{\delta}$.}
\label{fig:sre_film_correlations}
\end{figure}

\begin{figure}
\centering
\includegraphics[width=\linewidth]{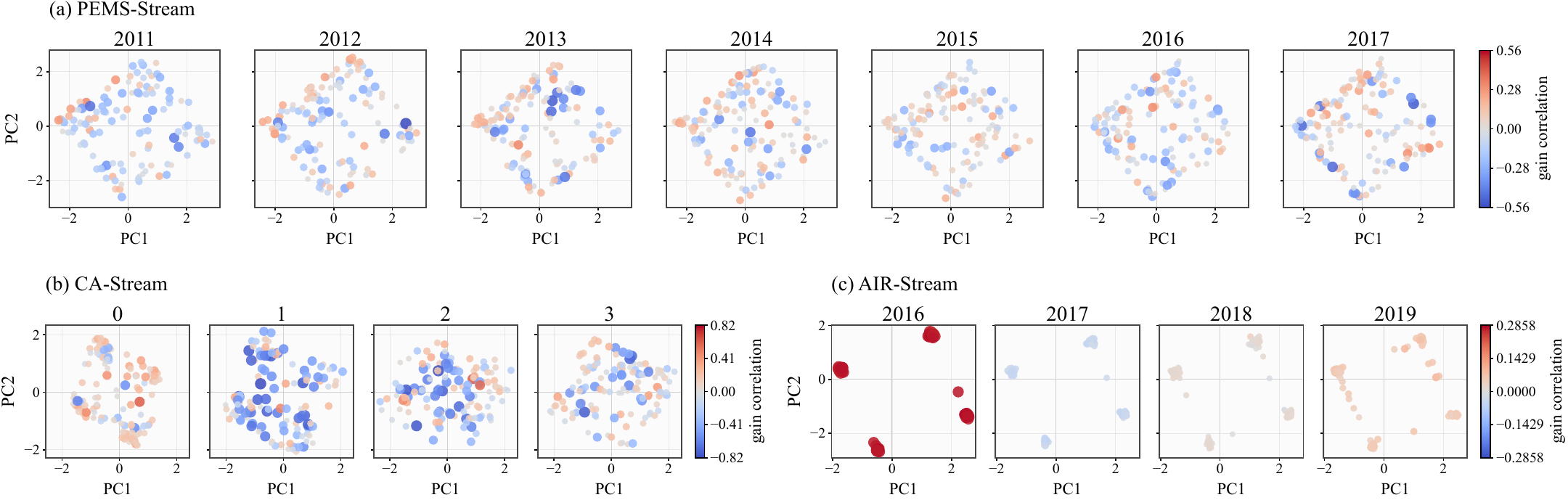}
\vspace{-0.5cm}
\caption{FiLM channel projections across three streams and periods. Channels are projected by PCA from concatenated $\gamma_j$--SRE and $\delta_j$--SRE correlations using a shared basis within each stream. Color and size respectively encode the signed and absolute correlations of modulation strength $|\gamma_j|+|\delta_j|$ with forecasting gain $\mathrm{MAE}_{\mathrm{local}}-\mathrm{MAE}_{\mathrm{full}}$ at horizon $6$.}
\vspace{-0.5cm}
\label{fig:film_gain_regions}
\end{figure}
\subsection{New-Sensor Adaptation (RQ5)}
\label{sec:exp:new_sensor_adaptation}

\providecommand{\stbpstd}[1]{%
  {\normalfont\fontsize{4.6}{5.0}\selectfont\ensuremath{\pm#1}}%
}
\begin{table*}[t]
\centering
\caption{Period-wise MAE on existing and newly introduced sensors of AIR-Stream. Results are mean $\pm$ standard deviation over three runs. Lower is better; \best{bold} and \second{underlined} denote the best and second-best methods within each period, sensor group, adaptation budget, and horizon.}
\label{tab:new_sensor_adaptation}
\setlength{\tabcolsep}{2.5pt}
\renewcommand{\arraystretch}{0.99}
\fontsize{6.2}{7.1}\selectfont
\resizebox{\linewidth}{!}{%
\begin{tabular}{@{}ll cccc cccc cccc@{}}
\toprule
\multirow{2}{*}{\textbf{Budget}} & \multirow{2}{*}{\textbf{Sensor group}}
& \multicolumn{4}{c}{FNO}
& \multicolumn{4}{c}{EVOLVE}
& \multicolumn{4}{c}{\textbf{\model-Large}}
\\[-0.15em]
\cmidrule(lr){3-6}\cmidrule(lr){7-10}\cmidrule(lr){11-14}
& & $h=3$ & $h=6$ & $h=12$ & \textit{Avg.} & $h=3$ & $h=6$ & $h=12$ & \textit{Avg.} & $h=3$ & $h=6$ & $h=12$ & \textit{Avg.} \\ 
\midrule
\rowcolor{datasetgray}
\multicolumn{14}{@{}l}{\hspace{0.25em}\textbf{2017}: 1,087 existing sensors; 67 newly introduced sensors} \\[-0.1em]
\multirow{2}{*}{Zero-shot} & Existing & \hval{17.57}{16.93}{17.57}\stbpstd{0.26} & \hsecond{23.15}{22.54}{23.92}\stbpstd{0.19} & \hsecond{27.09}{26.91}{28.38}\stbpstd{0.14} & \hsecond{21.71}{21.26}{22.39}\stbpstd{0.15} & \hsecond{17.37}{16.93}{17.57}\stbpstd{0.05} & \hval{23.92}{22.54}{23.92}\stbpstd{0.03} & \hval{28.38}{26.91}{28.38}\stbpstd{0.06} & \hval{22.39}{21.26}{22.39}\stbpstd{0.04} & \hbest{16.93}{16.93}{17.57}\stbpstd{0.14} & \hbest{22.54}{22.54}{23.92}\stbpstd{0.26} & \hbest{26.91}{26.91}{28.38}\stbpstd{0.39} & \hbest{21.26}{21.26}{22.39}\stbpstd{0.26} \\ 
 & Newly introduced & \hsecond{18.12}{17.36}{18.12}\stbpstd{0.32} & \hsecond{24.11}{23.41}{25.53}\stbpstd{0.15} & \hsecond{29.13}{28.82}{31.46}\stbpstd{0.34} & \hsecond{22.80}{22.25}{24.10}\stbpstd{0.24} & \hsecond{18.12}{17.36}{18.12}\stbpstd{0.07} & \hval{25.53}{23.41}{25.53}\stbpstd{0.05} & \hval{31.46}{28.82}{31.46}\stbpstd{0.08} & \hval{24.10}{22.25}{24.10}\stbpstd{0.06} & \hbest{17.36}{17.36}{18.12}\stbpstd{0.18} & \hbest{23.41}{23.41}{25.53}\stbpstd{0.34} & \hbest{28.82}{28.82}{31.46}\stbpstd{0.45} & \hbest{22.25}{22.25}{24.10}\stbpstd{0.33} \\ 
\cmidrule(lr){1-14}
\multirow{2}{*}{Full} & Existing & \hsecond{16.93}{16.61}{17.32}\stbpstd{0.11} & \hsecond{22.53}{22.06}{23.86}\stbpstd{0.17} & \hsecond{26.56}{26.24}{28.36}\stbpstd{0.21} & \hsecond{21.15}{20.80}{22.36}\stbpstd{0.14} & \hval{17.32}{16.61}{17.32}\stbpstd{0.21} & \hval{23.86}{22.06}{23.86}\stbpstd{0.12} & \hval{28.36}{26.24}{28.36}\stbpstd{0.12} & \hval{22.36}{20.80}{22.36}\stbpstd{0.17} & \hbest{16.61}{16.61}{17.32}\stbpstd{0.05} & \hbest{22.06}{22.06}{23.86}\stbpstd{0.04} & \hbest{26.24}{26.24}{28.36}\stbpstd{0.03} & \hbest{20.80}{20.80}{22.36}\stbpstd{0.02} \\ 
 & Newly introduced & \hsecond{17.32}{16.96}{17.96}\stbpstd{0.14} & \hsecond{23.40}{22.79}{25.36}\stbpstd{0.21} & \hsecond{28.61}{28.07}{31.44}\stbpstd{0.30} & \hsecond{22.17}{21.68}{24.01}\stbpstd{0.19} & \hval{17.96}{16.96}{17.96}\stbpstd{0.19} & \hval{25.36}{22.79}{25.36}\stbpstd{0.12} & \hval{31.44}{28.07}{31.44}\stbpstd{0.14} & \hval{24.01}{21.68}{24.01}\stbpstd{0.18} & \hbest{16.96}{16.96}{17.96}\stbpstd{0.07} & \hbest{22.79}{22.79}{25.36}\stbpstd{0.08} & \hbest{28.07}{28.07}{31.44}\stbpstd{0.09} & \hbest{21.68}{21.68}{24.01}\stbpstd{0.07} \\ 
\midrule
\rowcolor{datasetgray}
\multicolumn{14}{@{}l}{\hspace{0.25em}\textbf{2018}: 1,154 existing sensors; 39 newly introduced sensors} \\[-0.1em]
\multirow{2}{*}{Zero-shot} & Existing & \hsecond{14.56}{14.41}{15.16}\stbpstd{0.09} & \hsecond{19.56}{19.26}{20.41}\stbpstd{0.21} & \hsecond{23.10}{22.90}{23.99}\stbpstd{0.19} & \hsecond{18.42}{18.17}{19.22}\stbpstd{0.17} & \hval{15.16}{14.41}{15.16}\stbpstd{0.39} & \hval{20.41}{19.26}{20.41}\stbpstd{0.19} & \hval{23.99}{22.90}{23.99}\stbpstd{0.07} & \hval{19.22}{18.17}{19.22}\stbpstd{0.25} & \hbest{14.41}{14.41}{15.16}\stbpstd{0.06} & \hbest{19.26}{19.26}{20.41}\stbpstd{0.07} & \hbest{22.90}{22.90}{23.99}\stbpstd{0.09} & \hbest{18.17}{18.17}{19.22}\stbpstd{0.05} \\ 
 & Newly introduced & \hsecond{15.12}{14.67}{15.61}\stbpstd{0.14} & \hsecond{20.78}{19.71}{21.06}\stbpstd{0.35} & \hsecond{24.61}{23.52}{24.86}\stbpstd{0.30} & \hsecond{19.54}{18.60}{19.87}\stbpstd{0.32} & \hval{15.61}{14.67}{15.61}\stbpstd{0.43} & \hval{21.06}{19.71}{21.06}\stbpstd{0.22} & \hval{24.86}{23.52}{24.86}\stbpstd{0.10} & \hval{19.87}{18.60}{19.87}\stbpstd{0.28} & \hbest{14.67}{14.67}{15.61}\stbpstd{0.08} & \hbest{19.71}{19.71}{21.06}\stbpstd{0.14} & \hbest{23.52}{23.52}{24.86}\stbpstd{0.07} & \hbest{18.60}{18.60}{19.87}\stbpstd{0.06} \\ 
\cmidrule(lr){1-14}
\multirow{2}{*}{Full} & Existing & \hsecond{14.39}{14.17}{14.78}\stbpstd{0.07} & \hsecond{19.23}{19.03}{20.47}\stbpstd{0.15} & \hsecond{22.81}{22.59}{24.35}\stbpstd{0.08} & \hsecond{18.09}{17.87}{19.18}\stbpstd{0.06} & \hval{14.78}{14.17}{14.78}\stbpstd{0.04} & \hval{20.47}{19.03}{20.47}\stbpstd{0.13} & \hval{24.35}{22.59}{24.35}\stbpstd{0.19} & \hval{19.18}{17.87}{19.18}\stbpstd{0.10} & \hbest{14.17}{14.17}{14.78}\stbpstd{0.00} & \hbest{19.03}{19.03}{20.47}\stbpstd{0.01} & \hbest{22.59}{22.59}{24.35}\stbpstd{0.05} & \hbest{17.87}{17.87}{19.18}\stbpstd{0.02} \\ 
 & Newly introduced & \hsecond{14.95}{14.48}{15.16}\stbpstd{0.06} & \hsecond{20.45}{19.56}{21.09}\stbpstd{0.06} & \hsecond{24.48}{23.32}{25.25}\stbpstd{0.32} & \hsecond{19.22}{18.39}{19.79}\stbpstd{0.10} & \hval{15.16}{14.48}{15.16}\stbpstd{0.03} & \hval{21.09}{19.56}{21.09}\stbpstd{0.13} & \hval{25.25}{23.32}{25.25}\stbpstd{0.23} & \hval{19.79}{18.39}{19.79}\stbpstd{0.11} & \hbest{14.48}{14.48}{15.16}\stbpstd{0.03} & \hbest{19.56}{19.56}{21.09}\stbpstd{0.06} & \hbest{23.32}{23.32}{25.25}\stbpstd{0.05} & \hbest{18.39}{18.39}{19.79}\stbpstd{0.03} \\ 
\midrule
\rowcolor{datasetgray}
\multicolumn{14}{@{}l}{\hspace{0.25em}\textbf{2019}: 1,193 existing sensors; 9 newly introduced sensors} \\[-0.1em]
\multirow{2}{*}{Zero-shot} & Existing & \hval{15.38}{15.11}{15.38}\stbpstd{0.18} & \hsecond{19.83}{19.57}{20.71}\stbpstd{0.30} & \hsecond{23.18}{23.09}{24.38}\stbpstd{0.12} & \hsecond{18.85}{18.66}{19.53}\stbpstd{0.12} & \hsecond{15.36}{15.11}{15.38}\stbpstd{0.49} & \hval{20.71}{19.57}{20.71}\stbpstd{0.29} & \hval{24.38}{23.09}{24.38}\stbpstd{0.23} & \hval{19.53}{18.66}{19.53}\stbpstd{0.35} & \hbest{15.11}{15.11}{15.38}\stbpstd{0.06} & \hbest{19.57}{19.57}{20.71}\stbpstd{0.13} & \hbest{23.09}{23.09}{24.38}\stbpstd{0.11} & \hbest{18.66}{18.66}{19.53}\stbpstd{0.08} \\ 
 & Newly introduced & \hsecond{21.14}{21.07}{21.22}\stbpstd{0.46} & \hbest{26.97}{26.97}{27.98}\stbpstd{0.58} & \hbest{30.46}{30.46}{31.51}\stbpstd{0.48} & \hbest{25.29}{25.29}{26.12}\stbpstd{0.40} & \hval{21.22}{21.07}{21.22}\stbpstd{0.68} & \hval{27.98}{26.97}{27.98}\stbpstd{0.31} & \hval{31.51}{30.46}{31.51}\stbpstd{0.19} & \hval{26.12}{25.29}{26.12}\stbpstd{0.34} & \hbest{21.07}{21.07}{21.22}\stbpstd{0.08} & \hsecond{27.08}{26.97}{27.98}\stbpstd{0.29} & \hsecond{31.24}{30.46}{31.51}\stbpstd{0.53} & \hsecond{25.61}{25.29}{26.12}\stbpstd{0.29} \\ 
\cmidrule(lr){1-14}
\multirow{2}{*}{Full} & Existing & \hsecond{14.13}{14.08}{14.55}\stbpstd{0.13} & \hsecond{18.98}{18.78}{20.17}\stbpstd{0.26} & \hsecond{22.58}{22.37}{24.06}\stbpstd{0.12} & \hsecond{17.85}{17.69}{18.89}\stbpstd{0.16} & \hval{14.55}{14.08}{14.55}\stbpstd{0.08} & \hval{20.17}{18.78}{20.17}\stbpstd{0.03} & \hval{24.06}{22.37}{24.06}\stbpstd{0.04} & \hval{18.89}{17.69}{18.89}\stbpstd{0.06} & \hbest{14.08}{14.08}{14.55}\stbpstd{0.09} & \hbest{18.78}{18.78}{20.17}\stbpstd{0.10} & \hbest{22.37}{22.37}{24.06}\stbpstd{0.06} & \hbest{17.69}{17.69}{18.89}\stbpstd{0.05} \\ 
 & Newly introduced & \hsecond{20.18}{20.10}{20.47}\stbpstd{0.17} & \hbest{26.31}{26.31}{27.61}\stbpstd{0.28} & \hbest{29.85}{29.85}{31.62}\stbpstd{0.10} & \hbest{24.46}{24.46}{25.65}\stbpstd{0.14} & \hval{20.47}{20.10}{20.47}\stbpstd{0.17} & \hval{27.61}{26.31}{27.61}\stbpstd{0.15} & \hval{31.62}{29.85}{31.62}\stbpstd{0.29} & \hval{25.65}{24.46}{25.65}\stbpstd{0.21} & \hbest{20.10}{20.10}{20.47}\stbpstd{0.03} & \hsecond{26.41}{26.31}{27.61}\stbpstd{0.13} & \hsecond{30.87}{29.85}{31.62}\stbpstd{0.32} & \hsecond{24.80}{24.46}{25.65}\stbpstd{0.14} \\ 
\bottomrule
\end{tabular}%
}
\end{table*}

Table~\ref{tab:new_sensor_adaptation} compares FNO, EVOLVE, and
\model-Large on AIR-Stream under zero-shot and full adaptation.
Without current-period fine-tuning, \model-Large outperforms the
baselines in most new-sensor comparisons, supporting its ability
to transfer forecasting knowledge to newly introduced sensors.
Full adaptation further reduces MAE for both new and existing sensors
in every expansion period.
Under both settings, \model-Large achieves the lowest MAE in all
existing-sensor comparisons and most new-sensor comparisons.
These results suggest that knowledge transferred from previous periods
provides a useful starting point for new sensors, while current-period
adaptation further improves forecasting performance across the
expanded network.
Appendix~\ref{app:new_sensor_adaptation} provides additional results
across datasets and metrics.

\section{Conclusion}
\label{sec:conclusion}

In this work, we propose \model{}, combining coordinate-based lifting
and querying with spectrally conditioned Fourier and attention updates
in a shared latent field.
This formulation accommodates sensor expansion without changing
parameter shapes, with prediction consistency established under
sampling and grid-refinement assumptions.
Experiments demonstrate leading accuracy on both traffic streams
and competitive performance on AIR-Stream, supporting the shared
latent field as an effective representation for continual forecasting
under evolving sensor layouts and dynamics.

\clearpage

\subsection*{AI use statement}

Generative AI tools assisted with language revision.
\bibliography{references}
\bibliographystyle{iclr2027_conference}

\clearpage


\appendix

\makeatletter
\let\addcontentsline\stfo@addcontentsline
\makeatother

\clearpage
\phantomsection
\section*{Contents of the Appendix}
\makeatletter
\@starttoc{toc}
\makeatother

\section{Related Work}
\label{sec:related}

\textbf{\textit{Spatio-temporal forecasting models.}}
Graph-based models have played a central role in spatio-temporal
forecasting. Representative architectures combine graph convolution
with gated temporal convolution \citep{yu2018spatiotemporal},
integrate dynamic graph construction with recurrent modeling
\citep{li2022dynamic}, learn adaptive adjacency matrices
\citep{wu2019graph}, construct synchronous space--time graphs
\citep{song2020spatialtemporal}, learn directed graphs with mix-hop
propagation \citep{wu2020connecting}, and apply spatial and temporal
attention \citep{guo2019attention,zheng2020gman}.
Other approaches show that competitive forecasting does not always
require explicit graph construction: a vanilla transformer with
adaptive spatio-temporal embeddings can achieve strong performance
\citep{liu2023staeformer}, while inverted-attention and
frequency-debiased transformers further advance sequence-based
forecasting \citep{liu2023itransformer,piao2024fredformer}.
Recent work also pursues efficiency through spatio-temporal distillation
\citep{zhang2025efficient} and lightweight language models coupled
with adaptive hypergraphs \citep{chen2025decoupling}.
Despite differences in architecture, these methods generally organize
spatial computation around observed sensors or variables, so changes
in the observation set affect the representation used for forecasting.

\textbf{\textit{Continual spatio-temporal forecasting.}}
Continual forecasting methods extend spatio-temporal backbones to
retain useful knowledge as data streams evolve and new sensors
are introduced.
\citet{chen2021trafficstream} combine experience replay with parameter
smoothing for expanding road networks; subsequent work introduces
knowledge expansion and consolidation \citep{wang2023knowledge},
pattern expansion and consolidation \citep{wang2023pattern},
and continual prompt pools \citep{chen2025eac}.
Meanwhile, \citet{wang2024towards} decouple trend and seasonal patterns
and update them at different frequencies to adapt to evolving graphs.
Continual learning has also been studied on evolving graphs
\citep{yuan2026continual,galke2023lifelong}, while
\citet{ao2023continual} review continual deep learning for time-series
modeling more broadly.
Retrieval-augmented approaches store and retrieve historical patterns
to support forecasting under distribution shifts in streaming
spatio-temporal data \citep{11674299}.
These studies address adaptation to changing data and graph structure;
our work focuses on the dependence of forecasting representations
on sensor discretization.

\textbf{\textit{Neural operators and implicit representations.}}
Operator learning approximates mappings between functions.
The Fourier neural operator parameterizes translation-invariant integral
kernels in the spectral domain and supports resolution transfer for
parametric PDEs \citep{li2020fourier}, while DeepONet represents
operators through a branch--trunk architecture with universal
approximation guarantees \citep{lu2021learning}.
Later work extends spectral operators \citep{shin2022pseudodifferential},
models time-dependent dynamics \citep{diab2025temporal},
combines spectral and attention mechanisms \citep{bie2025spacefrequency},
develops mesh-informed architectures for operator learning in
finite-element spaces \citep{franco2023meshinformed},
and advances operator approximation theory
\citep{abdeljawad2026approximation}.
Physics-informed learning provides complementary tools for incorporating
physical constraints
\citep{karniadakis2021physicsinformed,cuomo2022scientific}.
In parallel, implicit neural representations encode continuous signals
as functions of coordinates, as exemplified by neural radiance fields
and multiresolution hash encodings
\citep{mildenhall2021nerf,mller2022instant}.
GINO \citep{li2023gino} uses graph neural operators to map between
irregular meshes and a regular latent grid, with a Fourier neural
operator acting on the grid.
Our work focuses on continual forecasting from expanding real-world
sensor sets under distribution drift.
We combine normalized Gaussian lifting and readout with
spectral-regime conditioning and establish sensor-discretization
consistency under explicit sampling assumptions.

\section{Notation and Symbols}
\label{app:notation}

Table~\ref{tab:notation} summarizes the symbols used in the main text and proofs.
Scalars and scalar fields are italic ($T_h$, $u_t$), vectors are bold italic
($\bm u_i$, $\bm q_g$), matrices and gridded fields are bold upright
($\mathbf{Z}$, $\mathbf{A}$), and number spaces are blackboard bold ($\R$).
Calligraphic letters denote selected sets and operators
($\mathcal D$, $\mathcal S$, $\Gop_\theta$, $\mathcal P_\theta$).
In the proofs, dark blue marks continuum or reference objects, while dark
orange marks discrete or empirical objects, as described at the start of
Appendix~\ref{app:proofs}.

\begin{longtable}{
    @{}
    >{\raggedright\arraybackslash}p{0.26\linewidth}
    >{\raggedright\arraybackslash}p{0.70\linewidth}
    @{}
}
\caption{Symbols used throughout the paper.}
\label{tab:notation}\\

\toprule
\rowcolor{gray!20}
\textbf{Symbol} & \textbf{Meaning} \\
\midrule
\endfirsthead

\multicolumn{2}{@{}l}{
    \small\itshape
    Table~\ref{tab:notation}, continued from the previous page.
}\\
\toprule
\rowcolor{gray!20}
\textbf{Symbol} & \textbf{Meaning} \\
\midrule
\endhead

\midrule
\multicolumn{2}{r@{}}{
    \small\itshape continued on the next page
}\\
\endfoot

\bottomrule
\endlastfoot


\rowcolor{gray!12}
\multicolumn{2}{@{}l}{
    \textbf{\textit{Problem and data}}
}\\

$P$
&
Number of periods in the continual forecasting stream.
\\

$\mathcal V_p$, $N_p$
&
Sensor set in period $p$ and its size, $N_p=|\mathcal V_p|$;
$\mathcal V_p\subseteq\mathcal V_{p+1}$ for $p<P$.
\\

$\bm s_i$
&
Static descriptor vector of sensor $i$, containing geographic location and available attributes.
\\

$\chi_p$, $\bm x_i^{(p)}$
&
Period-specific coordinate map, the coordinate assigned to sensor $i$.
\\

$\Omega=[0,1]^2$
&
Normalized spatial domain.
\\

$T_h$, $T_f$
&
Input history length and forecast horizon.
\\

$u_t$, $y_{i,t}$, $\varepsilon_{i,t}$
&
Spatial field, its noisy observation at sensor $i$, and observation noise.
\\

$\bm u_i$
&
Observation history vector of sensor $i$.
\\

$\mathcal D$
&
Input set of coordinate--history pairs for one period.
\\

$\Gop_\theta$, $\hat{\bm y}_q$
&
Forecasting operator and its forecast at query sensor $q$.
\\

$\mathcal S$, $h_{\mathcal S}$
&
Sensor layout and its fill distance.
\\

$\nu_{\mathcal S}$, $\mu$
&
Empirical sensor measure and the common limiting measure in fixed-bandwidth sampling comparisons.
\\

$W_1(\cdot,\cdot)$
&
Wasserstein-1 distance between sampling measures.
\\


\rowcolor{gray!12}
\multicolumn{2}{@{}l}{
    \textbf{\textit{Grid, field, and encoder}}
}\\

$n_g$, $G=n_g^2$, $h_g=1/(n_g-1)$
&
Grid points per axis (denoted $n$ in the proofs), total point count, and spacing.
\\

$\Xi_{h_g}$, $\bm\xi_g$
&
Latent Cartesian grid and its node indexed by $g$.
\\

$\bm t_g$, $\bm t_{\bm a}$
&
Fourier phase node under scalar index $g$ or grid multi-index $\bm a$.
\\

$d$
&
Hidden width, i.e., the channel dimension of the latent field.
\\

$d_a$
&
Dimension of attention queries and keys.
\\

$d_\rho$
&
Output dimension of the Fourier coordinate encoding.
\\

$\tau_\theta$, $\bm e_i$
&
Shared temporal encoder and the history code of sensor $i$.
\\

$\kappa_\sigma$, $\sigma_{\mathrm{enc}}$
&
Gaussian kernel and learned encoder bandwidth.
\\

$a_{gi}$, $\mathbf{A}$
&
Normalized lift weights and their $G\times N$ matrix.
\\

$\mathbf{E}$, $\mathbf{R}$
&
Sensor-code matrix and the lifted field
$\mathbf{R}=\mathbf{A}\mathbf{E}$.
\\

$\rho$, $\psi_\theta$
&
Fixed Fourier coordinate encoding and shared post-encoding map.
\\

$\mathbf{Z}$, $\bm z_g$
&
Latent grid field and the feature vector at node $g$; row $g$ of $\mathbf{Z}$ is $\bm z_g^\top$.
\\


\rowcolor{gray!12}
\multicolumn{2}{@{}l}{
    \textbf{\textit{Spectral regime descriptor}}
}\\

$\bar z_g$, $\bar z$, $\widehat{\bar z}$
&
Channel mean at node $g$, the scalar grid, and its forward-normalized DFT.
\\

$\mathcal J$, $\bm\omega_m$
&
Selected nonzero frequency bins and their wavevectors; spectral proofs use $\bm k_j$ with bin index $j$.
\\

$E_m$, $\omega_m$
&
Stabilized bin power and radial frequency; denoted $E_j$ and $k_j$ in spectral proofs.
\\

$\epsilon_s$, $\epsilon_f$
&
Power/denominator and frequency stabilizers; written $\varepsilon_s,\varepsilon_f$ in the proofs.
\\

$\lambda_m$, $\mu_m$
&
Log radial frequency and log spectral power of bin $m$.
\\

$\alpha$, $\beta$, $\tilde\beta$
&
Normalized spectral centroid, regularized negative log--log slope,
and clipped slope $\mathrm{clip}(\beta,-10,10)$.
\\

$\eta_b$, $\mathcal B_b$, $B$
&
Band-energy ratio, logarithmically spaced frequency band, and number of bands.
\\

$\bm c$
&
Spectral regime descriptor
$[\alpha,\tilde\beta,\eta_1,\ldots,\eta_B]^\top$.
\\

$\mathcal K_M$, $M$
&
Retained low-frequency mode set and its cutoff.
\\

$\mathbf{R}_\theta(\bm k)$
&
Complex Fourier multiplier at each mode; diagonal in the depthwise implementation.
\\


\rowcolor{gray!12}
\multicolumn{2}{@{}l}{
    \textbf{\textit{Field operator and decoder}}
}\\

$\bm\gamma$, $\bm\delta$
&
FiLM gain adjustment and offset from $\bm c$; the multiplicative gain is $\mathbf{1}_d+\bm\gamma$.
\\

$\mathbf{Z}^{\mathrm c}$
&
Channel-modulated latent field.
\\

$\mathcal P_\theta$, $\mathcal A_\theta$
&
Spectral propagation path and normalized attention path.
\\

$\phi$, $\epsilon_a$
&
Positive feature map $\mathrm{ELU}(\cdot)+1$ and denominator stabilizer; the latter is written $\varepsilon_a$ in the proofs.
\\

$\mathbf{Q}$, $\mathbf{K}$, $\mathbf{V}$
&
Query, key, and value matrices, with rows
$\bm q_g^\top$, $\bm k_g^\top$, and $\bm v_g^\top$, respectively.
\\

$O_\theta$
&
Affine output map of the attention path.
\\

$\LN$, $\bm b_{\mathrm{res}}$
&
Rowwise LayerNorm and the shared output bias of an operator block.
\\

$\mathbf{H}$, $\bm h_g$
&
Block output field and the feature vector at node $g$; row $g$ of $\mathbf{H}$ is $\bm h_g^\top$.
\\

$L$
&
Number of stacked operator blocks.
\\

$b_{qg}$, $\bm b_q$, $\sigma_{\mathrm{dec}}$
&
Normalized decoder gather weights, their column vector, and the learned decoder bandwidth.
\\

$\tilde{\bm h}_q$
&
Gathered field feature at query sensor $q$.
\\

$\mathbf{W}_{\mathrm{out}}$, $\mathbf{W}_{\mathrm{skip}}$
&
Weight matrices of the field readout and local-history prediction.
\\

$\bm e_{T_h}$
&
Last standard basis vector in $\R^{T_h}$, used to initialize $\mathbf{W}_{\mathrm{skip}}$.
\\

$\bm f_q$
&
Field contribution to the forecast at query sensor $q$.
\\


\rowcolor{gray!12}
\multicolumn{2}{@{}l}{
    \textbf{\textit{Analysis and proofs}}
}\\

$\|\mathbf A\|_\infty$, $\|Z\|_\infty$
&
Maximum absolute row sum of $\mathbf A$; maximum Euclidean row norm of a grid field $Z$ in the proofs.
\\

$B_\zeta$, $L_\zeta$
&
Uniform bound and Lipschitz constant of the noiseless code field $\zeta(\bm x)=\tau_\theta(\bm u(\bm x))$.
\\

$L_{\mathrm{post}}$, $L_{\mathrm{dec}}$
&
Lipschitz constants of the post-encoding map and decoder readout.
\\

$L_{\mathrm{spec}}$, $L_{\mathrm{attn}}$, $L_{\mathrm{LN}}$
&
Lipschitz bounds for spectral propagation, attention, and LayerNorm.
\\

$m_\sigma$, $L_\kappa$
&
Lower bound on the Gaussian kernel and upper bound on its spatial gradient norm.
\\

$q_{\mathcal S}$, $c_q$, $c_-,c_+$
&
Sensor separation radius and constants in the local-bandwidth sampling assumptions.
\\

$r_*$, $\mathcal E_{\mathcal S}$
&
Reference lift and the lift error bound for a sensor layout.
\\

$\mathbf{Z}_{\mathcal S}$, $\mathbf{Z}_*$
&
Discrete and reference post-encoded fields.
\\

$\varepsilon_{\mathrm{LN}}$, $d_0$
&
LayerNorm variance stabilizer and attention denominator lower bound.
\\

$J_{h_g}$, $J_*$
&
Discrete grid gather and its continuum reference.
\\

$\Phi^M_{\theta,h_g}$, $\Phi^M_\theta$
&
Grid and continuum \textnormal{\SRE/\DFO} stacks.
\\

$Q_{h_g}$
&
Endpoint-grid quadrature.
\\

$C_{\theta,M,h_g}$, $C_{\mathrm{geo}}$
&
Constants controlling operator stability and local interpolation, respectively.
\\

$C_{\mathrm{quad},M}$, $C_{\mathrm{grid},M}$
&
Grid-to-continuum error constants for the operator stack and full forecasting model, respectively.
\\

$\delta_\chi$, $\delta_q$, $R_\chi$
&
Maximum sensor-coordinate displacement, query-coordinate displacement, and induced prediction discrepancy.
\\

$C_\beta$
&
Constant in the spectral-slope perturbation bound.
\\

\end{longtable}

\section{Full Proofs}
\label{app:proofs}

The analysis follows the lift--condition--evolve--query construction. We first
bound the error introduced by sensor sampling, then propagate it through the
trained operator, and finally analyze latent-grid refinement and spectral
perturbations. Throughout a sampling comparison, the coordinate assignment
$\chi$, field history, trained neural maps, latent grid, and common query
history are fixed. The period-specific assignments are
$\bx_i^{(p)}=\chi_p(\bm s_i)\in\Omega=[0,1]^2$; coordinate reassignment is
quantified separately in Section~\ref{app:remapping}. Write
$\mathcal S=\{\bx_i\}_{i=1}^N$, let $h_{\mathcal S}$ be its fill distance,
and use the Euclidean channel norm inside every function or grid sup norm.
Dark blue labels continuum and reference quantities: the limiting
sampling measure $\refvar{\mu}$, smoothed lifts, continuum operators, and
reference predictions. Dark orange labels discrete and empirical
quantities: sensor layouts, empirical measures, grid quantities, and the
predictions computed by the finite model.

\subsection{Operator Conventions and Assumptions}
\label{app:operator_details}
\label{app:assumptions}

\textbf{\textit{Code field and two sampling regimes.}}
The noiseless history field is $\bm u(\bx)\in\R^{T_h}$ and its temporal
encoding is $\zeta(\bx)=\tau_\theta(\bm u(\bx))\in\R^d$. Let
$\nu_{\mathcal S}=N^{-1}\sum_i\delta_{\bx_i}$ be the empirical sensor
measure. For a probability measure $\nu$ and bandwidth $\sigma>0$, the
normalized field lift is
\begin{equation}
 \discvar{I_{\nu,\sigma}\zeta}(\bx)
 =\frac{\int_\Omega\kappa_\sigma(\bx,\bm y)\zeta(\bm y)\,d\discvar{\nu}(\bm y)}
 {\int_\Omega\kappa_\sigma(\bx,\bm y)\,d\discvar{\nu}(\bm y)},
 \qquad
 \discvar{I_{\mathcal S,\sigma}\zeta}=\discvar{I_{\nu_{\mathcal S},\sigma}\zeta}.
 \label{eq:discrete_encoder_interp}
\end{equation}
This identity expresses the implemented equal-weight sensor lift as a
normalized integral. At fixed learned bandwidth, its reference is the
smoothed field $\refvar{I_{\mu,\sigma}\zeta}$ associated with a common limiting sensor
measure $\refvar{\mu}$. When the bandwidth decreases with fill distance, its reference
is the pointwise code field $\zeta$.

\begin{assumption}[Admissible sampling families]
\label{ass:geometry}
Each comparison uses one of the following regimes. In the fixed-bandwidth
regime, $\sigma_{\mathrm{enc}}=\sigma>0$ and
$\discvar{W_1(\nu_{\mathcal S},\mu)}\to0$ for a common probability measure $\refvar{\mu}$ on
$\Omega$. Here $W_1(\nu,\mu)$ is the infimum of
$\int\|\bm x-\bm y\|\,d\pi(\bm x,\bm y)$ over couplings $\pi$ of
$\nu$ and $\mu$. In the local-bandwidth regime, the separation radius
$q_{\mathcal S}=\tfrac12\min_{i\ne j}\|\bx_i-\bx_j\|$ satisfies
$q_{\mathcal S}\ge c_qh_{\mathcal S}$, and
$c_-h_{\mathcal S}\le\sigma_{\mathrm{enc}}\le c_+h_{\mathcal S}$ for
fixed positive $c_q,c_-,c_+$, with $h_{\mathcal S}\to0$ along refinement.
\end{assumption}

These conditions specify respectively convergence of sampling density and
convergence of local interpolation neighborhoods. For example, equal-mass
cells of $\mu$ with diameter at most $Ch_{\mathcal S}$, each containing one
sensor, give $W_1(\nu_{\mathcal S},\mu)\le Ch_{\mathcal S}$ by transporting
each cell's mass to its sensor. Thus the fixed-bandwidth result also has a
fill-distance rate for such sampling families.

\begin{assumption}[Regularity and trained maps]
\label{ass:regularity}
The code field is bounded by $B_\zeta$ and is $L_\zeta$-Lipschitz. The
coordinate feature map $\rho$ and post-encoding map $\psi_\theta$ are
Lipschitz. The network has finitely many blocks with finite trained weights
and Lipschitz activations. Every LayerNorm uses a positive variance
stabilizer $\varepsilon_{\mathrm{LN}}$. The spectral power and denominator
stabilizers $\varepsilon_s,\varepsilon_f,\varepsilon_a$ are positive.
\end{assumption}

For the spectral descriptor, let $\bar z_g=d^{-1}\mathbf{1}_d^\top\bm z_g$,
$P_j=|\widehat{\bar z}(\bk_j)|^2$, $E_j=P_j+\varepsilon_s$,
$k_j=\|\bk_j\|+\varepsilon_f$, and
$x_j=\log k_j-|\mathcal J|^{-1}\sum_{r\in\mathcal J}\log k_r$.
Writing $A=\sum_jE_j$ and $D_\beta=\sum_jx_j^2+\varepsilon_s$, the
statistics in Eq.~\eqref{eq:sre_statistics_main} become
\begin{equation}
 \alpha=\frac{\sum_jE_jk_j}{k_{\max}A+\varepsilon_s},
 \qquad
 \beta=-\frac{\sum_jx_j\log E_j}{D_\beta},
 \qquad
 \eta_b=\frac{\sum_{j\in\mathcal B_b}E_j}{A+\varepsilon_s}.
 \label{eq:sre_beta}
\end{equation}
The centering of $\log E_j$ cancels because $\sum_jx_j=0$.
The state is $\bm c=[\alpha,\operatorname{clip}(\beta,-10,10),
\eta_1,\ldots,\eta_B]^\top$. Effective powers obey
$E_j\ge\varepsilon_s$ and $A\ge|\mathcal J|\varepsilon_s$.

\begin{assumption}[Fixed-design grid refinement]
\label{ass:grid_consistency}
For the grid comparison, the Fourier cutoff $M$, descriptor frequency set
$\mathcal J$, band membership, and Fourier normalization are fixed. Grids
have $n$ points per axis, $G=n^2$, and $h_g=1/(n-1)$; retained modes lie
below Nyquist for sufficiently large $n$. The decoder uses either a fixed
positive bandwidth or $\sigma_{\mathrm{dec}}\asymp h_g$. All continuum
fields reached by the finite stack are bounded and Lipschitz.
\end{assumption}

This grid comparison concerns the same finite spectral design at increasing
resolution. The sensor-sampling comparison keeps the grid and its full
spectral design fixed. For $\bm a\in\{0,\ldots,n-1\}^2$, the spatial node
is $\bm\xi_{\bm a}=\bm a/(n-1)$ and the FFT phase node is
$\bm t_{\bm a}=\bm a/n$. We use the normalized forward transform
\begin{equation}
 \widehat{\bm z}_n(\bk)
 =\frac1G\sum_{\bm a}\bm z_{\bm a}e^{-2\pi\mathrm{i}\bk\cdot\bm t_{\bm a}},
 \qquad
 (\discvar{\mathcal P_{\theta,n}}\bm z)_{\bm a}
 =\sum_{\bk\in\mathcal K_M}\mathbf{R}_\theta(\bk)
 \widehat{\bm z}_n(\bk)e^{2\pi\mathrm{i}\bk\cdot\bm t_{\bm a}}.
 \label{eq:dfo_spectral}
\end{equation}
Conjugate symmetry of retained modes and multipliers gives a real output for
real input. Section~\ref{app:grid_proof} accounts for the phase-node offset.

\subsection{Normalized Kernel Lifting and Gathering}

\begin{lemma}[Fixed-bandwidth sampling error]
\label{lem:fixed_bandwidth}
For a fixed $\sigma>0$, let $m_\sigma=e^{-1/\sigma^2}$ and
$L_\kappa=e^{-1/2}/\sigma$. Under Assumption~\ref{ass:regularity},
\begin{equation}
 \|\discvar{I_{\nu,\sigma}\zeta}-\refvar{I_{\mu,\sigma}\zeta}\|_\infty
 \le C_{\sigma,\zeta}\,\discvar{W_1(\nu,\mu)},
 \qquad
 C_{\sigma,\zeta}=
 \frac{L_\zeta+2B_\zeta L_\kappa}{m_\sigma}.
 \label{eq:fixed_bandwidth_bound}
\end{equation}
\end{lemma}

\begin{proof}
Fix a query $\bx$. Write $A_\nu=\int\kappa_\sigma(\bx,\bm y)
\zeta(\bm y)\,d\nu(\bm y)$ and
$D_\nu=\int\kappa_\sigma(\bx,\bm y)\,d\nu(\bm y)$.
The domain diameter is $\sqrt2$, so $D_\nu,D_\mu\ge m_\sigma$.
The Gaussian derivative satisfies
$\|\nabla_{\bm y}\kappa_\sigma(\bx,\bm y)\|
\le\max_{r\ge0}r\sigma^{-2}e^{-r^2/(2\sigma^2)}=L_\kappa$.
For any coupling of $\nu$ and $\mu$, subtract the paired integrands and
apply their Lipschitz bounds. Taking the infimum over couplings gives
$\|A_\nu-A_\mu\|\le(L_\zeta+B_\zeta L_\kappa)W_1(\nu,\mu)$ and
$|D_\nu-D_\mu|\le L_\kappa W_1(\nu,\mu)$. Since
$\|A_\mu/D_\mu\|\le B_\zeta$, the quotient is expanded over the common
denominator as
\begin{equation}
 \begin{aligned}
 \frac{A_\nu}{D_\nu}-\frac{A_\mu}{D_\mu}
 &=\frac{A_\nu D_\mu-A_\mu D_\nu}{D_\nu D_\mu}\\
 &=\frac{(A_\nu-A_\mu)D_\mu+A_\mu(D_\mu-D_\nu)}{D_\nu D_\mu}\\
 &=\frac{A_\nu-A_\mu}{D_\nu}
 +\frac{A_\mu}{D_\mu}\frac{D_\mu-D_\nu}{D_\nu},\\
 \left\|\frac{A_\nu}{D_\nu}-\frac{A_\mu}{D_\mu}\right\|
 &\le\frac{\|A_\nu-A_\mu\|}{D_\nu}
 +\left\|\frac{A_\mu}{D_\mu}\right\|
   \frac{|D_\mu-D_\nu|}{D_\nu}\\
 &\le\frac{L_\zeta+2B_\zeta L_\kappa}{m_\sigma}W_1(\nu,\mu).
 \end{aligned}
 \label{eq:fixed_bandwidth_chain}
\end{equation}
The bounds are uniform in $\bx$, proving the claim.
\end{proof}

\begin{lemma}[Local-bandwidth interpolation error]
\label{lem:encoder_interp}
Under the local-bandwidth regime of Assumption~\ref{ass:geometry}, there is
a finite $C_{\mathrm{geo}}$, independent of $N$ and $h_{\mathcal S}$, such
that
\begin{equation}
 \|\discvar{I_{\mathcal S,\sigma_{\mathrm{enc}}}\zeta}-\zeta\|_\infty
 \le C_{\mathrm{geo}}L_\zeta\discvar{h_{\mathcal S}}.
 \label{eq:encoder_interp_bound}
\end{equation}
\end{lemma}

\begin{proof}
Let $a_i(\bx)$ be the normalized Gaussian weights. Their positivity and unit
sum give
\begin{equation}
 \begin{aligned}
 \|I_{\mathcal S,\sigma}\zeta(\bx)-\zeta(\bx)\|
 &\le\sum_i a_i(\bx)\|\zeta(\bx_i)-\zeta(\bx)\|\\
 &\le L_\zeta\sum_i a_i(\bx)\|\bx_i-\bx\|.
 \end{aligned}
 \label{eq:encoder_interp_chain}
\end{equation}
Set $h=h_{\mathcal S}$. A sensor lies within $h$ of every $\bx$, so the
unnormalized weight sum is at least $c_0=e^{-1/(2c_-^2)}$.
Partition sensors into shells
$A_m=\{\bx_i:mh\le\|\bx_i-\bx\|<(m+1)h\}$ for integers $m\ge0$.
The disjoint radius-$q_{\mathcal S}$ balls around these sensors fit inside
a radius-$(m+1)h+q_{\mathcal S}$ ball, so comparing volumes gives
\begin{equation}
 |A_m|\,\pi q_{\mathcal S}^2
 \le\pi\big((m+1)h+q_{\mathcal S}\big)^2,
 \qquad
 |A_m|\le\left(1+\frac{(m+1)h}{q_{\mathcal S}}\right)^2
 \le C_A(m+1)^2,
 \label{eq:shell_count}
\end{equation}
with $C_A=(1+1/c_q)^2$, using $q_{\mathcal S}\ge c_qh$ and $m+1\ge1$.
Consequently,
\begin{equation}
 \begin{aligned}
 \sum_i a_i(\bx)\|\bx_i-\bx\|
 &\le\frac1{c_0}\sum_{m\ge0}|A_m|(m+1)h
     e^{-m^2/(2c_+^2)}\\
 &\le\frac{C_A}{c_0}h\sum_{m\ge0}(m+1)^3e^{-m^2/(2c_+^2)}
 =C_{\mathrm{geo}}h.
 \end{aligned}
 \label{eq:gaussian_first_moment}
\end{equation}
The Gaussian series converges. Substitution into
Eq.~\eqref{eq:encoder_interp_chain} completes the proof.
\end{proof}

\textbf{\textit{Post-encoding propagation.}}
For either sampling regime, let $\refvar{r_*}$ be its reference lift and let
$\discvar{\mathcal E_{\mathcal S}}$ denote the corresponding lift error bound:
\begin{equation}
 (\refvar{r_*},\discvar{\mathcal E_{\mathcal S}})=
 \begin{cases}
 (\refvar{I_{\mu,\sigma}\zeta},
 \discvar{C_{\sigma,\zeta}W_1(\nu_{\mathcal S},\mu)}),&\text{fixed bandwidth},\\
 (\zeta,\discvar{C_{\mathrm{geo}}L_\zeta h_{\mathcal S}}),&\text{local bandwidth}.
 \end{cases}
 \label{eq:sampling_error_scale}
\end{equation}
Define $\discvar{\bm z_{\mathcal S,g}}=\psi_\theta([\discvar{I_{\mathcal S,\sigma}\zeta}
(\bm\xi_g)\|\rho(\bm\xi_g)])$ and
$\refvar{\bm z_{*,g}}=\psi_\theta([\refvar{r_*}(\bm\xi_g)\|\rho(\bm\xi_g)])$.
The identical coordinate features cancel from the input difference, so the
$L_{\mathrm{post}}$-Lipschitz map $\psi_\theta$ gives
$\|\discvar{\mathbf{Z}_{\mathcal S}}-\refvar{\mathbf{Z}_*}\|_\infty
\le L_{\mathrm{post}}\discvar{\mathcal E_{\mathcal S}}$.
\label{lem:post_encoder}

\begin{lemma}[Decoder gather error]
\label{lem:decoder_interp}
The normalized Gaussian gather $J_{h_g}$ is nonexpansive from the grid sup
norm to the query norm. For a bounded Lipschitz field $v$, it satisfies
\begin{equation}
 \left\|\discvar{J_{h_g}}(v|_{\Xi_{h_g}})(\xq)-\refvar{J_*}v(\xq)\right\|
 \le C_{\mathrm{dec},v}\discvar{h_g},
 \qquad
 J_*v=
 \begin{cases}
 I_{\lambda,\sigma_{\mathrm{dec}}}v,&\sigma_{\mathrm{dec}}>0\text{ fixed},\\
 v,&\sigma_{\mathrm{dec}}\asymp h_g,
 \end{cases}
 \label{eq:decoder_interp_bound}
\end{equation}
where $\lambda$ is Lebesgue probability measure on the unit square.
\end{lemma}

\begin{proof}
For arbitrary grid fields $V,V'$, positivity of $b_{qg}$ and
$\sum_gb_{qg}=1$ give
$\|\sum_gb_{qg}(\bm v_g-\bm v'_g)\|\le\|V-V'\|_\infty$.
For fixed bandwidth, partition $\Omega$ into $n^2$ equal squares and couple
each square to its associated endpoint-grid node. For node
$\bm\xi_{\bm a}=\bm a/(n-1)$ and cell
$C_{\bm a}=\prod_{j=1}^{2}[a_j/n,(a_j+1)/n]$, each axis contributes at most
$1/n\le h_g$, so each transport distance is at most $\sqrt2h_g$ and
$W_1(G^{-1}\sum_g\delta_{\bm\xi_g},\lambda)\le\sqrt2h_g$.
Lemma~\ref{lem:fixed_bandwidth} with $\zeta=v$ proves the first case.
For $\sigma_{\mathrm{dec}}\asymp h_g$, the grid is quasi-uniform and the
first-moment proof gives
\begin{equation}
 \begin{aligned}
 \left\|\sum_gb_{qg}v(\bm\xi_g)-v(\xq)\right\|
 &\le L_v\sum_gb_{qg}\|\bm\xi_g-\xq\|\\
 &\le C_{\mathrm{geo}}L_vh_g,
 \end{aligned}
 \label{eq:decoder_interp_chain}
\end{equation}
with the grid geometry absorbed into $C_{\mathrm{geo}}$.
\end{proof}

\subsection{Stability of the Trained Field Operator}

\textbf{\textit{Spectral descriptor and conditioning.}}
\label{lem:sre_film_lipschitz}
A normalized Fourier coefficient is a bounded linear map with sup-norm gain
at most one. The power difference satisfies
$\big||a|^2-|b|^2\big|\le(|a|+|b|)|a-b|$.
On bounded fields, the stabilized logarithms and ratios in
Eq.~\eqref{eq:sre_beta} therefore have finite Lipschitz constants; their
denominators are bounded below by positive stabilizers. Clipping is
nonexpansive. Denote the resulting descriptor constant by $L_c$.
In block $\ell$, $\|\bm\gamma_\ell\|_\infty\le0.1$ and
$(\bm\gamma_\ell,\bm\delta_\ell)$ is Lipschitz in $\bm c$.
For fields bounded by $B_\ell$, the modulation
$\mathbf{X}_\ell^{\mathrm c}
=\mathbf{X}_\ell\odot(\mathbf{1}_G(\mathbf{1}_d+\bm\gamma_\ell)^\top)
+\mathbf{1}_G\bm\delta_\ell^\top$ expands into three terms,
\begin{equation}
 \begin{aligned}
 \mathbf{X}_\ell^{\mathrm c}-\mathbf{X}_\ell'{}^{\mathrm c}
 &=(\mathbf{X}_\ell-\mathbf{X}'_\ell)\odot(\mathbf{1}_G(\mathbf{1}_d+\bm\gamma_\ell)^\top)
+\mathbf{X}'_\ell\odot(\mathbf{1}_G(\bm\gamma_\ell-\bm\gamma'_\ell)^\top)
 +\mathbf{1}_G(\bm\delta_\ell-\bm\delta'_\ell)^\top,\\\\
 \|X_\ell^{\mathrm c}-X_\ell'{}^{\mathrm c}\|_\infty
 &\le1.1\|X_\ell-X'_\ell\|_\infty
 +B_\ell\|\bm\gamma_\ell(\bm c)-\bm\gamma_\ell(\bm c')\|+\|\bm\delta_\ell(\bm c)-\bm\delta_\ell(\bm c')\|\\
 &\le1.1\|X_\ell-X'_\ell\|_\infty
 +(B_\ell L_{\gamma,\ell}+L_{\delta,\ell})\|\bm c-\bm c'\|,
 \end{aligned}
 \label{eq:film_lipschitz_bound}
\end{equation}
where the first line uses $\|\mathbf{1}_d+\bm\gamma_\ell\|_\infty\le1.1$
and the field bound.

\begin{lemma}[Normalized linear attention]
\label{lem:attention_lipschitz}
On bounded feature sets, the attention map in
Eq.~\eqref{eq:dfo_attention_main} has a Lipschitz constant independent of $G$.
\end{lemma}

\begin{proof}
Use column vectors $\bm q_g,\bm k_g\in\R^{d_a}$ and
$\bm v_g\in\R^{d}$, and form the matrices
$\mathbf{Q}=[\bm q_1,\ldots,\bm q_G]^\top\in\R^{G\times d_a}$,
$\mathbf{K}=[\bm k_1,\ldots,\bm k_G]^\top\in\R^{G\times d_a}$, and
$\mathbf{V}=[\bm v_1,\ldots,\bm v_G]^\top\in\R^{G\times d}$.
The attention features before the output projection have the exact matrix
form
\begin{equation}
 \begin{aligned}
 \mathbf{T}&=\operatorname{diag}(\mathbf{Q}\mathbf{K}^\top\mathbf{1}_G+\varepsilon_a\mathbf{1}_G)^{-1}
        \mathbf{Q}(\mathbf{K}^\top\mathbf{V}),\\
 \bm t_g^\top&=\frac{\bm q_g^\top\overline{\mathbf{B}}^\top}
 {\bm q_g^\top\overline{\bm k}+\varepsilon_a/G},
 \qquad
 \overline{\mathbf{B}}=\frac1G\mathbf{V}^\top\mathbf{K}
 =\frac1G\sum_j\bm v_j\bm k_j^\top,
 \quad
 \overline{\bm k}=\frac1G\mathbf{K}^\top\mathbf{1}_G
 =\frac1G\sum_j\bm k_j,
 \end{aligned}
 \label{eq:attention_fraction}
\end{equation}
where $\bm t_g^\top$ is row $g$ of $\mathbf{T}$.
The first line is an algebraic identity that reads entrywise. Its diagonal
factor collects the row sums
$[\mathbf{Q}\mathbf{K}^\top\mathbf{1}_G]_g=\sum_j\bm q_g^\top\bm k_j$,
while $\mathbf{Q}(\mathbf{K}^\top\mathbf{V})$ has row $g$
\begin{equation}
 \bm q_g^\top\Big(\sum_j\bm k_j\bm v_j^\top\Big)
 =\sum_j(\bm q_g^\top\bm k_j)\bm v_j^\top,
 \label{eq:attention_row_expansion}
\end{equation}
so row $g$ of $\mathbf{T}$ reproduces the fraction in
Eq.~\eqref{eq:dfo_attention_main}. The second line evaluates the same row
through the two shared aggregates, because
\begin{equation}
 \overline{\mathbf{B}}\bm q_g
 =\frac1G\sum_j(\bm q_g^\top\bm k_j)\bm v_j,
 \qquad
 \bm q_g^\top\overline{\bm k}+\frac{\varepsilon_a}{G}
 =\frac1G\Big(\sum_j\bm q_g^\top\bm k_j+\varepsilon_a\Big)
 \label{eq:attention_aggregate_identity}
\end{equation}
factor a common $G^{-1}$ from numerator and denominator.
Bounded affine preactivations and
$\phi(t)=\operatorname{ELU}(t)+1>0$ give a common componentwise interval
$[m_\phi,B_\phi]$ with $m_\phi>0$, so every entry of $\bm q_g$ and $\bm k_j$
is at least $m_\phi$ and
\begin{equation}
 \bm q_g^\top\bm k_j=\sum_{a=1}^{d_a}q_{g,a}k_{j,a}\ge d_am_\phi^2.
 \label{eq:attention_denominator_bound}
\end{equation}
Thus every denominator is at least
$d_0=d_am_\phi^2$, independently of $G$.

Let $B_q,B_k,B_v$ bound the projected vector norms and let
$L_q,L_k,L_v$ be their Lipschitz constants. If
$e=\|X-X'\|_\infty$, one outer product splits as
$\bm v_j\bm k_j^\top-\bm v'_j\bm k'_j{}^\top
=(\bm v_j-\bm v'_j)\bm k_j^\top+\bm v'_j(\bm k_j-\bm k'_j)^\top$, so
\begin{equation}
 \begin{aligned}
 \|\overline{\mathbf{B}}-\overline{\mathbf{B}}'\|
 &\le\frac1G\sum_j
 \|(\bm v_j-\bm v'_j)\bm k_j^\top+
       \bm v'_j(\bm k_j-\bm k'_j)^\top\|\\
 &\le(B_kL_v+B_vL_k)e,
 \qquad
 \|\overline{\bm k}-\overline{\bm k}'\|\le L_ke.
 \end{aligned}
 \label{eq:attention_aggregate_difference}
\end{equation}
For $N_g=\overline{\mathbf{B}}\bm q_g$ and
$D_g=\bm q_g^\top\overline{\bm k}+\varepsilon_a/G$, this implies
$\|N_g\|\le B_N=B_vB_kB_q$,
$\|N_g-N'_g\|\le L_Ne$, and $|D_g-D'_g|\le L_De$, with
$L_N=B_q(B_kL_v+B_vL_k)+B_vB_kL_q$ and
$L_D=B_kL_q+B_qL_k$. The quotient expansion yields
\begin{equation}
 \begin{aligned}
 \left\|\frac{N_g}{D_g}-\frac{N'_g}{D'_g}\right\|
 &\le\frac{\|N_g-N'_g\|}{D_g}
 +\frac{\|N'_g\|\,|D'_g-D_g|}{D_gD'_g}\\
 &\le\left(\frac{L_N}{d_0}+\frac{B_NL_D}{d_0^2}\right)e.
 \end{aligned}
 \label{eq:attention_lipschitz_chain}
\end{equation}
The affine output map contributes the norm $L_O$ of its linear part. Hence
$L_{\mathrm{attn}}=L_O(L_N/d_0+B_NL_D/d_0^2)$ is independent of $G$.
\end{proof}

\textbf{\textit{Residual normalization.}}
The retained Fourier path has gain
$L_{\mathrm{spec}}\le\sum_{\bk\in\mathcal K_M}\|\mathbf{R}_\theta(\bk)\|$
by the triangle inequality over modes.
For LayerNorm, write the centering projection
$\mathbf{C}=\mathbf{I}-\mathbf{1}_d\mathbf{1}_d^\top/d$,
$s=(\|\mathbf{C}\bm v\|^2/d+\varepsilon_{\mathrm{LN}})^{1/2}$, and let
$\bm a,\bm b$ be its learned scale and bias. Since $\mathbf{C}$ is symmetric
and idempotent, $\partial s/\partial\bm v=\mathbf{C}\bm v/(ds)$, and
differentiating the quotient gives
\begin{equation}
 \begin{aligned}
 \LN(\bm v)&=\operatorname{diag}(\bm a)\frac{\mathbf{C}\bm v}{s}
 +\bm b,\\
 D\LN(\bm v)
 &=\operatorname{diag}(\bm a)
 \left(\frac{\mathbf{C}}s-\frac{(\mathbf{C}\bm v)(\mathbf{C}\bm v)^\top}{ds^3}\right),\\
 \|D\LN(\bm v)\|
 &\le\frac{\|\bm a\|_\infty}{\sqrt{\varepsilon_{\mathrm{LN}}}}.
 \end{aligned}
 \label{eq:layernorm_jacobian}
\end{equation}
The bracket is symmetric with the eigenpairs
$\mathbf{C}\bm v\mapsto\varepsilon_{\mathrm{LN}}/s^3$ and
$\bm w\mapsto1/s$ for $\bm w\perp\{\mathbf{1}_d,\mathbf{C}\bm v\}$, and it
annihilates $\mathbf{1}_d$. The first eigenvalue follows from
$\|\mathbf{C}\bm v\|^2=d(s^2-\varepsilon_{\mathrm{LN}})$, and
$s\ge\sqrt{\varepsilon_{\mathrm{LN}}}$ shows that $1/s$ is the larger one,
so the spectral norm of the bracket is $1/s$. This establishes the stated
bound and also bounds every block output on bounded inputs.

\begin{lemma}[Finite stack with shared spectral state]
\label{lem:lipschitz_modules}
For a fixed query, let $\Psi_{\theta,h_g,M}$ map the post-encoded field to
the field-path forecast through all $L$ operator blocks and the decoder.
Under Assumption~\ref{ass:regularity},
\begin{equation}
 \|\Psi_{\theta,h_g,M}(Z)-\Psi_{\theta,h_g,M}(Z')\|
 \le C_{\theta,M,h_g}\|Z-Z'\|_\infty.
 \label{eq:fixed_operator_lipschitz}
\end{equation}
\end{lemma}

\begin{proof}
Set $X_0=Z$, $X'_0=Z'$, and $e_\ell=\|X_\ell-X'_\ell\|_\infty$.
Every block uses the initial descriptor, so
$\|\bm c(Z)-\bm c(Z')\|\le L_ce_0$ throughout the stack.
Combining Eq.~\eqref{eq:film_lipschitz_bound} with the unconditioned residual
connection gives the scalar recursion
\begin{equation}
 \begin{aligned}
 e_{\ell+1}
 &\le L_{\mathrm{LN},\ell}\left[
 e_\ell+(L_{\mathrm{spec},\ell}+L_{\mathrm{attn},\ell})
 \{1.1e_\ell+(B_\ell L_{\gamma,\ell}+L_{\delta,\ell})L_ce_0\}\right]\\
 &=a_\ell e_\ell+b_\ell L_ce_0.
 \end{aligned}
 \label{eq:residual_operator_recursion}
\end{equation}
Unrolling it gives
$e_1\le a_0e_0+b_0L_ce_0$,
$e_2\le a_1a_0e_0+(a_1b_0+b_1)L_ce_0$, and in general
\begin{equation}
 \begin{aligned}
 e_L&\le\left[\prod_{\ell=0}^{L-1}a_\ell+
 L_c\sum_{j=0}^{L-1}b_j\prod_{\ell=j+1}^{L-1}a_\ell\right]e_0.
 \end{aligned}
 \label{eq:residual_operator_lipschitz}
\end{equation}
Here $a_\ell=L_{\mathrm{LN},\ell}[1+1.1(L_{\mathrm{spec},\ell}+
L_{\mathrm{attn},\ell})]$,
$b_\ell=L_{\mathrm{LN},\ell}(L_{\mathrm{spec},\ell}+
L_{\mathrm{attn},\ell})(B_\ell L_{\gamma,\ell}+L_{\delta,\ell})$,
and an empty product equals one. The decoder gather is nonexpansive.
Multiplying the bracket by the final normalization/readout constant
$L_{\mathrm{dec}}$ defines $C_{\theta,M,h_g}$. For fixed spectral design,
these constants can also be chosen uniformly over sufficiently fine grids.
\end{proof}

\subsection{Sensor-Discretization Consistency}

\begin{theorem}[Sensor-discretization consistency]
\label{thm:convergence}
Fix the coordinate representation, field history, trained neural maps, latent
grid, spectral design, and common query. Under
Assumptions~\ref{ass:geometry}--\ref{ass:regularity}, the reference lift
$r_*$ in Eq.~\eqref{eq:sampling_error_scale} defines a sampling-independent
prediction $\refvar{\overline{\Gop}_{\theta}^{h_g,M}}(\xq)$ satisfying
\begin{equation}
 \left\|\discvar{\Gop_{\theta}^{\mathcal S,h_g,M}}(\xq)
 -\refvar{\overline{\Gop}_{\theta}^{h_g,M}}(\xq)\right\|
 \le C_{\theta,M,h_g}L_{\mathrm{post}}\discvar{\mathcal E_{\mathcal S}}.
 \label{eq:conv_bound}
\end{equation}
At fixed bandwidth the error is $O(W_1(\nu_{\mathcal S},\mu))$; under
local-bandwidth refinement it is $O(h_{\mathcal S})$.
\end{theorem}

\begin{proof}
The common local skip is
$\bm s_q=\mathbf{W}_{\mathrm{skip}}\bm u_q+\bm b_{\mathrm{skip}}$.
With the post-encoded fields defined after
Eq.~\eqref{eq:sampling_error_scale}, the full predictions factor as
\begin{equation}
 \begin{aligned}
 \discvar{\Gop_{\theta}^{\mathcal S,h_g,M}}(\xq)
 &=\Psi_{\theta,h_g,M}(\discvar{\mathbf{Z}_{\mathcal S}})+\bm s_q,\\
 \refvar{\overline{\Gop}_{\theta}^{h_g,M}}(\xq)
 &=\Psi_{\theta,h_g,M}(\refvar{\mathbf{Z}_*})+\bm s_q,\\
 \|\discvar{\Gop_{\theta}^{\mathcal S,h_g,M}}(\xq)
 -\refvar{\overline{\Gop}_{\theta}^{h_g,M}}(\xq)\|
 &\le C_{\theta,M,h_g}\|\discvar{\mathbf{Z}_{\mathcal S}}-\refvar{\mathbf{Z}_*}\|_\infty\\
 &\le C_{\theta,M,h_g}L_{\mathrm{post}}\discvar{\mathcal E_{\mathcal S}}.
 \end{aligned}
 \label{eq:sampling_reference_factorization}
\end{equation}
The reference uses the same trained maps, common query, and reference field
for every layout in its sampling family. The respective sampling assumptions
make $\mathcal E_{\mathcal S}$ converge to zero.
\end{proof}

\begin{corollary}[Expansion stability]
\label{cor:noforget}
Let two layouts $\mathcal S,\mathcal S'$ share a sampling regime and its
reference field, as well as the fixed quantities of
Theorem~\ref{thm:convergence}. Then
\begin{equation}
 \begin{aligned}
 \|\discvar{\Gop_{\theta}^{\mathcal S',h_g,M}}(\xq)
 -\discvar{\Gop_{\theta}^{\mathcal S,h_g,M}}(\xq)\|
 &\le\|\discvar{\Gop_{\theta}^{\mathcal S',h_g,M}}(\xq)
 -\refvar{\overline{\Gop}_{\theta}^{h_g,M}}(\xq)\|+\|\refvar{\overline{\Gop}_{\theta}^{h_g,M}}(\xq)
 -\discvar{\Gop_{\theta}^{\mathcal S,h_g,M}}(\xq)\|\\
 &\le C_{\theta,M,h_g}L_{\mathrm{post}}
 (\discvar{\mathcal E_{\mathcal S'}}+\discvar{\mathcal E_{\mathcal S}}).
 \end{aligned}
 \label{eq:expansion_bound}
\end{equation}
In particular, local-bandwidth refinement gives
$C_{\mathrm{samp}}(h_{\mathcal S'}+h_{\mathcal S})$, where
$C_{\mathrm{samp}}=C_{\theta,M,h_g}L_{\mathrm{post}}
C_{\mathrm{geo}}L_\zeta$. For nested layouts,
$h_{\mathcal S'}\le h_{\mathcal S}$.
\end{corollary}

\textbf{\textit{Coordinate reassignment.}}
\label{app:remapping}
Fix histories, bandwidths, neural maps, and the grid for the same sensor
identities $V$. Let $\chi,\chi'$ be two assignments and set
$\discvar{\delta_\chi}=\max_{i\in V}\|\chi'(\bm s_i)-\chi(\bm s_i)\|$ and
$\discvar{\delta_q}=\|\chi'(\bm s_q)-\chi(\bm s_q)\|$.
Interpolate the coordinates linearly between the assignments. For encoder
logits $\ell_i=-\|\bm\xi_g-\bx_i\|^2/(2\sigma_{\mathrm{enc}}^2)$,
$|\dot\ell_i|\le\sqrt2\delta_\chi/\sigma_{\mathrm{enc}}^2$.
Differentiating the normalized weights gives
$\dot a_i=a_i(\dot\ell_i-\sum_j a_j\dot\ell_j)$ and thus
$\sum_i|\dot a_i|\le2\sqrt2\delta_\chi/\sigma_{\mathrm{enc}}^2$.
Applying the same calculation to decoder weights and integrating along the
coordinate path yields the reassignment bound
\begin{equation}
 \discvar{R_\chi}\le
 \frac{2\sqrt2C_{\theta,M,h_g}L_{\mathrm{post}}B_\zeta}
 {\sigma_{\mathrm{enc}}^2}\discvar{\delta_\chi}
 +\frac{2\sqrt2L_{\mathrm{dec}}B_H}{\sigma_{\mathrm{dec}}^2}\discvar{\delta_q},
 \label{eq:coordinate_remap_discrepancy}
\end{equation}
where $B_H$ bounds the evolved field along the comparison and $R_\chi$ is
the norm of the difference between the two full predictions. For an expansion $V\subseteq V'$, suppose the layouts under $\chi'$ share the sampling regime and reference field of Theorem~\ref{thm:convergence}. Insert the prediction on $V$ under $\chi'$ to obtain
\begin{equation}
 \begin{aligned}
 &\|\Gop_\theta(\mathcal D_{\chi'}(V'))(\chi'(\bm s_q))
 -\Gop_\theta(\mathcal D_\chi(V))(\chi(\bm s_q))\|\le C_{\theta,M,h_g}L_{\mathrm{post}}
 (\discvar{\mathcal E_{\mathcal S_{\chi'}(V')}}+
  \discvar{\mathcal E_{\mathcal S_{\chi'}(V)}})+\discvar{R_\chi}.
 \end{aligned}
 \label{eq:expansion_with_remapping}
\end{equation}
Here $\mathcal D_\chi(V)=\{(\chi(\bm s_i),\bm u_i):i\in V\}$ and
$\mathcal S_\chi(V)=\{\chi(\bm s_i):i\in V\}$. The two sampling errors
are evaluated in the new representation. This decomposition separates the
contributions of sensor sampling and coordinate motion.

\subsection{Fixed-Design Grid-to-Continuum Consistency}
\label{app:grid_proof}

\begin{lemma}[Endpoint-grid quadrature]
\label{lem:riemann}
For an $L_f$-Lipschitz function $f:\Omega\to\R^m$, the endpoint-grid
quadrature $\discvar{Q_{h_g}}f=G^{-1}\sum_gf(\bm\xi_g)$ satisfies
$\|\discvar{Q_{h_g}}f-\int_\Omega f(\bx)\,d\bx\|\le\sqrt2L_f\discvar{h_g}$.
\end{lemma}

\begin{proof}
Associate node $\bm\xi_{\bm a}=\bm a/(n-1)$ with the equal-mass cell
$C_{\bm a}=\prod_{j=1}^2[a_j/n,(a_j+1)/n]$; shared boundaries have zero
measure. Then $|C_{\bm a}|=1/G$ and
$\sup_{\bx\in C_{\bm a}}\|\bx-\bm\xi_{\bm a}\|\le\sqrt2h_g$.
Consequently,
\begin{equation}
 \begin{aligned}
 \left\|\discvar{Q_{h_g}}f-\int_\Omega f(\bx)\,d\bx\right\|
 &\le\sum_{\bm a}\int_{C_{\bm a}}
 \|f(\bm\xi_{\bm a})-f(\bx)\|\,d\bx\\
 &\le\sqrt2L_f\discvar{h_g}\sum_{\bm a}|C_{\bm a}|
 =\sqrt2L_f\discvar{h_g}.
 \end{aligned}
 \label{eq:riemann_sum_bound}
\end{equation}
\end{proof}

\begin{lemma}[Consistency of spectral conditioning and the operator stack]
\label{lem:grid_consistency}
Let $\discvar{\Phi_{\theta,h_g}^M}$ be the grid \textnormal{\SRE/\DFO} stack and let
$\refvar{\Phi_\theta^M}$ be its continuum counterpart using normalized Fourier
integrals, the same finite modes and descriptor bins, and normalized attention
integrals. Under Assumptions~\ref{ass:regularity}--\ref{ass:grid_consistency},
for a bounded Lipschitz input $z$,
\begin{equation}
 \max_g\|\discvar{\Phi_{\theta,h_g}^M}(z|_{\Xi_{h_g}})_g
 -\refvar{\Phi_\theta^M}(z)(\bm\xi_g)\|
 \le C_{\mathrm{quad},M}\discvar{h_g}.
 \label{eq:grid_consistency_bound}
\end{equation}
\end{lemma}

\begin{proof}
We track the Fourier, conditioning, and attention errors before propagating them through the stack.

\textbf{\textit{Fourier coefficients and phases.}}
Let $w$ be bounded by $B_w$ and $L_w$-Lipschitz, and let
$\refvar{\widehat w(\bk)}=\int_\Omega w(\bx)e^{-2\pi\mathrm{i}\bk\cdot\bx}\,d\bx$.
The integrand has Lipschitz constant at most $L_w+2\pi\|\bk\|B_w$, and the
phase factor obeys $|e^{\mathrm{i}x}-e^{\mathrm{i}y}|\le|x-y|$ with
$\|\bm t_{\bm a}-\bm\xi_{\bm a}\|\le\sqrt2h_g$. Splitting the grid sum from
its continuum integral at the spatial nodes therefore gives
\begin{equation}
 \begin{aligned}
 \|\discvar{\widehat w_n(\bk)}-\refvar{\widehat w(\bk)}\|
 &\le\frac1G\sum_{\bm a}\|w(\bm\xi_{\bm a})\|
 \left|e^{-2\pi\mathrm{i}\bk\cdot\bm t_{\bm a}}
       -e^{-2\pi\mathrm{i}\bk\cdot\bm\xi_{\bm a}}\right|\\
 &\quad+\left\|\discvar{Q_{h_g}}(we^{-2\pi\mathrm{i}\bk\cdot(\cdot)})
 -\int_\Omega we^{-2\pi\mathrm{i}\bk\cdot\bx}\,d\bx\right\|\\
 &\le 2\sqrt2\pi\|\bk\|B_w\discvar{h_g}
 +\sqrt2(L_w+2\pi\|\bk\|B_w)\discvar{h_g}\\
 &\le\sqrt2(L_w+4\pi\|\bk\|B_w)\discvar{h_g}.
 \end{aligned}
 \label{eq:fourier_coefficient_quadrature}
\end{equation}
The inverse transform has the corresponding phase correction. Expanding
mode by mode gives
\begin{equation}
 \begin{aligned}
 \|\discvar{(\mathcal P_{\theta,n}w)_g}-\refvar{\mathcal P_\theta w}(\bm\xi_g)\|&\le\sum_{\bk\in\mathcal K_M}\|\mathbf{R}_\theta(\bk)\|
 \left[\|\discvar{\widehat w_n(\bk)}-\refvar{\widehat w(\bk)}\|
 +\|\refvar{\widehat w(\bk)}\|
 |e^{2\pi\mathrm{i}\bk\cdot\bm t_g}-e^{2\pi\mathrm{i}\bk\cdot\bm\xi_g}|\right]\\
 &\le C_{\mathrm{spec},M}\discvar{h_g}.
 \end{aligned}
 \label{eq:retained_fourier_quadrature}
\end{equation}
Both coefficient and evaluation-phase errors therefore have first-order
bounds for the same finite retained modes.

\textbf{\textit{Descriptor and common conditioning state.}}
Equation~\eqref{eq:fourier_coefficient_quadrature} applies to the fixed bins
in $\mathcal J$. The power difference bound, positive stabilizers, and
nonexpansive clipping yield
$\|\discvar{\bm c_{h_g}}-\refvar{\bm c_*}\|\le C_c\discvar{h_g}$.
Every block receives this same descriptor discrepancy. Its contribution to
FiLM is bounded by Eq.~\eqref{eq:film_lipschitz_bound}.

\textbf{\textit{Normalized attention integrals.}}
For a continuum conditioned field, let
$\refvar{\mathbf{B}_*}=\int_\Omega\bm v(\bm y)\bm k(\bm y)^\top\,d\bm y$ and
$\refvar{\bm k_*}=\int_\Omega\bm k(\bm y)\,d\bm y$.
Its attention output before projection is
$\refvar{\mathbf{T}_*}(\bx)=\refvar{\mathbf{B}_*}\bm q(\bx)/(\bm q(\bx)^\top\refvar{\bm k_*})$.
The integrands are bounded and Lipschitz, so endpoint-grid quadrature gives
$\|\overline{\mathbf{B}}-\refvar{\mathbf{B}_*}\|
+\|\overline{\bm k}-\refvar{\bm k_*}\|=O(\discvar{h_g})$.
Both normalized denominators have lower bound $d_0$.
The grid denominator additionally contains $\varepsilon_a/G$, giving
\begin{equation}
 \begin{aligned}
 \left\|\frac{\overline{\mathbf{B}}\bm q_g}
 {\bm q_g^\top\overline{\bm k}+\varepsilon_a/G}
 -\frac{\refvar{\mathbf{B}_*}\bm q_g}
 {\bm q_g^\top\refvar{\bm k_*}}\right\|
 &\le\frac{B_q\|\overline{\mathbf{B}}-\refvar{\mathbf{B}_*}\|}{d_0}
 +\frac{B_N}{d_0^2}
 \left(B_q\|\overline{\bm k}-\refvar{\bm k_*}\|+\frac{\varepsilon_a}{G}\right)
 =O(\discvar{h_g}).
 \end{aligned}
 \label{eq:attention_quadrature}
\end{equation}
Here $G^{-1}=n^{-2}=O(h_g^2)$. The affine output projection preserves the
rate. Changes in the conditioned input contribute another first-order term
by Lemma~\ref{lem:attention_lipschitz}.

\textbf{\textit{Propagation through all blocks.}}
Let $e_\ell$ be the maximum difference between grid and continuum samples
after block $\ell$, with $e_0=0$. The preceding estimates and residual
normalization give $e_{\ell+1}\le a_\ell e_\ell+d_\ell h_g$, where
$d_\ell$ includes the shared descriptor and quadrature errors. Expanding the
recursion gives
$e_L\le h_g\sum_{j=0}^{L-1}d_j\prod_{\ell=j+1}^{L-1}a_\ell$.
The finite sum defines $C_{\mathrm{quad},M}$ and proves the lemma.
\end{proof}

\begin{theorem}[Fixed-design grid consistency]
\label{thm:grid_consistency}
Under Assumptions~\ref{ass:geometry}--\ref{ass:grid_consistency}, let
$\refvar{\Gop_\theta^M}$ use the continuum spectral/attention stack and the decoder
reference $\refvar{J_*}$ in Lemma~\ref{lem:decoder_interp}. For either sampling
reference in Eq.~\eqref{eq:sampling_error_scale},
\begin{equation}
 \begin{aligned}
 \|\refvar{\overline{\Gop}_{\theta}^{h_g,M}}(\xq)-\refvar{\Gop_\theta^M}(\xq)\|
 &\le C_{\mathrm{grid},M}\discvar{h_g},\\
 \|\discvar{\Gop_{\theta}^{\mathcal S,h_g,M}}(\xq)-\refvar{\Gop_\theta^M}(\xq)\|
 &\le C_{\theta,M,h_g}L_{\mathrm{post}}\discvar{\mathcal E_{\mathcal S}}
 +C_{\mathrm{grid},M}\discvar{h_g}.
 \end{aligned}
 \label{eq:grid_operator_bound}
\end{equation}
\end{theorem}

\begin{proof}
The reference post-encoding field
$z_*(\bx)=\psi_\theta([r_*(\bx)\|\rho(\bx)])$ is bounded and Lipschitz.
For fixed encoder bandwidth, this follows by differentiating its positive
Gaussian numerator and denominator; for local bandwidth it follows from the
assumed regularity of $\zeta$. Let
$\discvar{H_{h_g}}=\discvar{\Phi_{\theta,h_g}^M}(\refvar{z_*}|_{\Xi_{h_g}})$ and
$\refvar{H_*}=\refvar{\Phi_\theta^M}(\refvar{z_*})$.
Lemma~\ref{lem:grid_consistency} gives
$\|\discvar{H_{h_g}}-\refvar{H_*}|_{\Xi_{h_g}}\|_\infty
\le C_{\mathrm{quad},M}\discvar{h_g}$.
The common local skip cancels, while adding and subtracting the gather of
$H_*$ yields
\begin{equation}
 \begin{aligned}
 \|\refvar{\overline{\Gop}_{\theta}^{h_g,M}}(\xq)-\refvar{\Gop_\theta^M}(\xq)\|
 &\le L_{\mathrm{dec}}
 \|\discvar{J_{h_g}}(\discvar{H_{h_g}}-\refvar{H_*}|_{\Xi_{h_g}})(\xq)\|+L_{\mathrm{dec}}
 \|\discvar{J_{h_g}}(\refvar{H_*}|_{\Xi_{h_g}})(\xq)
 -\refvar{J_*}\refvar{H_*}(\xq)\|\\
 &\le L_{\mathrm{dec}}
 (C_{\mathrm{quad},M}+C_{\mathrm{dec},H_*})\discvar{h_g}.
 \end{aligned}
 \label{eq:grid_theorem_decomposition}
\end{equation}
This defines $C_{\mathrm{grid},M}$. Adding the sensor bound in
Theorem~\ref{thm:convergence} proves the joint estimate.
\end{proof}

\subsection{Perturbation Stability of the Spectral State}

\begin{theorem}[Stabilized spectral descriptor]
\label{thm:spectral}
Let $P_j,P'_j\ge0$ be spectra on the same frequency set, with
$E_j=P_j+\varepsilon_s$, $E'_j=P'_j+\varepsilon_s$, and
$\Delta_j=E'_j-E_j$. If $\eta=\|\Delta\|_1/\|E\|_1<1$, then
\begin{equation}
 |\alpha(P')-\alpha(P)|\le\frac\eta{1-\eta},
 \qquad
 |\eta_b(P')-\eta_b(P)|\le\frac\eta{1-\eta}.
 \label{eq:spectral_descriptor_perturbation}
\end{equation}
If also $\max_j|\Delta_j|/E_j\le\eta_\infty<1$, the clipped slope obeys
\begin{equation}
 |\widetilde\beta(P')-\widetilde\beta(P)|
 \le C_\beta\frac{\eta_\infty}{1-\eta_\infty},
 \qquad C_\beta=\frac{\sqrt{|\mathcal J|}\,\|x\|_2}{D_\beta}.
 \label{eq:spectral_slope_perturbation}
\end{equation}
\end{theorem}

\begin{proof}
Write $A=\sum_jE_j$ and $D=\sum_j\Delta_j$, so
$A+D\ge(1-\eta)A$. Both the centroid and band ratios have the form
$F(E)=\sum_j E_js_j/(A+\tau)$ with $s_j\in[0,1]$ and $\tau>0$:
for the centroid use $s_j=k_j/k_{\max}$ and
$\tau=\varepsilon_s/k_{\max}$; for band $b$ use
$s_j=\mathbf1\{j\in\mathcal B_b\}$ and $\tau=\varepsilon_s$.
Thus $F(E)\in[0,1]$, and expansion of the ratio gives
\begin{equation}
 \begin{aligned}
 F(E+\Delta)-F(E)
 &=\frac{\sum_j\Delta_js_j-F(E)\sum_j\Delta_j}{A+D+\tau}\\
 &=\frac{\sum_j\Delta_j(s_j-F(E))}{A+D+\tau},\\
 |F(E+\Delta)-F(E)|
 &\le\frac{\|\Delta\|_1}{(1-\eta)A+\tau}
 \le\frac\eta{1-\eta}.
 \end{aligned}
 \label{eq:alpha_signed_identity}
\end{equation}
This proves both ratio bounds for signed power changes.
For the slope, let $\delta_j=\Delta_j/E_j$.
Since $\sum_jx_j=0$, the centered log-power difference reduces to
\begin{equation}
 \begin{aligned}
 \beta(P')-\beta(P)
 &=-\frac1{D_\beta}\sum_jx_j
 \left[\log E'_j-\log E_j
 -\frac1{|\mathcal J|}\sum_r(\log E'_r-\log E_r)\right]\\
 &=-\frac1{D_\beta}\sum_jx_j\log(1+\delta_j),\\
 |\beta(P')-\beta(P)|
 &\le\frac{\|x\|_2\,\|\log(1+\delta)\|_2}{D_\beta}
 \le C_\beta\frac{\eta_\infty}{1-\eta_\infty}.
 \end{aligned}
 \label{eq:beta_difference_identity}
\end{equation}
The last inequality uses
$|\log(1+\delta_j)|\le-\log(1-\eta_\infty)
\le\eta_\infty/(1-\eta_\infty)$. Clipping to $[-10,10]$ is
$1$-Lipschitz, completing the result.
\end{proof}

\textbf{\textit{Amplitude dependence.}}
For positive raw powers, omit the additive power and ratio stabilizers and
write the resulting statistics as $\alpha_0,\eta_{b,0},\beta_0$, keeping
$D_\beta$ fixed. Multiplying the latent field by $a>0$ scales raw power by
$a^2$. Common factors cancel from the idealized centroid and band ratios,
while the centered slope satisfies
\begin{equation}
 \beta_0(a^2P)
 =-\frac{\sum_jx_j(2\log a+\log P_j)}{D_\beta}
 =-\frac{2\log a\sum_jx_j}{D_\beta}
  -\frac{\sum_jx_j\log P_j}{D_\beta}
 =\beta_0(P),
 \label{eq:amp_beta}
\end{equation}
Thus the idealized state measures relative spectral shape. In the stabilized
state, $E'_j=a^2P_j+\varepsilon_s$ gives a continuous amplitude dependence,
which approaches the idealized statistics as power grows relative to the
stabilizers. General additive field changes produce signed per-bin power
perturbations through their Fourier cross terms; Theorem~\ref{thm:spectral}
quantifies the corresponding descriptor change.

\section{Experimental Setup Details}
\label{app:setup}

\textbf{\textit{Architecture.}}
We evaluate three model scales with hidden widths $d=32$ (Small),
$d=64$ (Medium), and $d=128$ (Large). Within each dataset, the
three scales differ only in hidden width. All configurations use
a latent grid with $n_g=32$ points per axis, a Fourier cutoff of $M=4$, attention width $d_a=32$, and depthwise spectral
mixing. The temporal encoder consists of a linear projection
followed by layer normalization and GELU, while \CQD applies
layer normalization followed by a linear readout.
The coordinate encoding uses $16$ random frequencies, whose
sine and cosine components yield $d_\rho=32$ features.
We use $L=2$ operator layers for PEMS-Stream and CA-Stream
and $L=1$ for AIR-Stream, with $B=4$ spectral bands for
PEMS-Stream and $B=3$ for CA-Stream and AIR-Stream.
The field readout is combined with a learnable local-history
skip connection shared across sensors.

\textbf{\textit{Coordinate assignment.}}
We use three types of static descriptors: geographic coordinates,
numerical attributes, and categorical attributes. Numerical attributes
are median-imputed and standardized, while categorical attributes
are one-hot encoded. The processed numerical and categorical
attributes are jointly reduced to two dimensions using PCA.
Preprocessing statistics and the PCA projection are estimated
from the first period and reused in subsequent periods.
Within each period, geographic coordinates and PCA outputs are
separately normalized and then blended to obtain a continuous
two-dimensional position for each sensor; datasets without
additional attributes use geographic coordinates alone.
These positions guide a one-to-one assignment of sensors to fixed,
uniformly spaced points on a square grid in $\Omega=[0,1]^2$.
The assigned grid points serve as sensor coordinates for both
the encoder and the query decoder. The assignment grid and
sensor allocation are reconstructed for each period's sensor set,
while the latent field grid and learned parameter shapes remain fixed.

\textbf{\textit{Optimization.}}
We use AdamW with learning rates of $5\times10^{-3}$ for
PEMS-Stream and $10^{-2}$ for CA-Stream and AIR-Stream,
a batch size of $64$, and gradient norm clipping at $5.0$.
Training runs for at most $200$ epochs per period and stops
early if validation MAE fails to improve for more than
$10$ consecutive epochs. The training objective is masked
MAE, computed in standardized target space:
\begin{equation}
\mathcal{L}_{\mathrm{MAE}}
=
\frac{\sum_{i,t} m_{it}\left|\hat y_{it}-y_{it}\right|}
{\sum_{i,t}m_{it}+\varepsilon},
\label{eq:loss_mae}
\end{equation}
where $m_{it}$ indicates finite target entries after preprocessing,
and the sums run over the mini-batch.

\textbf{\textit{Continual protocol.}}
We train the model from scratch in period $1$.
For each subsequent period, we initialize it from the selected
parameters of the previous period and fine-tune the same shared
parameter set on current training samples. Observation histories
and coordinate assignments are updated for the current sensor set,
while the latent-grid dimensions and learned parameter shapes
remain unchanged.

\textbf{\textit{Metrics and evaluation protocol.}}
We evaluate forecasting performance using MAE, RMSE, and MAPE.
All experiments follow the stream construction and evaluation
protocol of \citet{liu2026stbp}, using the same data splits,
prediction horizons, and evaluation masks across methods.

\section{Algorithms}
\label{app:algorithm}

Algorithm~\ref{alg:forward} summarizes the forward pass described
in Section~\ref{sec:method}, while Algorithm~\ref{alg:continual}
presents the continual training protocol described in
Sections~\ref{sec:continual} and~\ref{sec:exp:setup}.
Both algorithms follow the notation in Appendix~\ref{app:notation}
and use the optimization settings specified in
Appendix~\ref{app:setup}.

\begin{algorithm}[t]
\caption{\model{} forward pass for one period}
\label{alg:forward}
\begin{algorithmic}[1]
\Require histories $\{\bm u_i\}_{i\in\mathcal V_p}$, coordinates $\{\bm x_i\}_{i\in\mathcal V_p}$, query set $\mathcal V_q\subseteq\mathcal V_p$, parameters $\theta$
\Ensure forecasts $\{\hat{\bm y}_q\}_{q\in\mathcal V_q}$
\State $\bm e_i \gets \tau_\theta(\bm u_i)$ for all $i\in\mathcal V_p$ \Comment{temporal encoding}
\For{$g=1$ to $G$}
    \State $a_{gi} \gets \kappa_{\sigma_{\mathrm{enc}}}(\bm\xi_g,\bm x_i)\big/\sum_{j\in\mathcal V_p}\kappa_{\sigma_{\mathrm{enc}}}(\bm\xi_g,\bm x_j)$
    \State $\bm r_g \gets \sum_{i\in\mathcal V_p} a_{gi}\bm e_i$ \Comment{normalized kernel lift}
    \State $\bm z_g \gets \psi_\theta([\bm r_g\,\|\,\rho(\bm\xi_g)])$ \Comment{position-aware latent code}
\EndFor
\State $\bar z_g \gets d^{-1}\mathbf{1}_d^\top\bm z_g$; $\widehat{\bar z} \gets \mathrm{DFT}(\bar z)$
\State $E_m \gets |\widehat{\bar z}(\bm\omega_m)|^2+\epsilon_s$ and $\omega_m \gets \|\bm\omega_m\|_2+\epsilon_f$ for all $m\in\mathcal J$
\State $\bm c \gets [\alpha,\tilde\beta,\eta_1,\ldots,\eta_B]^\top$ \Comment{regime descriptor, Eq.~\eqref{eq:sre_statistics_main}}
\For{$\ell=1$ to $L$}
    \State $(\bm\gamma,\bm\delta) \gets 0.1\tanh(\mathrm{MLP}_{\theta_\ell}(\bm c))$; $\mathbf{Z}^{\mathrm c} \gets \mathbf{Z}\odot(\mathbf{1}_G(\mathbf{1}_d+\bm\gamma)^\top)+\mathbf{1}_G\bm\delta^\top$ \Comment{FiLM conditioning}
    \State $\mathbf{Z} \gets \mathbf{Z}+\mathcal P_{\theta_\ell}(\mathbf{Z}^{\mathrm c})+\mathcal A_{\theta_\ell}(\mathbf{Z}^{\mathrm c})$ \Comment{spectral and attention paths}
    \State $\mathbf{Z} \gets \LN_\ell(\mathbf{Z})+\mathbf{1}_G\bm b_{\mathrm{res},\ell}^\top$ \Comment{residual normalization}
\EndFor
\State $\mathbf{H} \gets \mathbf{Z}$
\For{$q\in\mathcal V_q$}
    \State $b_{qg} \gets \kappa_{\sigma_{\mathrm{dec}}}(\bm x_q,\bm\xi_g)\big/\sum_{r}\kappa_{\sigma_{\mathrm{dec}}}(\bm x_q,\bm\xi_r)$
    \State $\tilde{\bm h}_q \gets \sum_g b_{qg}\bm h_g$
    \State $\hat{\bm y}_q \gets \mathbf{W}_{\mathrm{out}}\LN(\tilde{\bm h}_q)+\bm b_{\mathrm{out}}+\mathbf{W}_{\mathrm{skip}}\bm u_q+\bm b_{\mathrm{skip}}$ \Comment{coordinate query}
\EndFor
\State \Return $\{\hat{\bm y}_q\}_{q\in\mathcal V_q}$
\end{algorithmic}
\end{algorithm}

\begin{algorithm}[t]
\caption{Continual fine-tuning of \model{} across the sensor stream}
\label{alg:continual}
\begin{algorithmic}[1]
\Require period data $\{(\mathcal V_p,\bm s_{\mathcal V_p},\mathcal D_p)\}_{p=1}^{P}$,
learning rate $\eta$, batch size $B_s$, maximum epochs $E_{\max}$,
patience $E_{\mathrm{pat}}$
\Ensure parameters $\theta^{(P)}$
\State initialize $\theta$ from scratch
\For{$p=1$ to $P$}
    \State $\bm x_i^{(p)} \gets \chi_p(\bm s_i)$ for all $i\in\mathcal V_p$
    \If{$p>1$}
        \State $\theta \gets \theta^{(p-1)}$
    \EndIf
    \State reset AdamW state; initialize checkpoint set $\mathcal H_p\gets\varnothing$
    \For{epoch $=1$ to $E_{\max}$}
        \ForAll{mini-batches $\mathcal B$ from the training split of $\mathcal D_p$ with batch size $B_s$}
            \State $\{\hat{\bm y}_q\}\gets
            \textsc{Forward}(\mathcal B;\{\bm x_i^{(p)}\}_{i\in\mathcal V_p},\theta)$
            \Comment{Algorithm~\ref{alg:forward}}
            \State $\mathcal L \gets
            \sum_{i,t}m_{it}|\hat y_{it}-y_{it}|
            \big/(\sum_{i,t}m_{it}+\varepsilon)$
            \Comment{masked MAE, Eq.~\eqref{eq:loss_mae}}
            \State update $\theta$ using AdamW with learning rate $\eta$ and gradient clipping
        \EndFor
        \State evaluate MAE on the validation split of $\mathcal D_p$;
        save $\theta$ in $\mathcal H_p$ if it improves the best value so far
        \If{validation MAE has not improved for $E_{\mathrm{pat}}$ epochs}
            \State \textbf{break}
        \EndIf
    \EndFor
    \State select $\theta^{(p)}$ from $\mathcal H_p$ by validation MAE with optional checkpoint parameter averaging
\EndFor
\State \Return $\theta^{(P)}$
\end{algorithmic}
\end{algorithm}

\section{Baseline Descriptions}
\label{app:baselines}

Section~\ref{sec:exp:setup} compares \model{} with four families
of methods. We describe each baseline and its role in the
comparison. All methods follow the evaluation protocol,
prediction horizons, data splits, and evaluation masks
specified in Appendix~\ref{app:setup}.

\textbf{\textit{\ding{182} General spatio-temporal forecasting.}}
\textsc{GWNet}~\citep{wu2019graph} combines graph convolution
with dilated causal temporal convolution and learns a self-adaptive
adjacency matrix to capture spatial dependencies without a
predefined graph.
\textsc{STID}~\citep{shao2022stid} incorporates spatial and temporal
identity embeddings into a lightweight predictor, demonstrating
that a simple architecture with explicit identity information
can remain competitive with more elaborate designs.
\textsc{iTransformer}~\citep{liu2023itransformer} inverts the
attention axis of a transformer by treating each series as a
token, enabling effective multivariate forecasting without an
explicit spatial prior. Together, these methods represent
graph-based, identity-based, and sequence-based approaches
to forecasting from the same observations.

\textbf{\textit{\ding{183} Neural operators.}}
\textsc{DeepONet}~\citep{lu2021learning} approximates nonlinear
operators through a branch--trunk factorization: the branch
encodes the input function, while the trunk encodes the query
coordinate. This architecture is supported by a universal
approximation theorem for operators.
\textsc{FNO}~\citep{li2020fourier} parameterizes a global
convolution kernel in the spectral domain and retains a small
set of Fourier modes, allowing transfer across resolutions.
Both methods share the operator perspective of \model{},
mapping input functions to output functions. In our experiments,
they are trained on fixed meshes and serve as operator baselines
for comparison with forecasting on irregular, expanding
sensor layouts.

\textbf{\textit{\ding{184} Continual spatio-temporal learning.}}
\textsc{TrafficStream}~\citep{chen2021trafficstream} combines
experience replay with parameter smoothing for growing road networks.
\textsc{STKEC}~\citep{wang2023knowledge} expands and consolidates
task-specific knowledge, enabling new sensors to reuse learned structure.
\textsc{PECPM}~\citep{wang2023pattern} expands and consolidates
a bank of spatio-temporal patterns as the graph evolves.
\textsc{STRAP}~\citep{zhang2026strap} retrieves spatio-temporal
patterns to improve out-of-distribution generalization.
\textsc{EAC}~\citep{chen2025eac} studies continual tuning
by expanding and compressing adapted parameters.
\textsc{EVOLVE}~\citep{yang2027evolve} addresses changing
traffic-network topologies through continual learning.
All six retain the current sensor set in their representations;
\model{} treats expansion as changing sampling locations
on a fixed latent field.

\textbf{\textit{\ding{185} Online adaptation.}}
\textsc{DOL}~\citep{wang2025dol} uses AST-Net, which augments
a frozen Graph WaveNet backbone with location-specific learners
that adapt online.
We implement \textsc{ST-TTC}~\citep{chen2025stttc} on an
\textsc{EAC} backbone and calibrate its forecasts through online
updates to spectral amplitude and phase offsets.

\section{Statistical Significance Analysis}
\label{app:statistical_significance}

\begin{table}[!htbp]
  \centering
  \begin{threeparttable}
  \caption{Wilcoxon signed-rank test results (one-sided,
  H$_1$: STFO-L $<$ Baseline) for the main forecasting results.
  $N=27$ paired entries per comparison (dataset $\times$ horizon $\times$ seed).
  Shading and stars indicate nominal Holm-adjusted levels.}
  \label{tab:table1_wilcoxon}
  \begingroup
  \small
  \definecolor{STFOsigStrong}{RGB}{231,232,250}
  \definecolor{STFOsigMedium}{RGB}{244,244,252}
  \setlength{\tabcolsep}{5.5pt}
  \renewcommand{\arraystretch}{0.95}
  \begin{tabular}{lccccrccc}
    \toprule
    \textbf{Baseline} & $N$ & \textbf{STFO-L Wins}
    & \textbf{Baseline Wins} & \textbf{Ties}
    & \textbf{Mean $\Delta$\%} & $W^{+}$
    & $p$\textbf{-value} & $p_{\mathrm{Holm}}$ \\
    \midrule
    \multicolumn{9}{c}{\textbf{Metric: MAE}\quad H$_1$: STFO-L $<$ Baseline} \\
    \midrule
    \rowcolor{STFOsigStrong} GWNet & 27 & 27 & 0 & 0 & 37.07\% & 0.0 & $<0.0001$ & $<0.0001^{***}$ \\
    \rowcolor{STFOsigStrong} STID & 27 & 27 & 0 & 0 & 48.23\% & 0.0 & $<0.0001$ & $<0.0001^{***}$ \\
    \rowcolor{STFOsigStrong} iTrans & 27 & 27 & 0 & 0 & 36.40\% & 0.0 & $<0.0001$ & $<0.0001^{***}$ \\
    \rowcolor{STFOsigStrong} DeepONet & 27 & 27 & 0 & 0 & 11.84\% & 0.0 & $<0.0001$ & $<0.0001^{***}$ \\
    \rowcolor{STFOsigStrong} FNO & 27 & 26 & 1 & 0 & 13.25\% & 1.0 & $<0.0001$ & $<0.0001^{***}$ \\
    \rowcolor{STFOsigStrong} TrafficStream & 27 & 27 & 0 & 0 & 27.06\% & 0.0 & $<0.0001$ & $<0.0001^{***}$ \\
    \rowcolor{STFOsigStrong} STKEC & 27 & 27 & 0 & 0 & 27.07\% & 0.0 & $<0.0001$ & $<0.0001^{***}$ \\
    \rowcolor{STFOsigStrong} PECPM & 27 & 27 & 0 & 0 & 26.89\% & 0.0 & $<0.0001$ & $<0.0001^{***}$ \\
    \rowcolor{STFOsigStrong} STRAP & 27 & 27 & 0 & 0 & 17.87\% & 0.0 & $<0.0001$ & $<0.0001^{***}$ \\
    \rowcolor{STFOsigStrong} EAC & 27 & 27 & 0 & 0 & 24.00\% & 0.0 & $<0.0001$ & $<0.0001^{***}$ \\
    \rowcolor{STFOsigStrong} EVOLVE & 27 & 27 & 0 & 0 & 18.18\% & 0.0 & $<0.0001$ & $<0.0001^{***}$ \\
    \rowcolor{STFOsigStrong} DOL & 27 & 24 & 3 & 0 & 4.27\% & 18.0 & $<0.0001$ & $<0.0001^{***}$ \\
    \rowcolor{STFOsigStrong} ST-TTC & 27 & 27 & 0 & 0 & 14.79\% & 0.0 & $<0.0001$ & $<0.0001^{***}$ \\
    \midrule
    \multicolumn{9}{c}{\textbf{Metric: RMSE}\quad H$_1$: STFO-L $<$ Baseline} \\
    \midrule
    \rowcolor{STFOsigStrong} GWNet & 27 & 27 & 0 & 0 & 33.28\% & 0.0 & $<0.0001$ & $<0.0001^{***}$ \\
    \rowcolor{STFOsigStrong} STID & 27 & 27 & 0 & 0 & 44.32\% & 0.0 & $<0.0001$ & $<0.0001^{***}$ \\
    \rowcolor{STFOsigStrong} iTrans & 27 & 27 & 0 & 0 & 32.68\% & 0.0 & $<0.0001$ & $<0.0001^{***}$ \\
    \rowcolor{STFOsigStrong} DeepONet & 27 & 27 & 0 & 0 & 10.67\% & 0.0 & $<0.0001$ & $<0.0001^{***}$ \\
    \rowcolor{STFOsigStrong} FNO & 27 & 26 & 1 & 0 & 12.15\% & 1.0 & $<0.0001$ & $<0.0001^{***}$ \\
    \rowcolor{STFOsigStrong} TrafficStream & 27 & 27 & 0 & 0 & 23.58\% & 0.0 & $<0.0001$ & $<0.0001^{***}$ \\
    \rowcolor{STFOsigStrong} STKEC & 27 & 27 & 0 & 0 & 23.84\% & 0.0 & $<0.0001$ & $<0.0001^{***}$ \\
    \rowcolor{STFOsigStrong} PECPM & 27 & 27 & 0 & 0 & 23.43\% & 0.0 & $<0.0001$ & $<0.0001^{***}$ \\
    \rowcolor{STFOsigStrong} STRAP & 27 & 27 & 0 & 0 & 17.27\% & 0.0 & $<0.0001$ & $<0.0001^{***}$ \\
    \rowcolor{STFOsigStrong} EAC & 27 & 27 & 0 & 0 & 19.80\% & 0.0 & $<0.0001$ & $<0.0001^{***}$ \\
    \rowcolor{STFOsigStrong} EVOLVE & 27 & 27 & 0 & 0 & 17.35\% & 0.0 & $<0.0001$ & $<0.0001^{***}$ \\
    \rowcolor{STFOsigStrong} DOL & 27 & 27 & 0 & 0 & 4.93\% & 0.0 & $<0.0001$ & $<0.0001^{***}$ \\
    \rowcolor{STFOsigStrong} ST-TTC & 27 & 27 & 0 & 0 & 11.79\% & 0.0 & $<0.0001$ & $<0.0001^{***}$ \\
    \midrule
    \multicolumn{9}{c}{\textbf{Metric: MAPE}\quad H$_1$: STFO-L $<$ Baseline} \\
    \midrule
    \rowcolor{STFOsigStrong} GWNet & 27 & 27 & 0 & 0 & 36.29\% & 0.0 & $<0.0001$ & $<0.0001^{***}$ \\
    \rowcolor{STFOsigStrong} STID & 27 & 27 & 0 & 0 & 49.18\% & 0.0 & $<0.0001$ & $<0.0001^{***}$ \\
    \rowcolor{STFOsigStrong} iTrans & 27 & 27 & 0 & 0 & 40.43\% & 0.0 & $<0.0001$ & $<0.0001^{***}$ \\
    \rowcolor{STFOsigStrong} DeepONet & 27 & 26 & 1 & 0 & 17.11\% & 2.0 & $<0.0001$ & $<0.0001^{***}$ \\
    \rowcolor{STFOsigStrong} FNO & 27 & 24 & 3 & 0 & 18.16\% & 15.0 & $<0.0001$ & $<0.0001^{***}$ \\
    \rowcolor{STFOsigStrong} TrafficStream & 27 & 27 & 0 & 0 & 20.97\% & 0.0 & $<0.0001$ & $<0.0001^{***}$ \\
    \rowcolor{STFOsigStrong} STKEC & 27 & 27 & 0 & 0 & 20.13\% & 0.0 & $<0.0001$ & $<0.0001^{***}$ \\
    \rowcolor{STFOsigStrong} PECPM & 27 & 27 & 0 & 0 & 21.22\% & 0.0 & $<0.0001$ & $<0.0001^{***}$ \\
    \rowcolor{STFOsigStrong} STRAP & 27 & 27 & 0 & 0 & 15.01\% & 0.0 & $<0.0001$ & $<0.0001^{***}$ \\
    \rowcolor{STFOsigStrong} EAC & 27 & 27 & 0 & 0 & 19.80\% & 0.0 & $<0.0001$ & $<0.0001^{***}$ \\
    \rowcolor{STFOsigStrong} EVOLVE & 27 & 26 & 1 & 0 & 16.01\% & 1.0 & $<0.0001$ & $<0.0001^{***}$ \\
    \rowcolor{STFOsigMedium} DOL & 27 & 20 & 7 & 0 & 2.87\% & 80.0 & $0.0038$ & $0.0038^{**}$ \\
    \rowcolor{STFOsigStrong} ST-TTC & 27 & 27 & 0 & 0 & 20.39\% & 0.0 & $<0.0001$ & $<0.0001^{***}$ \\
    \bottomrule
  \end{tabular}
  \begin{tablenotes}[flushleft]
  \footnotesize
  \setlength{\itemsep}{1pt}
  \item \colorbox{STFOsigStrong}{\strut\hspace{7pt}}
    $^{***}p_{\mathrm{Holm}}<0.001$
    \quad \colorbox{STFOsigMedium}{\strut\hspace{7pt}}
    $^{**}p_{\mathrm{Holm}}<0.01$.
    Holm correction is applied across 13 baselines per metric.
  \item Mean $\Delta$\% averages relative error reductions across nine
    dataset--horizon settings using three-seed means, excluding \texttt{Avg.} rows.
    $W^{+}$ is the positive-rank sum for STFO-L $-$ Baseline.
\end{tablenotes}
  \endgroup
  \end{threeparttable}
\end{table}

We compare \model-Large with 13 baselines using one-sided
Wilcoxon signed-rank tests on 27 paired entries per metric,
with Holm correction across baselines within each metric.
Table~\ref{tab:table1_wilcoxon} reports adjusted $p$-values
below $0.0001$ for all comparisons except MAPE against
\textsc{DOL} ($p_{\mathrm{Holm}}<0.01$). We report these
$p$-values for reference (some paired entries are dependent).

\section{Full New-Sensor Adaptation Results}
\label{app:new_sensor_adaptation}

We evaluate performance on existing and newly introduced
sensors across expansion periods. Zero-shot evaluation
uses the model trained in the preceding period without
adaptation, whereas full adaptation fine-tunes that model using
all available current-period supervision.

\textbf{\textit{PEMS-Stream.}}
Tables~\ref{tab:new_sensor_adaptation_pems_mae},
\ref{tab:new_sensor_adaptation_pems_rmse},
and~\ref{tab:new_sensor_adaptation_pems_mape}
show that \model-Large leads across all metrics, horizons,
sensor groups, and six periods under full adaptation.
Its zero-shot advantage is stronger on existing sensors,
with the lowest average MAE and RMSE in five periods.
On new sensors, it achieves the lowest average RMSE in
2012, 2015, and 2017, and the lowest average MAE and MAPE
in 2015.

\textbf{\textit{CA-Stream.}}
Tables~\ref{tab:new_sensor_adaptation_ca_mae},
\ref{tab:new_sensor_adaptation_ca_rmse},
and~\ref{tab:new_sensor_adaptation_ca_mape}
show that \model-Large ranks first across all three periods,
metrics, horizons, and sensor groups under full adaptation.
Under zero-shot evaluation, it achieves the lowest average
MAE, RMSE, and MAPE on existing sensors in P1 and P2.
On new sensors, it leads in average MAPE in P1 and in all
three average metrics in P2.

\textbf{\textit{AIR-Stream.}}
Tables~\ref{tab:new_sensor_adaptation_air_rmse}
and~\ref{tab:new_sensor_adaptation_air_mape}
complement Table~\ref{tab:new_sensor_adaptation}.
In zero-shot evaluation, \model-Large leads in average RMSE
and MAPE on existing sensors in every period and on new
sensors in 2017 and 2018. With full adaptation, it leads
in RMSE and MAPE at every horizon for both groups in
2017 and 2018; in 2019, it leads in average MAPE for both
groups but not in RMSE at every horizon.

Overall, \model-Large achieves the lowest errors across all periods,
horizons, and sensor groups on both traffic streams under full adaptation.

\providecommand{\stbpstd}[1]{%
  {\normalfont\fontsize{4.6}{5.0}\selectfont\ensuremath{\pm#1}}%
}
\begin{table*}[t]
\centering
\caption{Period-wise MAE on existing and newly introduced sensors of PEMS-Stream. Results are mean $\pm$ standard deviation over three runs. Lower is better; \best{bold} and \second{underlined} denote the best and second-best methods within each period, sensor group, adaptation budget, and horizon.}
\label{tab:new_sensor_adaptation_pems_mae}
\setlength{\tabcolsep}{2.5pt}
\renewcommand{\arraystretch}{0.99}
\fontsize{6.2}{7.1}\selectfont
\resizebox{\linewidth}{!}{%
\begin{tabular}{@{}ll cccc cccc cccc@{}}
\toprule
\multirow{2}{*}{\textbf{Budget}} & \multirow{2}{*}{\textbf{Sensor group}}
& \multicolumn{4}{c}{FNO}
& \multicolumn{4}{c}{EVOLVE}
& \multicolumn{4}{c}{\textbf{\model-Large}}
\\[-0.15em]
\cmidrule(lr){3-6}\cmidrule(lr){7-10}\cmidrule(lr){11-14}
& & $h=3$ & $h=6$ & $h=12$ & \textit{Avg.} & $h=3$ & $h=6$ & $h=12$ & \textit{Avg.} & $h=3$ & $h=6$ & $h=12$ & \textit{Avg.} \\ 
\midrule
\rowcolor{datasetgray}
\multicolumn{14}{@{}l}{\hspace{0.25em}\textbf{2012}: 655 existing sensors; 60 newly introduced sensors} \\[-0.1em]
\multirow{2}{*}{Zero-shot} & Existing & \hbest{12.46}{12.46}{12.97}\stbpstd{0.08} & \hbest{14.14}{14.14}{14.64}\stbpstd{0.04} & \hsecond{17.77}{17.28}{19.12}\stbpstd{0.04} & \hsecond{14.44}{14.34}{15.12}\stbpstd{0.04} & \hsecond{12.97}{12.46}{12.97}\stbpstd{0.13} & \hval{14.64}{14.14}{14.64}\stbpstd{0.24} & \hval{19.12}{17.28}{19.12}\stbpstd{0.35} & \hval{15.12}{14.34}{15.12}\stbpstd{0.24} & \hbest{12.46}{12.46}{12.97}\stbpstd{0.29} & \hsecond{14.25}{14.14}{14.64}\stbpstd{0.48} & \hbest{17.28}{17.28}{19.12}\stbpstd{0.65} & \hbest{14.34}{14.34}{15.12}\stbpstd{0.43} \\ 
 & Newly introduced & \hval{13.44}{12.74}{13.44}\stbpstd{0.13} & \hval{15.93}{14.39}{15.93}\stbpstd{0.14} & \hval{21.44}{18.46}{21.44}\stbpstd{0.35} & \hval{16.40}{14.76}{16.40}\stbpstd{0.16} & \hbest{12.74}{12.74}{13.44}\stbpstd{0.13} & \hbest{14.39}{14.39}{15.93}\stbpstd{0.27} & \hbest{18.46}{18.46}{21.44}\stbpstd{0.41} & \hbest{14.76}{14.76}{16.40}\stbpstd{0.26} & \hsecond{12.85}{12.74}{13.44}\stbpstd{0.34} & \hsecond{15.08}{14.39}{15.93}\stbpstd{0.55} & \hsecond{18.79}{18.46}{21.44}\stbpstd{0.97} & \hsecond{15.19}{14.76}{16.40}\stbpstd{0.57} \\
\bottomrule
\end{tabular}%
}
\end{table*}

\begin{table*}[t]
\centering
\caption{Period-wise RMSE on existing and newly introduced sensors of PEMS-Stream. Results are mean $\pm$ standard deviation over three runs. Lower is better; \best{bold} and \second{underlined} denote the best and second-best methods within each period, sensor group, adaptation budget, and horizon.}
\label{tab:new_sensor_adaptation_pems_rmse}
\setlength{\tabcolsep}{2.5pt}
\renewcommand{\arraystretch}{0.99}
\fontsize{6.2}{7.1}\selectfont
\resizebox{\linewidth}{!}{%
\begin{tabular}{@{}ll cccc cccc cccc@{}}
\toprule
\multirow{2}{*}{\textbf{Budget}} & \multirow{2}{*}{\textbf{Sensor group}}
& \multicolumn{4}{c}{FNO}
& \multicolumn{4}{c}{EVOLVE}
& \multicolumn{4}{c}{\textbf{\model-Large}}
\\[-0.15em]
\cmidrule(lr){3-6}\cmidrule(lr){7-10}\cmidrule(lr){11-14}
& & $h=3$ & $h=6$ & $h=12$ & \textit{Avg.} & $h=3$ & $h=6$ & $h=12$ & \textit{Avg.} & $h=3$ & $h=6$ & $h=12$ & \textit{Avg.} \\ 
\midrule
\rowcolor{datasetgray}
\multicolumn{14}{@{}l}{\hspace{0.25em}\textbf{2012}: 655 existing sensors; 60 newly introduced sensors} \\[-0.1em]
\multirow{2}{*}{Zero-shot} & Existing & \hsecond{19.35}{19.25}{20.56}\stbpstd{0.05} & \hsecond{22.25}{22.02}{23.35}\stbpstd{0.01} & \hsecond{28.38}{26.95}{30.36}\stbpstd{0.04} & \hsecond{22.73}{22.21}{24.04}\stbpstd{0.02} & \hval{20.56}{19.25}{20.56}\stbpstd{0.32} & \hval{23.35}{22.02}{23.35}\stbpstd{0.51} & \hval{30.36}{26.95}{30.36}\stbpstd{0.59} & \hval{24.04}{22.21}{24.04}\stbpstd{0.48} & \hbest{19.25}{19.25}{20.56}\stbpstd{0.31} & \hbest{22.02}{22.02}{23.35}\stbpstd{0.55} & \hbest{26.95}{26.95}{30.36}\stbpstd{0.69} & \hbest{22.21}{22.21}{24.04}\stbpstd{0.48} \\ 
 & Newly introduced & \hsecond{21.37}{20.86}{21.54}\stbpstd{0.08} & \hval{25.31}{24.27}{25.31}\stbpstd{0.20} & \hval{33.42}{30.40}{33.42}\stbpstd{0.51} & \hval{25.89}{24.52}{25.89}\stbpstd{0.20} & \hval{21.54}{20.86}{21.54}\stbpstd{0.30} & \hsecond{24.32}{24.27}{25.31}\stbpstd{0.52} & \hsecond{31.08}{30.40}{33.42}\stbpstd{0.64} & \hsecond{24.91}{24.52}{25.89}\stbpstd{0.49} & \hbest{20.86}{20.86}{21.54}\stbpstd{0.24} & \hbest{24.27}{24.27}{25.31}\stbpstd{0.49} & \hbest{30.40}{30.40}{33.42}\stbpstd{0.87} & \hbest{24.52}{24.52}{25.89}\stbpstd{0.52} \\
\bottomrule
\end{tabular}%
}
\end{table*}

\begin{table*}[t]
\centering
\caption{Period-wise MAPE on existing and newly introduced sensors of PEMS-Stream. Results are mean $\pm$ standard deviation over three runs. Lower is better; \best{bold} and \second{underlined} denote the best and second-best methods within each period, sensor group, adaptation budget, and horizon.}
\label{tab:new_sensor_adaptation_pems_mape}
\setlength{\tabcolsep}{2.5pt}
\renewcommand{\arraystretch}{0.99}
\fontsize{6.2}{7.1}\selectfont
\resizebox{\linewidth}{!}{%
\begin{tabular}{@{}ll cccc cccc cccc@{}}
\toprule
\multirow{2}{*}{\textbf{Budget}} & \multirow{2}{*}{\textbf{Sensor group}}
& \multicolumn{4}{c}{FNO}
& \multicolumn{4}{c}{EVOLVE}
& \multicolumn{4}{c}{\textbf{\model-Large}}
\\[-0.15em]
\cmidrule(lr){3-6}\cmidrule(lr){7-10}\cmidrule(lr){11-14}
& & $h=3$ & $h=6$ & $h=12$ & \textit{Avg.} & $h=3$ & $h=6$ & $h=12$ & \textit{Avg.} & $h=3$ & $h=6$ & $h=12$ & \textit{Avg.} \\ 
\midrule
\rowcolor{datasetgray}
\multicolumn{14}{@{}l}{\hspace{0.25em}\textbf{2012}: 655 existing sensors; 60 newly introduced sensors} \\[-0.1em]
\multirow{2}{*}{Zero-shot} & Existing & \hval{21.59}{17.41}{21.59}\stbpstd{0.58} & \hval{26.20}{19.82}{26.20}\stbpstd{0.58} & \hval{34.47}{26.06}{34.47}\stbpstd{0.46} & \hval{26.49}{20.61}{26.49}\stbpstd{0.26} & \hbest{17.41}{17.41}{21.59}\stbpstd{0.60} & \hbest{19.82}{19.82}{26.20}\stbpstd{1.46} & \hbest{26.06}{26.06}{34.47}\stbpstd{2.17} & \hbest{20.61}{20.61}{26.49}\stbpstd{1.33} & \hsecond{18.51}{17.41}{21.59}\stbpstd{1.57} & \hsecond{21.93}{19.82}{26.20}\stbpstd{2.27} & \hsecond{26.53}{26.06}{34.47}\stbpstd{3.35} & \hsecond{21.86}{20.61}{26.49}\stbpstd{2.13} \\ 
 & Newly introduced & \hval{37.93}{22.66}{37.93}\stbpstd{1.59} & \hval{50.17}{25.49}{50.17}\stbpstd{2.23} & \hval{75.31}{32.22}{75.31}\stbpstd{4.35} & \hval{51.68}{26.30}{51.68}\stbpstd{2.21} & \hbest{22.66}{22.66}{37.93}\stbpstd{0.69} & \hbest{25.49}{25.49}{50.17}\stbpstd{1.83} & \hbest{32.22}{32.22}{75.31}\stbpstd{2.69} & \hbest{26.30}{26.30}{51.68}\stbpstd{1.81} & \hsecond{29.49}{22.66}{37.93}\stbpstd{4.01} & \hsecond{37.09}{25.49}{50.17}\stbpstd{5.86} & \hsecond{46.44}{32.22}{75.31}\stbpstd{9.29} & \hsecond{36.83}{26.30}{51.68}\stbpstd{5.87} \\
\bottomrule
\end{tabular}%
}
\end{table*}

\begin{table*}[t]
\centering
\caption{Period-wise MAE on existing and newly introduced sensors of CA-Stream. Results are mean $\pm$ standard deviation over three runs. Lower is better; \best{bold} and \second{underlined} denote the best and second-best methods within each period, sensor group, adaptation budget, and horizon.}
\label{tab:new_sensor_adaptation_ca_mae}
\setlength{\tabcolsep}{2.5pt}
\renewcommand{\arraystretch}{0.99}
\fontsize{6.2}{7.1}\selectfont
\resizebox{\linewidth}{!}{%
\begin{tabular}{@{}ll cccc cccc cccc@{}}
\toprule
\multirow{2}{*}{\textbf{Budget}} & \multirow{2}{*}{\textbf{Sensor group}}
& \multicolumn{4}{c}{FNO}
& \multicolumn{4}{c}{EVOLVE}
& \multicolumn{4}{c}{\textbf{\model-Large}}
\\[-0.15em]
\cmidrule(lr){3-6}\cmidrule(lr){7-10}\cmidrule(lr){11-14}
& & $h=3$ & $h=6$ & $h=12$ & \textit{Avg.} & $h=3$ & $h=6$ & $h=12$ & \textit{Avg.} & $h=3$ & $h=6$ & $h=12$ & \textit{Avg.} \\ 
\midrule
\rowcolor{datasetgray}
\multicolumn{14}{@{}l}{\hspace{0.25em}\textbf{P1}: 480 existing sensors; 211 newly introduced sensors} \\[-0.1em]
\multirow{2}{*}{Zero-shot} & Existing & \hsecond{28.95}{26.28}{29.81}\stbpstd{1.41} & \hval{31.25}{28.52}{31.25}\stbpstd{1.85} & \hval{36.58}{34.29}{36.58}\stbpstd{2.55} & \hval{31.80}{29.58}{31.80}\stbpstd{1.88} & \hval{29.81}{26.28}{29.81}\stbpstd{0.27} & \hsecond{30.91}{28.52}{31.25}\stbpstd{0.47} & \hbest{34.29}{34.29}{36.58}\stbpstd{0.73} & \hsecond{31.31}{29.58}{31.80}\stbpstd{0.43} & \hbest{26.28}{26.28}{29.81}\stbpstd{0.24} & \hbest{28.52}{28.52}{31.25}\stbpstd{0.60} & \hsecond{34.95}{34.29}{36.58}\stbpstd{1.22} & \hbest{29.58}{29.58}{31.80}\stbpstd{0.56} \\ 
 & Newly introduced & \hsecond{24.68}{24.52}{25.19}\stbpstd{3.25} & \hbest{24.38}{24.38}{25.18}\stbpstd{0.70} & \hbest{26.28}{26.28}{28.49}\stbpstd{1.44} & \hbest{24.94}{24.94}{26.13}\stbpstd{0.65} & \hbest{24.52}{24.52}{25.19}\stbpstd{0.40} & \hval{25.18}{24.38}{25.18}\stbpstd{0.50} & \hsecond{27.35}{26.28}{28.49}\stbpstd{0.64} & \hsecond{25.52}{24.94}{26.13}\stbpstd{0.45} & \hval{25.19}{24.52}{25.19}\stbpstd{1.02} & \hsecond{25.15}{24.38}{25.18}\stbpstd{0.73} & \hval{28.49}{26.28}{28.49}\stbpstd{1.34} & \hval{26.13}{24.94}{26.13}\stbpstd{0.92} \\
\bottomrule
\end{tabular}%
}
\end{table*}

\begin{table*}[t]
\centering
\caption{Period-wise RMSE on existing and newly introduced sensors of CA-Stream. Results are mean $\pm$ standard deviation over three runs. Lower is better; \best{bold} and \second{underlined} denote the best and second-best methods within each period, sensor group, adaptation budget, and horizon.}
\label{tab:new_sensor_adaptation_ca_rmse}
\setlength{\tabcolsep}{2.5pt}
\renewcommand{\arraystretch}{0.99}
\fontsize{6.2}{7.1}\selectfont
\resizebox{\linewidth}{!}{%
\begin{tabular}{@{}ll cccc cccc cccc@{}}
\toprule
\multirow{2}{*}{\textbf{Budget}} & \multirow{2}{*}{\textbf{Sensor group}}
& \multicolumn{4}{c}{FNO}
& \multicolumn{4}{c}{EVOLVE}
& \multicolumn{4}{c}{\textbf{\model-Large}}
\\[-0.15em]
\cmidrule(lr){3-6}\cmidrule(lr){7-10}\cmidrule(lr){11-14}
& & $h=3$ & $h=6$ & $h=12$ & \textit{Avg.} & $h=3$ & $h=6$ & $h=12$ & \textit{Avg.} & $h=3$ & $h=6$ & $h=12$ & \textit{Avg.} \\ 
\midrule
\rowcolor{datasetgray}
\multicolumn{14}{@{}l}{\hspace{0.25em}\textbf{P1}: 480 existing sensors; 211 newly introduced sensors} \\[-0.1em]
\multirow{2}{*}{Zero-shot} & Existing & \hsecond{39.07}{36.38}{39.51}\stbpstd{1.56} & \hval{43.59}{41.35}{43.59}\stbpstd{1.78} & \hval{52.66}{48.10}{52.66}\stbpstd{3.27} & \hval{44.18}{42.41}{44.18}\stbpstd{2.10} & \hval{39.51}{36.38}{39.51}\stbpstd{0.17} & \hsecond{42.06}{41.35}{43.59}\stbpstd{0.36} & \hbest{48.10}{48.10}{52.66}\stbpstd{0.63} & \hsecond{42.58}{42.41}{44.18}\stbpstd{0.36} & \hbest{36.38}{36.38}{39.51}\stbpstd{0.33} & \hbest{41.35}{41.35}{43.59}\stbpstd{0.44} & \hsecond{51.78}{48.10}{52.66}\stbpstd{1.51} & \hbest{42.41}{42.41}{44.18}\stbpstd{0.55} \\ 
 & Newly introduced & \hsecond{30.10}{29.38}{30.15}\stbpstd{2.86} & \hbest{31.13}{31.13}{31.62}\stbpstd{0.41} & \hsecond{35.93}{35.41}{36.97}\stbpstd{2.39} & \hsecond{31.91}{31.62}{32.53}\stbpstd{0.50} & \hbest{29.38}{29.38}{30.15}\stbpstd{0.30} & \hbest{31.13}{31.13}{31.62}\stbpstd{0.36} & \hbest{35.41}{35.41}{36.97}\stbpstd{0.38} & \hbest{31.62}{31.62}{32.53}\stbpstd{0.27} & \hval{30.15}{29.38}{30.15}\stbpstd{1.07} & \hsecond{31.62}{31.13}{31.62}\stbpstd{0.93} & \hval{36.97}{35.41}{36.97}\stbpstd{1.86} & \hval{32.53}{31.62}{32.53}\stbpstd{1.19} \\
\bottomrule
\end{tabular}%
}
\end{table*}

\begin{table*}[t]
\centering
\caption{Period-wise MAPE on existing and newly introduced sensors of CA-Stream. Results are mean $\pm$ standard deviation over three runs. Lower is better; \best{bold} and \second{underlined} denote the best and second-best methods within each period, sensor group, adaptation budget, and horizon.}
\label{tab:new_sensor_adaptation_ca_mape}
\setlength{\tabcolsep}{2.5pt}
\renewcommand{\arraystretch}{0.99}
\fontsize{6.2}{7.1}\selectfont
\resizebox{\linewidth}{!}{%
\begin{tabular}{@{}ll cccc cccc cccc@{}}
\toprule
\multirow{2}{*}{\textbf{Budget}} & \multirow{2}{*}{\textbf{Sensor group}}
& \multicolumn{4}{c}{FNO}
& \multicolumn{4}{c}{EVOLVE}
& \multicolumn{4}{c}{\textbf{\model-Large}}
\\[-0.15em]
\cmidrule(lr){3-6}\cmidrule(lr){7-10}\cmidrule(lr){11-14}
& & $h=3$ & $h=6$ & $h=12$ & \textit{Avg.} & $h=3$ & $h=6$ & $h=12$ & \textit{Avg.} & $h=3$ & $h=6$ & $h=12$ & \textit{Avg.} \\ 
\midrule
\rowcolor{datasetgray}
\multicolumn{14}{@{}l}{\hspace{0.25em}\textbf{P1}: 480 existing sensors; 211 newly introduced sensors} \\[-0.1em]
\multirow{2}{*}{Zero-shot} & Existing & \hsecond{42.26}{38.11}{44.15}\stbpstd{0.87} & \hsecond{43.55}{36.63}{45.18}\stbpstd{1.75} & \hsecond{48.20}{39.66}{48.58}\stbpstd{0.23} & \hsecond{44.36}{38.21}{45.65}\stbpstd{0.94} & \hval{44.15}{38.11}{44.15}\stbpstd{0.86} & \hval{45.18}{36.63}{45.18}\stbpstd{0.99} & \hval{48.58}{39.66}{48.58}\stbpstd{1.25} & \hval{45.65}{38.21}{45.65}\stbpstd{1.04} & \hbest{38.11}{38.11}{44.15}\stbpstd{0.33} & \hbest{36.63}{36.63}{45.18}\stbpstd{0.97} & \hbest{39.66}{39.66}{48.58}\stbpstd{2.83} & \hbest{38.21}{38.21}{45.65}\stbpstd{0.96} \\ 
 & Newly introduced & \hsecond{60.70}{59.93}{63.85}\stbpstd{5.50} & \hsecond{60.78}{57.73}{65.25}\stbpstd{2.52} & \hbest{61.24}{61.24}{69.52}\stbpstd{4.64} & \hsecond{61.21}{60.42}{65.91}\stbpstd{4.10} & \hval{63.85}{59.93}{63.85}\stbpstd{1.74} & \hval{65.25}{57.73}{65.25}\stbpstd{1.85} & \hval{69.52}{61.24}{69.52}\stbpstd{1.85} & \hval{65.91}{60.42}{65.91}\stbpstd{1.97} & \hbest{59.93}{59.93}{63.85}\stbpstd{2.90} & \hbest{57.73}{57.73}{65.25}\stbpstd{2.85} & \hsecond{62.96}{61.24}{69.52}\stbpstd{2.55} & \hbest{60.42}{60.42}{65.91}\stbpstd{1.81} \\
\bottomrule
\end{tabular}%
}
\end{table*}

\begin{table*}[t]
\centering
\caption{Period-wise RMSE on existing and newly introduced sensors of AIR-Stream. Results are mean $\pm$ standard deviation over three runs. Lower is better; \best{bold} and \second{underlined} denote the best and second-best methods within each period, sensor group, adaptation budget, and horizon.}
\label{tab:new_sensor_adaptation_air_rmse}
\setlength{\tabcolsep}{2.5pt}
\renewcommand{\arraystretch}{0.99}
\fontsize{6.2}{7.1}\selectfont
\resizebox{\linewidth}{!}{%
\begin{tabular}{@{}ll cccc cccc cccc@{}}
\toprule
\multirow{2}{*}{\textbf{Budget}} & \multirow{2}{*}{\textbf{Sensor group}}
& \multicolumn{4}{c}{FNO}
& \multicolumn{4}{c}{EVOLVE}
& \multicolumn{4}{c}{\textbf{\model-Large}}
\\[-0.15em]
\cmidrule(lr){3-6}\cmidrule(lr){7-10}\cmidrule(lr){11-14}
& & $h=3$ & $h=6$ & $h=12$ & \textit{Avg.} & $h=3$ & $h=6$ & $h=12$ & \textit{Avg.} & $h=3$ & $h=6$ & $h=12$ & \textit{Avg.} \\ 
\midrule
\rowcolor{datasetgray}
\multicolumn{14}{@{}l}{\hspace{0.25em}\textbf{2017}: 1,087 existing sensors; 67 newly introduced sensors} \\[-0.1em]
\multirow{2}{*}{Zero-shot} & Existing & \hval{28.86}{27.91}{28.86}\stbpstd{0.45} & \hsecond{36.61}{35.78}{37.89}\stbpstd{0.22} & \hsecond{42.36}{41.76}{44.14}\stbpstd{0.42} & \hsecond{34.61}{33.87}{35.66}\stbpstd{0.27} & \hsecond{28.81}{27.91}{28.86}\stbpstd{0.08} & \hval{37.89}{35.78}{37.89}\stbpstd{0.06} & \hval{44.14}{41.76}{44.14}\stbpstd{0.02} & \hval{35.66}{33.87}{35.66}\stbpstd{0.03} & \hbest{27.91}{27.91}{28.86}\stbpstd{0.20} & \hbest{35.78}{35.78}{37.89}\stbpstd{0.37} & \hbest{41.76}{41.76}{44.14}\stbpstd{0.62} & \hbest{33.87}{33.87}{35.66}\stbpstd{0.38} \\ 
 & Newly introduced & \hsecond{28.55}{27.43}{28.96}\stbpstd{0.53} & \hsecond{36.56}{35.60}{38.85}\stbpstd{0.35} & \hsecond{43.49}{42.80}{46.67}\stbpstd{0.79} & \hsecond{34.83}{33.96}{36.78}\stbpstd{0.49} & \hval{28.96}{27.43}{28.96}\stbpstd{0.10} & \hval{38.85}{35.60}{38.85}\stbpstd{0.04} & \hval{46.67}{42.80}{46.67}\stbpstd{0.03} & \hval{36.78}{33.96}{36.78}\stbpstd{0.03} & \hbest{27.43}{27.43}{28.96}\stbpstd{0.17} & \hbest{35.60}{35.60}{38.85}\stbpstd{0.38} & \hbest{42.80}{42.80}{46.67}\stbpstd{0.60} & \hbest{33.96}{33.96}{36.78}\stbpstd{0.39} \\
\bottomrule
\end{tabular}%
}
\end{table*}

\begin{table*}[t]
\centering
\caption{Period-wise MAPE on existing and newly introduced sensors of AIR-Stream. Results are mean $\pm$ standard deviation over three runs. Lower is better; \best{bold} and \second{underlined} denote the best and second-best methods within each period, sensor group, adaptation budget, and horizon.}
\label{tab:new_sensor_adaptation_air_mape}
\setlength{\tabcolsep}{2.5pt}
\renewcommand{\arraystretch}{0.99}
\fontsize{6.2}{7.1}\selectfont
\resizebox{\linewidth}{!}{%
\begin{tabular}{@{}ll cccc cccc cccc@{}}
\toprule
\multirow{2}{*}{\textbf{Budget}} & \multirow{2}{*}{\textbf{Sensor group}}
& \multicolumn{4}{c}{FNO}
& \multicolumn{4}{c}{EVOLVE}
& \multicolumn{4}{c}{\textbf{\model-Large}}
\\[-0.15em]
\cmidrule(lr){3-6}\cmidrule(lr){7-10}\cmidrule(lr){11-14}
& & $h=3$ & $h=6$ & $h=12$ & \textit{Avg.} & $h=3$ & $h=6$ & $h=12$ & \textit{Avg.} & $h=3$ & $h=6$ & $h=12$ & \textit{Avg.} \\ 
\midrule
\rowcolor{datasetgray}
\multicolumn{14}{@{}l}{\hspace{0.25em}\textbf{2017}: 1,087 existing sensors; 67 newly introduced sensors} \\[-0.1em]
\multirow{2}{*}{Zero-shot} & Existing & \hval{21.96}{21.33}{21.96}\stbpstd{0.48} & \hsecond{29.41}{28.61}{30.05}\stbpstd{0.13} & \hbest{34.35}{34.35}{35.61}\stbpstd{0.43} & \hsecond{27.31}{27.16}{28.01}\stbpstd{0.18} & \hsecond{21.46}{21.33}{21.96}\stbpstd{0.19} & \hval{30.05}{28.61}{30.05}\stbpstd{0.18} & \hval{35.61}{34.35}{35.61}\stbpstd{0.06} & \hval{28.01}{27.16}{28.01}\stbpstd{0.14} & \hbest{21.33}{21.33}{21.96}\stbpstd{0.21} & \hbest{28.61}{28.61}{30.05}\stbpstd{0.26} & \hsecond{35.01}{34.35}{35.61}\stbpstd{0.98} & \hbest{27.16}{27.16}{28.01}\stbpstd{0.43} \\ 
 & Newly introduced & \hval{21.61}{20.94}{21.61}\stbpstd{0.40} & \hsecond{29.69}{28.78}{30.08}\stbpstd{0.35} & \hbest{36.60}{36.60}{37.28}\stbpstd{0.70} & \hsecond{27.95}{27.62}{28.35}\stbpstd{0.29} & \hsecond{21.02}{20.94}{21.61}\stbpstd{0.21} & \hval{30.08}{28.78}{30.08}\stbpstd{0.25} & \hval{37.28}{36.60}{37.28}\stbpstd{0.10} & \hval{28.35}{27.62}{28.35}\stbpstd{0.18} & \hbest{20.94}{20.94}{21.61}\stbpstd{0.30} & \hbest{28.78}{28.78}{30.08}\stbpstd{0.53} & \hsecond{36.84}{36.60}{37.28}\stbpstd{1.26} & \hbest{27.62}{27.62}{28.35}\stbpstd{0.66} \\
\bottomrule
\end{tabular}%
}
\end{table*}

\clearpage

\section{Hyperparameter Sensitivity}
\label{app:sensitivity}

We address RQ6 by examining sensitivity to DFO depth, latent grid size,
and the Fourier cutoff $M$. All experiments use the Large
configuration and follow the protocol in Section~\ref{sec:exp:setup}.
Figure~\ref{fig:hyperparameter} illustrates how these choices affect
forecasting accuracy across streams and prediction horizons.

\begin{figure}[t]
    \centering
    \includegraphics[width=\linewidth]{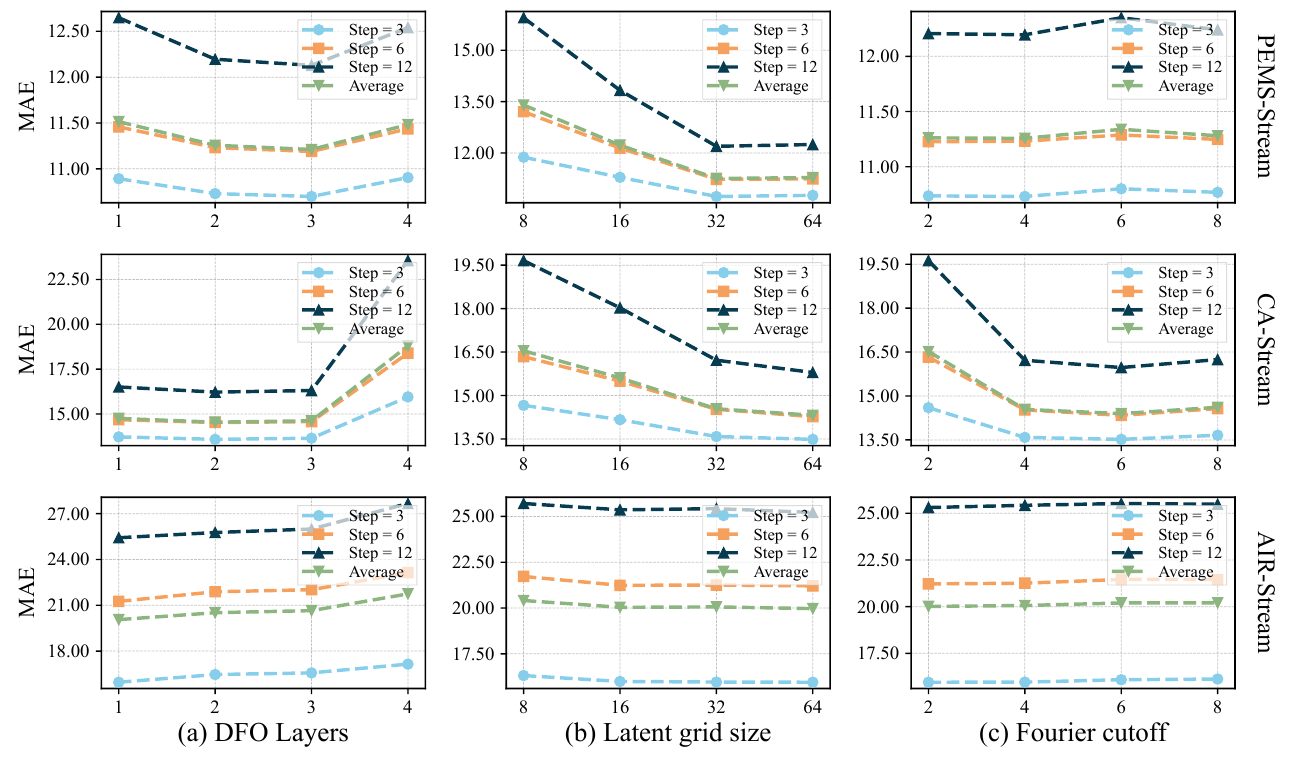}
    \caption{Hyperparameter sensitivity of \model-Large across three streams.
    Columns show variations in DFO depth, latent grid size, and Fourier cutoff.
    Curves report MAE at horizons $3$, $6$, and $12$, together with average MAE.}
    \label{fig:hyperparameter}
\end{figure}

\textbf{\textit{\ding{182} DFO depth.}}
The lowest average MAE is achieved with three layers on PEMS-Stream,
two on CA-Stream, and one on AIR-Stream. Increasing depth up to
three layers improves performance on PEMS-Stream, while CA-Stream
shows little variation between one and three layers. Errors on
AIR-Stream increase with depth. Using four layers worsens performance
on all three streams relative to their best settings, with a
particularly pronounced increase at horizon $12$ on CA-Stream.
Thus, increasing operator depth does not consistently improve forecasting.

\textbf{\textit{\ding{183} Latent grid size.}}
Increasing the number of grid points per axis from $8$ to $32$
substantially reduces errors on both traffic streams, especially
at horizon $12$. PEMS-Stream shows little further change at $64$,
while CA-Stream continues to improve, with smaller gains than
at coarser resolutions. AIR-Stream benefits mainly from increasing
the grid size from $8$ to $16$, with comparatively small changes
thereafter. These results indicate greater sensitivity to spatial
resolution on the traffic streams, with diminishing gains
as the grid becomes finer.

\textbf{\textit{\ding{184} Fourier cutoff.}}
PEMS-Stream, CA-Stream, and AIR-Stream achieve their lowest
average MAE at $M=4$, $M=6$, and $M=2$, respectively.
CA-Stream is the most sensitive: increasing $M$ from $2$ to $4$
sharply reduces errors, $M=6$ yields a further modest improvement,
and $M=8$ slightly increases errors.
PEMS-Stream varies little across the tested settings, while
AIR-Stream shows a slight increase in error as $M$ increases.
The benefit of a wider spectral bandwidth therefore depends
on the stream.

Overall, \model{} performs well with shallow operator stacks,
while the preferred grid resolution and spectral bandwidth
vary across streams. Increasing these hyperparameters beyond
suitable settings yields limited gains or degrades performance,
supporting the use of compact field operators with spatial
and spectral resolutions suited to each forecasting task.
\clearpage
\end{document}